\documentclass[lettersize,journal]{IEEEtran}
\usepackage[caption=false,font=normalsize,labelfont=sf,textfont=sf]{subfig}	
\usepackage{array}
\usepackage{textcomp}
\usepackage{stfloats}
\usepackage{url}
\usepackage{verbatim}
\usepackage{cite}
\usepackage{algpseudocode}
\usepackage{latexsym}
\usepackage[usenames,dvipsnames,svgnames,table]{xcolor}
\usepackage{xspace}
\makeatletter
\DeclareRobustCommand\onedot{\futurelet\@let@token\@onedot}
\usepackage{amsmath}	
\usepackage{amsfonts} 
\usepackage{graphicx}	
\usepackage{multirow}	
\usepackage{makecell}	
\usepackage{booktabs}	
\usepackage{subfig}	
\usepackage{algorithm}	
\usepackage{algorithmicx} 
\usepackage{amsthm}	
\newtheorem{theorem}{Theorem} 

\begin{document}
	
	\title{Scalable Label Distribution Learning for Multi-Label Classification}
	
	\author{Xingyu~Zhao,
		Yuexuan~An,
		Lei~Qi,
		and~Xin~Geng
		\IEEEcompsocitemizethanks{\IEEEcompsocthanksitem Xingyu Zhao is with the College of Computer Science and Technology, Nanjing University of Aeronautics and Astronautics, Nanjing 211106, China, and also with the School of Computer Science and Engineering, Southeast University, Nanjing 211189, China. Yuexuan~An, Lei~Qi, and Xin Geng are with the School of Computer Science and Engineering, Southeast University, Nanjing 211189, China, and also with Key Laboratory of New Generation Artificial Intelligence Technology and Its Interdisciplinary Applications (Southeast University), Ministry of Education, China (e-mail: zhaoxy@nuaa.edu.cn; yx\_an@seu.edu.cn;  qilei@seu.edu.cn; xgeng@seu.edu.cn).}
	}
	
	\maketitle
	
	\begin{abstract}
		Multi-label classification (MLC) refers to the problem of tagging a given instance with a set of relevant labels. Most existing MLC methods are based on the assumption that the correlation of two labels in each label pair is symmetric, which is violated in many real-world scenarios. Moreover, most existing methods design learning processes associated with the number of labels, which makes their computational complexity a bottleneck when scaling up to large-scale output space. To tackle these issues, we propose a novel method named Scalable Label Distribution Learning (SLDL) for multi-label classification which can describe different labels as distributions in a latent space, where the label correlation is asymmetric and the dimension is independent of the number of labels. Specifically, SLDL first converts labels into continuous distributions within a low-dimensional latent space and leverages the asymmetric metric to establish the correlation between different labels. Then, it learns the mapping from the feature space to the latent space, resulting in the computational complexity is no longer related to the number of labels. Finally, SLDL leverages a nearest-neighbor-based strategy to decode the latent representations and obtain the final predictions. Extensive experiments illustrate that SLDL achieves very competitive classification performances with little computational consumption.
	\end{abstract}
	

	\section{Introduction}
	
	\IEEEPARstart{L}{earning} with ambiguity, where an instance cannot be fully described using a single concept, is a prominent subject within the machine learning community \cite{ChenJ2023,HuangW2023,BCR,LDL}. Label ambiguity is unavoidable in some applications. In image classification, a single image might contain multiple objects (e.g., a cat and a dog in the same image). Similarly, in text classification, a document might cover multiple topics (e.g., a news article discusses both technology and finance). Multi-label classification (MLC) emerges as a common approach to address the challenge of label ambiguity \cite{MLL2014,MLL2022}. MLC allows instances to be annotated with multiple labels simultaneously, offering a practical solution to the ambiguity problem. In the context of MLC, each class label is treated as a logical indicator, where 1 denotes relevance to the instance, and 0 signifies irrelevance. Such labels, represented by 1 or 0, are commonly referred to as \textit{logical labels}. Over the years, various MLC techniques have found widespread application across diverse domains, including document classification \cite{ZongD2023}, image recognition \cite{ZhangJ2023}, video concept detection \cite{LoH2011}, temporal action detection \cite{GaoZ2023}, fraud detection \cite{Wang2020}, among others.
	
	In real-world MLC problems, exploring the correlations among different labels is a crucial research area. Label correlation provides more accurate label information compared to the original logical labels, making it a valuable aspect to enhance the learning capabilities of MLC models. Numerous research efforts focus on uncovering label correlation information to enhance the learning capabilities of MLC models, such as classifier chain-based approaches \cite{AdaBoost.C2,GerychW2021,CC}, sequence-based approaches \cite{SGM}, embedding-based approaches \cite{CLIF,C2AE,DELA}, label distribution learning (LDL)-based approaches \cite{FLEM,LEMLL,SMILE}, etc. Given the importance of effectively mining underlying information from different labels, the exploration of label correlation is widely undertaken in various real-world multi-label tasks \cite{SuX2023,CorNet,WangF2023}.
	
	\begin{figure} [t]
		\centering 
		\resizebox{\columnwidth}{!}{
			\subfloat[labels: \textit{sky}, \textit{cloud}, \textit{plant}]{
				\includegraphics[width=0.47\columnwidth]{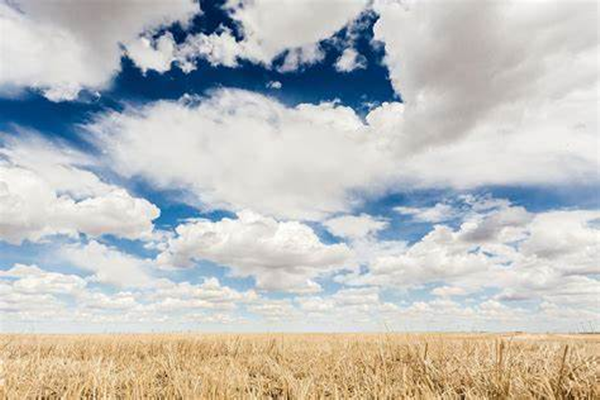}
				\label{fig:sky+cloud}			
			}
			\subfloat[labels: \textit{sky}, \textit{plant}]{
				\includegraphics[width=0.47\columnwidth]{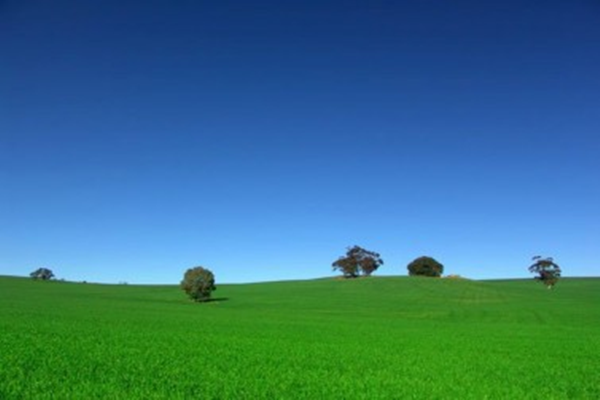}
				\label{fig:sky-cloud}
			}
		}
		\caption{An illustration of exemplar images and their corresponding labels. The correlation of ``\textit{sky}'' and ``\textit{cloud}'' are asymmetric. Specifically, if ``\textit{cloud}'' appears in an instance, then ``\textit{sky}'' also appears; while if ``\textit{sky}'' appears, ``\textit{cloud}'' may not necessarily appear.
		} 
		\label{fig:sky}
	\end{figure}

	However, existing MLC methods grapple with two primary challenges. Firstly, a prevalent assumption in many of these methods is that label correlations within each label pair are symmetric\cite{ZongD2023,C2AE,FLEM}, which often does not hold in real-world scenarios. For instance, Fig. \ref{fig:sky} illustrates two images which are annotated with labels ``\textit{sky}'', ``\textit{cloud}'', ``\textit{plant}'' and ``\textit{sky}'', ``\textit{plant}'', respectively. Notably, the correlation between ``\textit{sky}'' and ``\textit{cloud}'' is asymmetrical. Specifically, if ``\textit{cloud}'' appears in an instance, ``\textit{sky}'' also appears; nevertheless, if ``\textit{sky}'' appears, ``\textit{cloud}'' may not necessarily be present. Recognizing and accounting for such asymmetric correlations between labels can provide more accurate information for MLC, necessitating a shift away from the assumption of symmetric label correlations \cite{HuangS2012,BaoJ2022}. Secondly, many existing methods tailor their learning processes based on the number of labels \cite{CC,AdaBoost.C2,SGM}, leading to computational complexities that become bottlenecks when scaling up to large-scale output space MLC tasks. These tasks introduce new computational and statistical challenges, prompting the need for more scalable and efficient MLC methodologies.
	
	To address these challenges, we propose Scalable Label Distribution Learning (SLDL) for multi-label classification. SLDL characterizes distinct labels through latent distributions with asymmetric label correlations which are independent of the number of labels. Specifically, SLDL first transforms labels into continuous distributions within a low-dimensional Gaussian embedding space. To effectively capture the accurate correlation among different labels, it employs an asymmetric metric to establish label correlations in the Gaussian embedding space. Subsequently, it learns a mapping from the feature space to the latent space, where computational complexity is no longer tied to the number of labels. Finally, SLDL utilizes a nearest-neighbor-based strategy to decode the latent representations and generate the ultimate predictions. This approach significantly enhances both the scalability and classification performance of the model, offering a promising solution to the computational challenges posed by multi-label classification tasks with large-scale output space.
	
	Our contributions are as follows:
	\begin{itemize}
		\item We introduce Scalable Label Distribution Learning (SLDL) as a scalable and effective solution for multi-label classification tasks with large-scale output space.
		\item We articulate different labels within a latent low-dimensional Gaussian embedding space and propose an asymmetric metric to effectively capture the accurate correlation among different labels.
		\item We induce an approach combined with a simple yet effective objective function and the L-BFGS optimization algorithm to learn the mapping from the feature space to the latent embedding space, where the computational complexity is no longer related to the number of labels.
		\item We conduct comprehensive analyses of the proposed SLDL method, demonstrating its superior performance in addressing multi-label classification tasks with large-scale output space both theoretically and practically.
	\end{itemize}
	
	The subsequent sections of the paper are structured as follows. Firstly, Section \ref{sec:relatedwork} provides a concise review and discussion of related work. Secondly, Sections \ref{sec:method} delve into the technical details and theoretical analysis of Scalable Label Distribution Learning (SLDL). Following this, Section \ref{sec:experiments} presents the results of comparative experiments. Lastly, Section \ref{sec:conclusions} draws conclusions about SLDL.

	\section{Related Work} \label{sec:relatedwork}
	
	\subsection{Multi-Label Classification}
	
	In recent years, the field of multi-label classification (MLC) has witnessed significant attention \cite{VCLDL,PACA}. Traditional multi-label classification approaches fall into three main categories based on label correlation order: first-order, second-order, and high-order approaches \cite{MLL2014}. First-order methods \cite{BR,ZhangM2018,Bonsai,Slice} treat multi-label classification as independent binary classification tasks, neglecting potential information sharing among labels and overlooking the mutual benefits of knowledge gained from one label for learning other labels. Second-order approaches \cite{Rank-SVM,CLR} focus on pairwise label correlations, emphasizing differences between relevant and irrelevant labels without fully leveraging relationships among multiple labels. High-order methods \cite{SLEEC,AnnexML,RankAE,GLaS,EXMLDS} extend beyond pairwise correlations to consider relationships among label subsets or all class labels, aiming for a more comprehensive label space representation. However, a common limitation is their assumption of equal label importance.
	
	To address this limitation, researchers have explored assigning different weights to labels based on class frequencies. For instance, \cite{DB} introduces the distribution-balanced loss to rebalance weights and alleviate the over-suppression of negative labels. An alternative is the use of an asymmetric loss, as seen in \cite{RidnikT2021}, employing different $\gamma$ values to weight positive and negative samples in focal loss \cite{focal}. Additionally, \cite{HuangY2021} combines negative-tolerant regularization (NTR) \cite{DB} and class-balanced focal loss (CB) \cite{CB} to create a novel loss function called CB-NTR. Balanced softmax \cite{balancedsoftmax} transforms the multi-label classification loss into comparisons between relevant and irrelevant label scores to balance label importance. However, these methods rely on manually designed rules and lack the ability to dynamically adjust label importance based on specific instances. Moreover, they often overlook correlations among different labels, which are valuable for accurate predictions and should be fully leveraged.
	
	Regarding large-scale output space multi-label classification tasks, existing approaches can be categorized into four types: 1) One-vs-all, 2) Tree-based, 3) Embedding-based, and 4) Deep learning-based methods. One-vs-all methods \cite{PD-Sparse,DiSMEC} employ a separate binary classifier for each label when classifying a new instance. Tree-based methods \cite{FastXML,Bonsai} adopt decision tree principles. Embedding-based approaches \cite{SLEEC,DXML,GLaS,AnnexML} aim to reduce the effective number of labels by projecting label vectors into a lower-dimensional space. Deep learning-based methods \cite{xmlcnn,attentionXML} incorporate the latest deep learning technologies into large-scale output space classification. It should be noticed that most existing multi-label classification methods link the learning process to the number of labels, which seriously affects their scalability. Simultaneously, these methods often neglect to capitalize on asymmetric correlations among different labels, which restricts the performance of classification models.
	
	\begin{figure*}[t]
		\centering 
		\includegraphics[width=2\columnwidth]{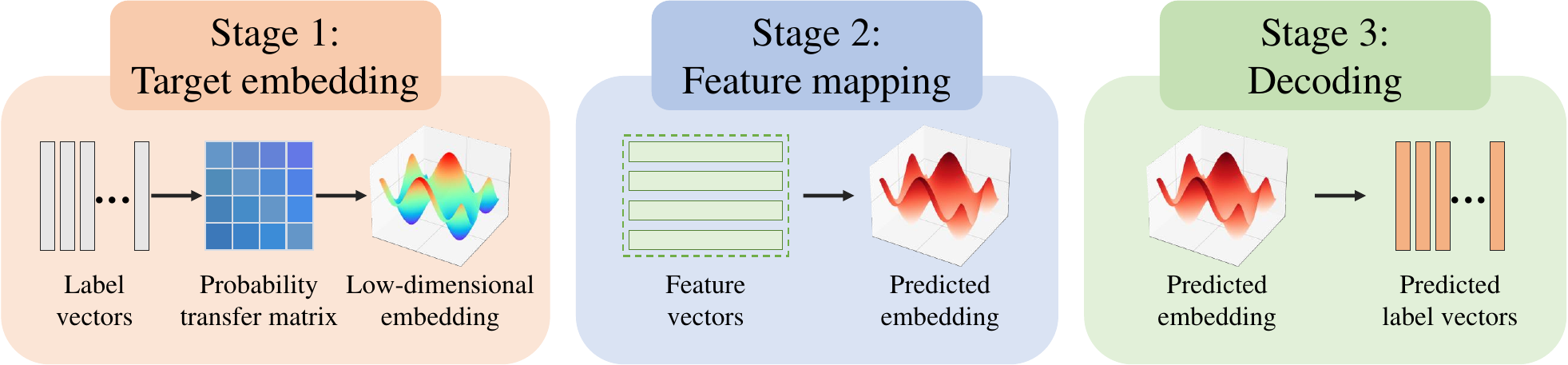}
		\caption{The schematic diagram of SLDL. The whole process of SLDL can be divided into three stages: (1) target embedding: transform the label vectors into low-dimensional embedding vectors, where the dimension of the target space is reduced and	the asymmetric label correlations are constructed; (2) feature mapping: learn a mapping function from feature vectors to embedding vectors, where the computational complexity is no longer related to the number of labels; (3) decoding: map the target embedding vector to the predicted label vector.}
		\label{fig:model}
	\end{figure*}
	
	\subsection{Label Distribution Learning}
	
	Label Distribution Learning (LDL) has emerged as a promising paradigm for uncovering correlations among different labels, assigning a label distribution to an instance and directly learning a mapping from instance to label distribution \cite{LDL,ILDL,GLDL}. LDL has demonstrated success in various real-world applications, including facial landmark detection \cite{Su2019}, age estimation \cite{Gao2018}, head pose estimation \cite{Geng2014}, zero-shot learning \cite{Huo2017}, and emotion analysis from texts \cite{Zhou2016}. The field has seen the development of specialized LDL algorithms such as SA-IIS and SA-BFGS \cite{LDL} which integrate the maximum entropy model \cite{BergerA1996} and K-L divergence using the improved iterative scaling strategy \cite{IIS} and the BFGS algorithm respectively for optimization. SCE-LDL \cite{SCE-LDL} introduces sparsity constraints into the objective function to enhance the model. Novel ideas have also been incorporated into LDL, including viewing it as a regression problem (LDSVR \cite{LDSVR}), using differentiable decision trees (LDL Forests \cite{LDLF}), and developing a deep neural network-based model called Deep LDL \cite{DLDL}.

	Efforts to exploit label correlations within LDL can be categorized into three main groups: 1) global label correlation, 2) local label correlation, and 3) both global and local label correlations. For global label correlation, methods including EDL \cite{EDL}, LDLLC \cite{LDLLC}, IncomLDL \cite{IncomLDL}, LALOT \cite{LALOT}, and LDLSF \cite{LDLSF} adopt various strategies to incorporate label correlations into their optimization objectives. Addressing local label correlation, EDL-LRL \cite{EDL-LRL}, GD-LDL-SCL \cite{GD-LDL-SCL}, and Adam-LDL-SCL \cite{Adam-LDL-SCL} focus on capturing local label correlations within clusters of samples. For both global and local label correlations, LDL-LCLR \cite{LDL-LCLR} captures global correlations with a low-rank matrix and updates it on different clusters to explore local label correlation. CLDL \cite{CLDL} leverages the continuous nature of labels to extract high-order correlations among different labels. 
	
	Some recent studies try to leverage label distribution to guide the training of multi-label classification models. \cite{LDR} studies a family of loss functions named label-distributionally robust (LDR) losses that are formulated from a distributionally robust optimization (DRO) perspective, where the uncertainty in the given label information is modeled and captured by taking the worst case of distributional weights. \cite{WangS2020} proposes an adversarial learning framework to enforce similarity between the joint distribution of the ground-truth multi-labels and the predicted multiple labels. By adversarial learning, the joint label distribution of the predicted multi-labels converges to the joint distribution inherent in the ground-truth multi-labels, and thus boosts the performance of multi-label classification. \cite{FLEM} proposes a novel approach that effectively integrates the learning process of label distribution and the training process of multi-label classification. However, It is essential to highlight that the computational complexity of LDL is generally higher than traditional learning frameworks, particularly when extending logical labels to label distribution vectors. The complexity significantly increases with a large number of labels, rendering these algorithms impractical. Moreover, existing methods often neglect to leverage asymmetric correlations among different labels, which further limits the performance of the model.
	
	In the next section, we will introduce a novel approach for multi-label classification tasks with large-scale output spaces named Scalable Label Distribution Learning (SLDL). SLDL differs substantially from existing MLC and LDL methods. In SLDL, labels are considered as a continuous distribution in the latent low-dimensional space. Learning in this continuous label distribution space allows for the effective capture of asymmetric label correlations. Moreover, the computational complexity is no longer tied to the number of labels, substantially improving the scalability of the model.

	\section{The SLDL Method}  \label{sec:method}
	
	\subsection{Problem Setting}

	Firstly, the primary notations employed in this paper are defined as follows. Let $\mathcal{D}=\{ \left(\boldsymbol{x}_1, \boldsymbol{y}_1\right), \cdots,  \left(\boldsymbol{x}_n, \boldsymbol{y}_n\right)\}$ represent the given training dataset, where $\boldsymbol{x}_i \in \mathcal{X} \subset \mathbb{R}^q$ denotes the input feature vector, and $\boldsymbol{y}_i \in \mathcal{Y} = \{ 0, 1\}^{c}$ represents the corresponding label vector, with $c$ denoting the number of possible labels. Define $\boldsymbol{X}=\left[ \boldsymbol{x}_1, \cdots, \boldsymbol{x}_n \right]$ as the feature matrix and $\boldsymbol{Y}=\left[ \boldsymbol{y}_1, \cdots, \boldsymbol{y}_n \right]$ as the label matrix. For an instance $\boldsymbol{x} \in \mathbb{R}^q$, the objective is to train a multi-label classifier $h\left(\boldsymbol{x}\right):\mathbb{R}^q \rightarrow \{ 0, 1\}^{c}$ to predict the proper labels for $\boldsymbol{x}$. The summary of the mainly used notations is listed in Table \ref{table:symbols}.

	\subsection{Overview}
	
	In various practical applications, the dimensionality of the output space can be exceedingly high. For instance, there are a vast number of categories on Wikipedia and one might wish to develop a classifier that annotates a new article with the most relevant subset of these categories. This results in a massive output space for the classifier, posing significant challenges to the learning system \cite{VCLDL,DXML}. To handle this challenge, SLDL incorporates the target embedding method. Concretely, the target embedding method represents logical labels as continuous, dense vectors in a lower-dimensional space, so as to reduce the label dimensionality and capture the correlations between labels \cite{GLaS}.
	As illustrated in Fig. \ref{fig:model}, SLDL first transforms $\boldsymbol{y}_i$ into an embedding vector ${\boldsymbol{z}}_i\in\mathcal{Z}=\mathbb{R}^{\hat{c}}$, where $\hat{c}\ll c$. Subsequently, SLDL learns the mapping from the instance to the low-dimensional embedding vector. During the prediction phase, the output score is decoded from the predicted embedding vector using a nearest-neighbor-based strategy. This methodology substantially reduces the resource consumption of the algorithm while effectively exploring the asymmetric correlation among different labels. In the remainder of this section, we will present the implementation of the SLDL method.
	
	\subsection{Asymmetric-Correlational Gaussian Embedding}
	
	In SLDL, the extraction of asymmetric correlations is facilitated through the utilization of Gaussian embedding whose capacity offers a flexible and expressive representation of label distributions. To learn a reasonable transition relationship between different labels, we construct a probability transfer matrix. Then, each label is treated as a multivariate Gaussian distribution, and an asymmetric metric is employed to establish the asymmetric correlation between each pair of labels. Through learning in the latent Gaussian embedding space, the continuous distributions that contain the asymmetric correlation among different labels can be effectively learned, which enhances the ability of the model to describe and exploit the nuanced asymmetric correlations inherent in real-world multi-label classification scenarios.
	
	\begin{table}[t]
		\centering
		\setlength{\tabcolsep}{4mm}
		\caption{Summary of the Mainly Used Notations.}
		\label{table:symbols}
		\begin{tabular}{ll}
			\hline
			{Symbol} & {Definition} \\
			\hline
			{$\mathcal{D}$} & {Training dataset} \\
			{$\mathcal{X}$} & {Feature space} \\
			{$\mathcal{Y}$} & {Label space} \\
			{$\boldsymbol{x}_i$} & {The $i$-th feature vector} \\
			{$\boldsymbol{y}_i$} & {The $i$-th label vector} \\
			{$\boldsymbol{z}_i$} & {The $i$-th embedding vector} \\
			{$\boldsymbol{X}$} & {The feature matrix of training instances} \\
			{$\boldsymbol{Y}$} & {The label matrix for training instances} \\
			{$\boldsymbol{Z}$} & {The embedding matrix for training instances} \\
			{$\boldsymbol{A}$} & {The label correlation matrix} \\
			{$\boldsymbol{P}$} & {The probability transfer matrix} \\
			{$\boldsymbol{W}$} & {The model parameter} \\
			{$q$} & {The number of feature dimensions} \\
			{$n$} & {The number of training instances} \\
			{$c$} & {The number of labels} \\
			{$\hat{c}$} & {The number of embedding dimension} \\
			{$h$} & {The multi-label classifier} \\
			{$\gamma$} & {The discount factor} \\
			\hline
		\end{tabular}%
	\end{table}%

	\subsubsection{Probability Transfer Matrix Construction}

	We first construct a label correlation matrix $\boldsymbol{A} \in \mathbb{R}^{c\times c}$, in which the corresponding element is 1 if the two labels ever co-exist with each other in any sample:
	\begin{equation}
		{\boldsymbol{A}_{ij}}=\left\{ \begin{aligned}
			& 1, \text{if the label $i$ and $j$ ever co-exist,} \\ 
			& 0, \text{otherwise.} \\ 
		\end{aligned} \right.
		\label{eq:label-correlation-A}
	\end{equation}
	
	For each row, we normalize the label correlation vector:
	\begin{equation}
		{\widehat{\boldsymbol{A}}_{ij}}=\frac{{\boldsymbol{A}_{ij}}}{\sum\nolimits_{l=1}^{c}{{\boldsymbol{A}_{il}}}}.
		\label{eq:label-correlation-A-normalized}
	\end{equation}
	
	In order to effectively establish the correlation between different labels, we leverage random walk simulation and aggregate the accumulated values. Denote ${\boldsymbol{P} }^{(i)}$ as the probability transfer matrix at the $i$-th step and ${\gamma }^{(i)}$ as the discount factor balancing the short-term and long-term transfer probabilities. The parameters of random walk simulation are initialized as
	\begin{equation}
		{\boldsymbol{P}^{(0)}}=\widehat{\boldsymbol{A}},
		\label{eq:init-P}
	\end{equation}
	\begin{equation}
		{{\gamma }^{(0)}}=1.
		\label{eq:init-gamma}
	\end{equation}
	
	Subsequently, we perform the random walk processes several times and obtain the accumulated probability transfer matrix ${\boldsymbol{P}^{(total)}}$:
	\begin{equation}
		\left\{ \begin{aligned}
			& {\boldsymbol{P}^{(i)}}\leftarrow {\boldsymbol{P}^{(i-1)}}\widehat{\boldsymbol{A}}, \\ 
			& {{\gamma }^{(i)}}\leftarrow {{\gamma }^{(i-1)}}/2, \\ 
			& {\boldsymbol{P}^{(total)}}\leftarrow {\boldsymbol{P}^{(total)}}+{{\gamma }^{(i)}}\cdot {\boldsymbol{P}^{(i)}}. \\ 
		\end{aligned} \right.
		\label{eq:random-walk}
	\end{equation}
	
	After that, we standardize the accumulated probability transfer matrix and obtain the standardized matrix ${\boldsymbol{\widehat{P}}^{(total)}}$ which can effectively reflect the probability transfer relation among different labels:
	\begin{equation}
		{\boldsymbol{\widehat{P}}_{ij}^{(total)}}=\frac{\boldsymbol{P}_{ij}^{(total)}}{\sum\nolimits_{l=1}^{c}{\boldsymbol{P}_{il}^{(total)}}}.
		\label{eq:standardized-probability-transfer-matrix}
	\end{equation}

	\subsubsection{Learning Asymmetric-Correlational Gaussian Embedding}
	
	After obtaining the standardized probability transfer matrix, we can learn the asymmetric-correlational Gaussian embedding. Inspired by \cite{DELA,VCLDL}, we assume each label as a multivariate Gaussian distribution with a diagonal covariance structure $\mathcal{N}\left( \boldsymbol{\mu}, \boldsymbol{\sigma}^2\boldsymbol{I}  \right)$, where $\boldsymbol{I}$ is the identity matrix. We follow a multi-round optimization to generate a reasonable distribution for each label. In each round, we perform the optimization for each label. Concretely, for label $i$, we consider label $i$ as an anchor, and then rank the remaining labels according to the $i$-th row in ${\boldsymbol{\widehat{P}}^{(total)}}$, choosing the $j$-th label of ranked remaining labels as positive label and the $(j+1)$-th label as negative label. Then the objective function can be formulated as:
	\begin{equation}
		\min {{\left[ \texttt{KL}\left( {\mathcal{N}_{i}}\left\| {\mathcal{N}_{j}} \right. \right)-\texttt{KL}\left( {\mathcal{N}_{i}}\left\| {\mathcal{N}_{j+1}} \right. \right)+\tau  \right]}_{+}},
		\label{eq:loss-embedding}
	\end{equation}
	where 
	\begin{equation}
		\texttt{KL}\left( {{\mathcal{N}}_{i}}\left\| {{\mathcal{N}}_{j}} \right. \right)=-\frac{1}{2}\sum\limits_{k=1}^{K}{\left[ \log \upsilon _{i,j}^{(k)}-\upsilon _{i,j}^{(k)}-\tau _{i,j}^{(k)}+1 \right]},
	\end{equation}
	denotes the Kullback-Leibler (KL) divergence which establish asymmetric correlation between $\mathcal{N}\left( {{\boldsymbol{\mu }}_{i}},\boldsymbol{\sigma }_{i}^{2}\boldsymbol{I} \right)$ and $\mathcal{N}\left( {{\boldsymbol{\mu }}_{j}},\boldsymbol{\sigma }_{j}^{2}\boldsymbol{I} \right)$ corresponding to labels $i$ and $j$ respectively, $\tau$ is the threshold, $\upsilon _{i,j}^{(k)}\text{=}\frac{\boldsymbol{\sigma }_{i}^{(k)2}}{\boldsymbol{\sigma }_{j}^{(k)2}}$, $\tau _{i,j}^{(k)}=\frac{{{\left( \boldsymbol{\mu }_{i}^{(k)}-\boldsymbol{\mu }_{j}^{(k)} \right)}^{2}}}{\boldsymbol{\sigma }_{j}^{(k)2}}$, $K$ is the dimension of the latent space and $(\cdot)^{(k)}$ denotes the $k$-th element.	We optimize the objective function for several rounds and obtain the Gaussian embedding for each label. Then, for a specific instance $i$, we leverage the learned Gaussian embedding and its label vector to obtain its embedding vector:
	\begin{equation}
		{\boldsymbol{z}_{i}}=\sum\limits_{l=1}^{c}{{\boldsymbol{y}_{i}^{(l)}}}{\boldsymbol{\mu }_{l}},
		\label{eq:z}
	\end{equation}
	where ${\boldsymbol{y}_{i}^{(l)}}$ denotes the $l$-th element of $\boldsymbol{y}_{i}$. The pseudocode of the target embedding algorithm is given in algorithm \ref{algorithm:embedding}.

	\begin{algorithm}[t]
		\caption{SLDL: Target Embedding Algorithm}
		\label{algorithm:embedding}
		\begin{algorithmic}
			\State {\bfseries Required:} Training dataset: ${\mathcal{D}=\left\{\left(\boldsymbol{x}_i,\boldsymbol{y}_i\right)\right\}}_{i=1}^N$
			\State {\bfseries Required:} Embedding dimensionality: $\hat{c}$
			\State {\bfseries 1:} Construct the normalized label correlation matrix $\boldsymbol{A}$ according to Eqs.(\ref{eq:label-correlation-A}) and (\ref{eq:label-correlation-A-normalized})
			\State {\bfseries 2:} Initialize the parameters of random walk simulation according to Eqs.(\ref{eq:init-P}) and (\ref{eq:init-gamma})
			\State {\bfseries 3:} Performance the random walk processes and obtain the standardized probability transfer matrix according to Eqs.(\ref{eq:random-walk}) and (\ref{eq:standardized-probability-transfer-matrix})
			\State {\bfseries 4:} Initialize the multivariate Gaussian distribution $\mathcal{N}$ for each label
			\State {\bfseries 5:} Optimize the latent distribution for each label according to Eq.(\ref{eq:loss-embedding})
			\State {\bfseries 6:} Obtain the target embedding vectors of training samples according to Eq.(\ref{eq:z})
			\State {\bfseries Output:} Target embedding matrix $\boldsymbol{Z}$ of training samples
		\end{algorithmic}
	\end{algorithm}

	\subsection{Model Training and Prediction}	
	
	\subsubsection{Model Training}
	
	The learning process of the model is to learn a mapping function from the feature space $\mathcal{X}$ to the embedding space $\mathcal{Z}$. After obtaining the embedding vector of each instance, we form the embedding matrix $\boldsymbol{Z}$ for all instances. Then we optimize the following simple yet effective objective function to obtain the model parameter:
	\begin{equation}
		\mathcal{L}\left(\boldsymbol{W}\right)=\left\| \boldsymbol{Z}-\boldsymbol{X}\boldsymbol{W} \right\|_{F}^{2}+\alpha \left\| \boldsymbol{W} \right\|_{F}^{2},
		\label{eq:loss}
	\end{equation}
	where $\boldsymbol{W}$ represents the model parameter, $\alpha$ is the balancing factor. We utilize the Limited-memory Broyde-Fletcher-Goldfarb-Shanno (L-BFGS) algorithm \cite{L-BFGS} to minimize the function $\mathcal{L}\left(\boldsymbol{W}\right)$. L-BFGS approximates the inverse Hessian matrix using an iteratively updated matrix instead of storing the full matrix, making it suitable for large-scale optimizations. Consider the second-order Taylor series of $\mathcal{L'}\left(\boldsymbol{W}\right) = -\mathcal{L}\left(\boldsymbol{W}\right)$ at the current estimate of the parameter $\boldsymbol{W}^{\left(t\right)}$:
	\begin{equation}
		\begin{aligned}
			{\mathcal{L}}'\left( {{\boldsymbol{W}}^{\left( t+1 \right)}} \right)  \approx & {\mathcal{L}}'\left( {{\boldsymbol{W}}^{\left( t \right)}} \right)+\nabla {\mathcal{L}}'{{\left( {{\boldsymbol{W}}^{\left( t+1 \right)}} \right)}^{T}}\boldsymbol{\Delta}  \\ 
			& +\frac{1}{2}{{\Delta }^{T}}\mathcal{H}\left( {{\boldsymbol{W}}^{\left( t \right)}} \right)\boldsymbol{\Delta},
		\end{aligned}
		\label{eq:Hessian}
	\end{equation}
	where $\boldsymbol{\Delta} ={{\boldsymbol{W}}^{\left( t+1 \right)}}-{{\boldsymbol{W}}^{\left( t \right)}}$ is the update step, $\nabla {{\mathcal{L}}^{\prime }}\left( {{\boldsymbol{W}}^{\left( t \right)}} \right)$ and $\mathcal{H}\left( {{\boldsymbol{W}}^{\left( t \right)}} \right)$ are the gradient and Hessian matrix of ${{\mathcal{L}}^{\prime }}\left( {{\boldsymbol{W}}^{\left( t \right)}} \right)$ at ${{\boldsymbol{W}}^{\left( t \right)}}$, respectively. Then the minimizer of Eq.(\ref{eq:Hessian}) is
	\begin{equation}
		{{\boldsymbol{\Delta }}^{\left( t \right)}}=-{{\mathcal{H}}^{-1}}\left( {{\boldsymbol{W}}^{\left( t \right)}} \right)\nabla {{\mathcal{L}}^{\prime }}\left( {{\boldsymbol{W}}^{\left( t \right)}} \right).
	\end{equation}

	\begin{algorithm}[t]
		\caption{SLDL: Training Algorithm}
		\label{algorithm:train}
		\begin{algorithmic}
			\State {\bfseries Required:} Training dataset: ${\mathcal{D}=\left\{\left(\boldsymbol{x}_i,\boldsymbol{y}_i\right)\right\}}_{i=1}^N$
			\State {\bfseries Required:} Target embedding matrix $\boldsymbol{Z}$ of training samples
			\State {\bfseries 1:} Normalize $\boldsymbol{x}_\ast$ using L2-normalization
			\State {\bfseries 2:} Initialize the model parameter $\boldsymbol{W}$
			\State {\bfseries 3:} Optimize the model parameter $\boldsymbol{W}$ according to Eq. (\ref{eq:loss})
			\State {\bfseries Output:} Model parameter $\boldsymbol{W}$
		\end{algorithmic}
	\end{algorithm}

	\begin{algorithm}[t]
		\caption{SLDL: Test Algorithm}
		\label{algorithm:test}
		\begin{algorithmic}
			\State {\bfseries Required:} Training dataset: ${\mathcal{D}=\left\{\left(\boldsymbol{x}_i,\boldsymbol{y}_i\right)\right\}}_{i=1}^N$
			\State {\bfseries Required:} Target embedding matrix $\boldsymbol{Z}$ of training samples
			\State {\bfseries Required:} testing sample $\boldsymbol{x}_\ast$
			\State {\bfseries Required:} number of neighbors for predicting: $k$
			\State {\bfseries 1:} Normalize $\boldsymbol{x}_\ast$ and $\boldsymbol{x}_\ast$ using L2-normalization
			\State {\bfseries 2:} Calculate the transformed embedding matrix $\boldsymbol{\widehat{Z}}$ according to Eq.(\ref{eq:Z-hat})
			\State {\bfseries 3:} Obtain the corresponding target vector in the embedding space ${{\boldsymbol{\hat{z}}}}_\ast$ for $\boldsymbol{x}_\ast$ according to Eq.(\ref{eq:z-ast})
			\State {\bfseries 4:} Form $\Xi$ using $k$ samples with the smallest cosine distance to $\ {\hat{\boldsymbol{z}}}_\ast$ in $\boldsymbol{\widehat{Z}}$
			\State {\bfseries 5:} Obtain output labels according to Eq.(\ref{eq:decoder})
			\State {\bfseries Output:} Predicted labels for the testing sample $\boldsymbol{x}_\ast$
		\end{algorithmic}
	\end{algorithm}

	The line search Newton method utilizes ${{\boldsymbol{\Delta }}^{\left( t \right)}}$ as the search direction, denoted as ${{\boldsymbol{D }}^{\left( t \right)}}={{\boldsymbol{\Delta }}^{\left( t \right)}}$, and the model parameters are updated using:
	\begin{equation}
		{{\boldsymbol{W}}^{\left( t+1 \right)}}={{\boldsymbol{W}}^{\left( t \right)}}+{{\beta }^{\left( t \right)}}{{\boldsymbol{D}}^{\left( t \right)}},
	\end{equation}
	where the step length ${{\beta }^{\left( t \right)}}$ is determined through a line search procedure to satisfy the strong Wolfe conditions \cite{JorgeN2006}:
	\begin{equation}
		\begin{aligned}
			{{\mathcal{L}}^{\prime }}\left( {{\boldsymbol{W}}^{\left( t \right)}}+{{\beta }^{\left( t \right)}}{{\boldsymbol{D}}^{\left( t \right)}} \right)\le & {{\mathcal{L}}^{\prime }}\left( {{\boldsymbol{W}}^{\left( t \right)}} \right) \\  & + {{c}_{1}}{{\beta }^{\left( t \right)}}\nabla {{\mathcal{L}}^{\prime }}{{\left( {{\boldsymbol{W}}^{\left( t \right)}} \right)}^{T}}{{\boldsymbol{D}}^{\left( t \right)}},
		\end{aligned}
	\end{equation}
	\begin{equation}
		\left| \nabla {{\mathcal{L}}^{\prime }}\left( {{\boldsymbol{W}}^{\left( t \right)}}+{{\beta }^{\left( t \right)}}{{\boldsymbol{D}}^{\left( t \right)}} \right) \right|\le {{c}_{2}}\left| \nabla {{\mathcal{L}}^{\prime }}{{\left( {{\boldsymbol{W}}^{\left( t \right)}} \right)}^{T}}{{\boldsymbol{D}}^{\left( t \right)}} \right|,
	\end{equation}
	where $0<c_1<c_2<1$. The L-BFGS algorithm aims to avoid the explicit calculation of $\mathcal{H}^{-1}\left({{\boldsymbol{W}}^{\left( t \right)}}\right)$ by approximating it with an iteratively updated matrix $\boldsymbol{B}$, i.e.,
	\begin{equation}
		\begin{aligned}
			&{{\boldsymbol{B}}^{\left( t+1 \right)}}\\
			= & \left( I-{{\boldsymbol{\Lambda}}^{\left( t \right)}}{{\boldsymbol{S}}^{\left( t \right)}}{{\left( {{\boldsymbol{U}}^{\left( t \right)}} \right)}^{T}} \right){{\boldsymbol{B}}^{\left( t \right)}}\left( I-{{\boldsymbol{\Lambda} }^{\left( t \right)}}{{\boldsymbol{S}}^{\left( t \right)}}{{\left( {{\boldsymbol{U}}^{\left( t \right)}} \right)}^{T}} \right)\\  &+ {{\boldsymbol{\Lambda} }^{\left( t \right)}}{{\boldsymbol{S}}^{\left( t \right)}}{{\left( {{\boldsymbol{U}}^{\left( t \right)}} \right)}^{T}}, \\ 
		\end{aligned}
	\end{equation}
	where $\boldsymbol{I}$ is the identity matrix, $\boldsymbol{S}={{\boldsymbol{W}}^{\left( t+1 \right)}}-{{\boldsymbol{W}}^{\left( t \right)}}$, ${\boldsymbol{U}^{\left( t \right)}}=\nabla {{\mathcal{L}}^{\prime }}\left( {{\boldsymbol{W}}^{\left( t+1 \right)}} \right)-\nabla {{\mathcal{L}}^{\prime }}\left( {{\boldsymbol{W}}^{\left( t \right)}} \right)$, and ${{\boldsymbol{\Lambda }}^{\left( t \right)}}=\frac{1}{{{\boldsymbol{S}}^{\left( t \right)}}{{\boldsymbol{U}}^{\left( t \right)}}}$.
	
	We use $m$ to denote the learning iterations of L-BFGS and $N$ to denote the number of samples, then the computational complexity of the training process of SLDL is $\mathcal{O}(mNq\hat{c})$, which is no longer associated with the number of labels.

	\subsubsection{Model Prediction}
	
	In the prediction process, SLDL first calculates the embedding matrix $\boldsymbol{\widehat{Z}}$ of training samples transformed by $\boldsymbol{W}$:
	\begin{equation}
		\boldsymbol{\widehat{Z}} = \boldsymbol{X}\boldsymbol{W}.
		\label{eq:Z-hat}
	\end{equation}
	
	For a new sample $\boldsymbol{x}_\ast$, SLDL first calculates its corresponding target vector in the embedding space ${{\boldsymbol{\hat{z}}}}_\ast$:
	\begin{equation}
		{{\boldsymbol{\hat{z}}}_\ast}={{\boldsymbol{x}}_\ast}\boldsymbol{W}.
		\label{eq:z-ast}
	\end{equation}
	After that, a decoder is employed to map the target embedding vector to the predicted label vector. For simplicity, SLDL adopts a nearest-neighbor-based decoding method. The cosine distance of the new sample ${{\boldsymbol{\hat{z}}}}_\ast$ to the $i$-th sample ${\boldsymbol{\hat{z}}}_i$ in ${{\boldsymbol{\widehat{Z}}}}$ is calculated by 
	\begin{equation}
		d_i=
		{1-\frac{{{\boldsymbol{\hat{z}}}_\ast}\cdot {{\boldsymbol{\hat{z}}}_{i}}}{\left\| {{\boldsymbol{\hat{z}}}_\ast} \right\|\left\| {{\boldsymbol{\hat{z}}}_{i}} \right\|}}
		.\end{equation}
	Subsequently, by utilizing the reciprocal of the cosine distance as the weight, SLDL acquires the weighted sum of the original label vectors corresponding to the samples with the smallest cosine distance to ${\hat{\boldsymbol{z}}}_\ast$ in ${\boldsymbol{\widehat{Z}}}$:
	\begin{equation}
		\label{eq:decoder}
		{\boldsymbol{\hat{y}}_{*}}=\sum\nolimits_{i\in \Xi }{\frac{{\boldsymbol{y}_{i}}}{d_i}}
		,\end{equation}
	where $\Xi$ represents the index set corresponding to the samples with the smallest cosine distance to ${\hat{\boldsymbol{z}}}_\ast$ in ${\boldsymbol{\widehat{Z}}}$. Finally, SLDL selects the labels with top-$k$ scores for model prediction. The pseudocode of the training algorithm and test algorithm is given in algorithm \ref{algorithm:train} and algorithm \ref{algorithm:test}, respectively.

	\begin{table*}[t]
		\centering
		\setlength{\tabcolsep}{5mm}
		\caption{Statistics of multi-label classification benchmark datasets.}
		\label{table:datasets}
		\begin{tabular}{llccc}
			\toprule
			No.   & Dataset & Number of samples & Number of feature dimensions & Number of labels \\
			\midrule
			1     & cal500 & 502   & 68    & 174 \\
			2     & corel16k-s1 & 13766 & 500   & 153 \\
			3     & corel16k-s2 & 13761 & 500   & 164 \\
			4     & corel16k-s3 & 13760 & 500   & 154 \\
			5     & CUB   & 11788 & 128   & 312 \\
			6     & delicious & 16105 & 500   & 983 \\
			7     & eurlex-dc & 19348 & 500   & 412 \\
			8     & eurlex-sm & 19348 & 500   & 201 \\
			9     & espgame & 20770 & 1000  & 268 \\
			10    & stackex-chemistry & 6961  & 540   & 175 \\
			11    & stackex-chess & 1675  & 585   & 227 \\
			12    & stackex-coffee & 225   & 1763  & 123 \\
			13    & stackex-cooking & 10491 & 577   & 400 \\
			14    & stackex-cs & 9270  & 635   & 274 \\
			15    & stackex-philosophy & 3971  & 842   & 233 \\
			\bottomrule
		\end{tabular}%
	\end{table*}%

	\subsection{Theoretical Analysis}
	
	We study the theoretical guarantees of SLDL to reveal the effectiveness of Gaussian embedding with asymmetric metrics	below. Let $\mathcal{C}$ be the possible cost function for measuring the penalty between the ground-truth label vector $\boldsymbol{y}$ and the predicted label vector $\boldsymbol{\hat{y}}$. We assume $\mathcal{C}\left(\boldsymbol{y},\boldsymbol{\hat{y}}\right) \ge 0$ iff $\boldsymbol{y}$ and $\boldsymbol{\hat{y}}$ are the same. Then we have the following theorem:
	\begin{theorem}
		\label{thm:bound}
		For any sample $\left(\boldsymbol{x},\boldsymbol{y}\right)$, let $\boldsymbol{z}$ be the embedding vector of $\boldsymbol{y}$, $\boldsymbol{\hat{z}}$ be the predicted embedding vector, and $\boldsymbol{\tilde{z}}$ be the nearest embedding vector of $\boldsymbol{\hat{z}}$. Then the following bound holds
		\begin{equation}
			\mathcal{C}\left( \boldsymbol{y},\boldsymbol{\hat{y}} \right)\le b{{\left( \mathcal{E}\left( \boldsymbol{z},\boldsymbol{\tilde{z}} \right)-\mathcal{C}{{\left( \boldsymbol{y},\boldsymbol{\hat{y}} \right)}^{\frac{1}{2}}} \right)}^{2}}+b{\mathcal{E}\left( \boldsymbol{z},\boldsymbol{\hat{z}} \right)^{2}},
		\end{equation}
		where $\mathcal{E}\left(\cdot, \cdot\right)$ denotes Euclidean distance and $b > 1$ is a constant.
	\end{theorem}
	\begin{proof}
		Since $\boldsymbol{\tilde{z}}$ is the nearest neighbor of $\boldsymbol{\hat{z}}$, we define $\mathcal{E}\left( \boldsymbol{z},\boldsymbol{\tilde{z}} \right) \le \left(b-1\right) \mathcal{E}\left( \boldsymbol{z},\boldsymbol{\hat{z}} \right)$ that bounds the distance of $\boldsymbol{z}$ and $\boldsymbol{\tilde{z}}$ \cite{HuangK2017}. Then we have:
		\begin{equation}
			\begin{aligned}
				& {{\left( \mathcal{E}\left( \boldsymbol{z},\boldsymbol{\tilde{z}} \right)-{{\left( \mathcal{C}\left( \boldsymbol{y},\boldsymbol{\hat{y}} \right) \right)}^{\frac{1}{2}}} \right)}^{2}}+\mathcal{E}{{\left( \boldsymbol{z},\boldsymbol{\hat{z}} \right)}^{2}} \\ 
				\ge & {{\left( \mathcal{E}\left( \boldsymbol{z},\boldsymbol{\tilde{z}} \right)-{{\left( \mathcal{C}\left( \boldsymbol{y},\boldsymbol{\hat{y}} \right) \right)}^{\frac{1}{2}}} \right)}^{2}}+\frac{1}{b-1}\mathcal{E}{{\left( \boldsymbol{z},\boldsymbol{\tilde{z}} \right)}^{2}} \\ 
				= & \frac{b}{b-1}{{\left( \mathcal{E}\left( \boldsymbol{z},\boldsymbol{\tilde{z}} \right)-\frac{b-1}{b}{{\left( \mathcal{C}\left( \boldsymbol{y},\boldsymbol{\hat{y}} \right) \right)}^{\frac{1}{2}}} \right)}^{2}}+\frac{1}{b}\mathcal{C}\left( \boldsymbol{y},\boldsymbol{\hat{y}} \right) \\ 
				\ge & \frac{1}{b}\mathcal{C}\left( \boldsymbol{y},\boldsymbol{\hat{y}} \right),
			\end{aligned}
		\end{equation}
		which finishes the proof.
	\end{proof}
	
	Theorem \ref{thm:bound} demonstrates that the cost function of SLDL can be bounded by the sum of embedding error ${{\left( \mathcal{E}\left( \boldsymbol{z},\boldsymbol{\tilde{z}} \right)-\mathcal{C}{{\left( \boldsymbol{y},\boldsymbol{\hat{y}} \right)}^{\frac{1}{2}}} \right)}^{2}}$ and regression error ${\mathcal{E}\left( \boldsymbol{z},\boldsymbol{\hat{z}} \right)^{2}}$, which indicates that superior embedding learning approach can effectively improve the performance of the model. As we can observe in Section \ref{sec:ablation} below, our embedding method has significant advantages over vanilla embedding approaches that use the mean square error loss function, which proves the effectiveness of our proposed method.

	\begin{table*}[t]
		\centering
		\scriptsize
		\setlength{\tabcolsep}{0.3mm}
		\caption{Experimental results (mean $\pm$ std) on the MLC benchmark datasets measured by P@$1$ (nDCG@$1$). Each underlined result indicates that SLDL is statistically superior to the comparing method. The best performance on each dataset is denoted in boldface.}
		\label{table:result-P@1}
		\resizebox{2\columnwidth}{!}{
			\begin{tabular}{lcccccccccccc}
				\toprule
				Dataset & SLEEC \cite{SLEEC} & DXML \cite{DXML} & FL \cite{focal} & CB \cite{CB} & DB \cite{DB} & PACA \cite{PACA} & FLEM \cite{FLEM} & DELA \cite{DELA} & CLIF \cite{CLIF} & KD-TEA \cite{KD-TEA} & AdaC2 \cite{AdaBoost.C2} & SLDL (Ours) \\
				\midrule
				cal500 & \underline{86.66$\pm$3.16} & \underline{86.45$\pm$2.51} & \underline{87.85$\pm$3.11} & \underline{87.85$\pm$3.82} & \underline{85.45$\pm$4.10} & \underline{88.05$\pm$3.42} & \textbf{88.45$\pm$2.46} & \underline{85.07$\pm$4.64} & \underline{84.67$\pm$4.58} & \underline{87.05$\pm$3.90} & \underline{85.26$\pm$5.07} & \textbf{88.45$\pm$2.11} \\
				corel16k-s1 & \underline{36.53$\pm$1.37} & \underline{31.70$\pm$0.94} & \underline{32.64$\pm$0.85} & \underline{32.58$\pm$1.01} & \underline{30.01$\pm$0.76} & \underline{28.38$\pm$1.71} & \underline{35.97$\pm$1.13} & \underline{27.89$\pm$0.70} & \underline{27.03$\pm$1.19} & \underline{30.82$\pm$1.66} & \underline{36.37$\pm$1.18} & \textbf{37.86$\pm$0.85} \\
				corel16k-s2 & \underline{36.60$\pm$1.05} & \underline{31.54$\pm$0.96} & \underline{33.26$\pm$0.66} & \underline{33.12$\pm$0.77} & \underline{30.22$\pm$1.26} & \underline{28.76$\pm$1.56} & \underline{36.58$\pm$1.04} & \underline{28.30$\pm$0.79} & \underline{27.05$\pm$1.04} & \underline{31.51$\pm$1.43} & \underline{36.36$\pm$1.46} & \textbf{37.79$\pm$1.00} \\
				corel16k-s3 & \underline{36.29$\pm$1.41} & \underline{31.96$\pm$1.26} & \underline{32.49$\pm$1.65} & \underline{32.64$\pm$1.62} & \underline{29.65$\pm$1.84} & \underline{28.53$\pm$1.71} & \underline{35.52$\pm$1.35} & \underline{27.88$\pm$1.28} & \underline{26.87$\pm$1.33} & \underline{31.87$\pm$1.58} & \underline{36.41$\pm$0.77} & \textbf{37.80$\pm$1.51} \\
				CUB   & \underline{86.02$\pm$1.53} & \underline{84.39$\pm$0.95} & \underline{86.28$\pm$1.06} & \underline{86.49$\pm$1.06} & \underline{86.11$\pm$1.38} & \underline{86.18$\pm$0.94} & \underline{84.96$\pm$0.64} & \underline{87.05$\pm$1.16} & \underline{83.90$\pm$1.19} & \underline{86.82$\pm$1.29} & \underline{71.41$\pm$5.10} & \textbf{87.13$\pm$1.06} \\
				delicious & \underline{67.49$\pm$0.87} & \underline{59.83$\pm$1.12} & \underline{65.68$\pm$0.92} & \underline{65.61$\pm$0.90} & \underline{62.59$\pm$0.73} & \underline{65.84$\pm$1.01} & \underline{67.15$\pm$0.48} & \underline{65.05$\pm$0.94} & \underline{61.34$\pm$0.77} & \underline{67.54$\pm$0.93} & \underline{65.57$\pm$1.47} & \textbf{67.67$\pm$0.96} \\
				eurlex-dc & \underline{54.41$\pm$1.36} & \underline{54.33$\pm$0.85} & \underline{26.40$\pm$1.37} & \underline{26.30$\pm$1.41} & \underline{27.26$\pm$1.49} & \underline{45.33$\pm$2.79} & \underline{22.46$\pm$1.43} & \underline{45.83$\pm$1.39} & \underline{58.89$\pm$1.19} & \underline{40.30$\pm$1.58} & \underline{30.51$\pm$2.46} & \textbf{62.03$\pm$1.34} \\
				eurlex-sm & \underline{64.34$\pm$1.46} & \underline{68.54$\pm$1.18} & \underline{35.50$\pm$1.54} & \underline{35.52$\pm$1.55} & \underline{29.83$\pm$1.14} & \underline{67.21$\pm$1.78} & \underline{31.49$\pm$1.39} & \underline{62.70$\pm$1.33} & \textbf{75.39$\pm$0.96} & \underline{61.01$\pm$1.68} & \underline{33.01$\pm$5.03} & 74.77$\pm$1.49 \\
				espgame & \underline{37.69$\pm$1.13} & \underline{31.35$\pm$1.42} & \underline{35.21$\pm$1.61} & \underline{35.45$\pm$1.53} & \underline{33.29$\pm$1.40} & \underline{37.47$\pm$1.62} & \underline{37.78$\pm$1.52} & \underline{33.83$\pm$1.28} & \underline{32.01$\pm$0.85} & \underline{36.20$\pm$1.01} & \textbf{43.57$\pm$1.24} & 41.89$\pm$1.36 \\
				stackex-chemistry & \underline{23.23$\pm$1.37} & \underline{28.89$\pm$1.18} & \underline{38.56$\pm$1.73} & \underline{38.40$\pm$1.87} & \underline{36.99$\pm$1.44} & \underline{35.57$\pm$1.57} & \underline{35.10$\pm$2.08} & \underline{38.73$\pm$1.83} & \underline{36.29$\pm$1.72} & \underline{40.30$\pm$1.50} & \underline{41.47$\pm$2.06} & \textbf{44.74$\pm$1.18} \\
				stackex-chess & \underline{29.38$\pm$3.98} & \underline{52.77$\pm$4.15} & \underline{44.06$\pm$3.39} & \underline{45.07$\pm$4.41} & \underline{41.49$\pm$2.62} & \underline{41.55$\pm$2.88} & \underline{39.82$\pm$3.48} & \underline{46.68$\pm$4.53} & \underline{44.24$\pm$3.96} & \underline{45.91$\pm$3.46} & \underline{49.73$\pm$3.07} & \textbf{55.40$\pm$4.77} \\
				stackex-coffee & \underline{41.40$\pm$12.06} & \underline{29.72$\pm$9.35} & \underline{22.31$\pm$10.80} & \underline{24.01$\pm$10.44} & \underline{25.45$\pm$14.38} & \underline{17.81$\pm$7.58} & \underline{24.90$\pm$14.01} & \underline{28.06$\pm$9.42} & \underline{29.88$\pm$11.79} & \underline{24.01$\pm$10.34} & \underline{34.19$\pm$8.26} & \textbf{45.32$\pm$10.70} \\
				stackex-cooking & \underline{22.93$\pm$0.98} & \underline{30.68$\pm$6.91} & \underline{44.10$\pm$1.69} & \underline{44.26$\pm$1.61} & \underline{42.61$\pm$1.73} & \underline{44.00$\pm$1.51} & \underline{37.82$\pm$1.25} & \underline{49.27$\pm$1.64} & \underline{45.76$\pm$1.02} & \underline{51.02$\pm$1.25} & \underline{55.55$\pm$1.65} & \textbf{57.46$\pm$0.88} \\
				stackex-cs & \underline{29.44$\pm$1.22} & \underline{43.79$\pm$1.49} & \underline{50.76$\pm$2.08} & \underline{50.70$\pm$2.12} & \underline{51.45$\pm$1.81} & \underline{48.28$\pm$2.02} & \underline{46.80$\pm$1.92} & \underline{51.31$\pm$2.08} & \underline{48.30$\pm$1.62} & \underline{53.68$\pm$2.12} & \underline{53.47$\pm$2.13} & \textbf{55.95$\pm$2.36} \\
				stackex-philosophy & \underline{22.59$\pm$2.21} & \underline{48.50$\pm$2.67} & \underline{41.15$\pm$2.50} & \underline{41.05$\pm$2.14} & \underline{38.30$\pm$2.57} & \underline{39.03$\pm$3.26} & \underline{37.98$\pm$2.23} & \underline{44.67$\pm$2.20} & \underline{43.54$\pm$1.91} & \underline{44.42$\pm$3.09} & \underline{49.99$\pm$2.05} & \textbf{53.64$\pm$2.78} \\
				\midrule
				Avg. Rank & 6.60  & 7.53  & 6.80  & 6.67  & 8.73  & 7.67  & 7.40  & 6.80  & 8.13  & 5.47  & 4.87  & \textbf{1.13}  \\
				\bottomrule
			\end{tabular}%
		}
	\end{table*}%

	\begin{table*}[t]
		\centering
		\scriptsize
		\setlength{\tabcolsep}{0.3mm}
		\caption{Experimental results (mean $\pm$ std) on the MLC benchmark datasets measured by P@$3$. Each underlined result indicates that SLDL is statistically superior to the comparing method. The best performance on each dataset is denoted in boldface.}
		\label{table:result-P@3}
		\resizebox{2\columnwidth}{!}{
			\begin{tabular}{lcccccccccccc}
				\toprule
				Dataset & SLEEC \cite{SLEEC} & DXML \cite{DXML} & FL \cite{focal} & CB \cite{CB} & DB \cite{DB} & PACA \cite{PACA} & FLEM \cite{FLEM} & DELA \cite{DELA} & CLIF \cite{CLIF} & KD-TEA \cite{KD-TEA} & AdaC2 \cite{AdaBoost.C2} & SLDL (Ours) \\
				\midrule
				cal500 & \underline{73.04$\pm$2.92} & \underline{73.50$\pm$2.76} & \underline{74.16$\pm$4.03} & \underline{74.95$\pm$4.92} & \underline{72.43$\pm$3.41} & \underline{74.89$\pm$3.55} & \underline{74.83$\pm$3.33} & \underline{72.97$\pm$4.15} & \underline{72.90$\pm$4.31} & \underline{75.30$\pm$2.93} & \underline{75.70$\pm$2.68} & \textbf{76.10$\pm$2.59} \\
				corel16k-s1 & \underline{27.37$\pm$0.91} & \underline{24.23$\pm$0.71} & \underline{25.23$\pm$0.45} & \underline{25.29$\pm$0.36} & \underline{22.95$\pm$0.43} & \underline{20.87$\pm$0.99} & \underline{27.32$\pm$0.43} & \underline{20.87$\pm$0.51} & \underline{20.46$\pm$0.64} & \underline{23.71$\pm$0.62} & \underline{27.47$\pm$0.61} & \textbf{28.39$\pm$0.67} \\
				corel16k-s2 & \underline{27.46$\pm$0.79} & \underline{24.26$\pm$0.58} & \underline{25.53$\pm$0.82} & \underline{25.51$\pm$0.71} & \underline{23.09$\pm$0.62} & \underline{21.51$\pm$0.85} & \underline{27.63$\pm$0.69} & \underline{21.27$\pm$0.45} & \underline{20.58$\pm$0.73} & \underline{23.92$\pm$0.89} & \underline{27.57$\pm$0.86} & \textbf{28.59$\pm$0.77} \\
				corel16k-s3 & \underline{27.32$\pm$1.05} & \underline{24.39$\pm$0.96} & \underline{25.22$\pm$0.99} & \underline{25.33$\pm$1.06} & \underline{22.84$\pm$0.95} & \underline{20.91$\pm$1.24} & \underline{27.50$\pm$0.99} & \underline{20.83$\pm$0.85} & \underline{20.25$\pm$0.85} & \underline{23.59$\pm$0.95} & \underline{27.35$\pm$0.61} & \textbf{28.36$\pm$0.86} \\
				CUB   & \underline{80.20$\pm$1.02} & \underline{74.48$\pm$0.97} & \underline{79.14$\pm$0.90} & \underline{79.42$\pm$0.96} & \underline{79.80$\pm$0.95} & \underline{79.95$\pm$0.99} & \underline{77.73$\pm$0.80} & \textbf{81.28$\pm$0.77} & \underline{78.32$\pm$1.08} & 80.47$\pm$1.06 & \underline{72.56$\pm$2.46} & 80.27$\pm$0.86 \\
				delicious & \underline{61.43$\pm$0.85} & \underline{53.45$\pm$0.98} & \underline{59.91$\pm$0.40} & \underline{59.86$\pm$0.51} & \underline{57.23$\pm$0.67} & \underline{59.92$\pm$0.91} & \underline{61.33$\pm$0.44} & \underline{59.66$\pm$0.68} & \underline{55.78$\pm$1.01} & \underline{61.75$\pm$0.43} & \underline{60.82$\pm$1.15} & \textbf{62.13$\pm$0.72} \\
				eurlex-dc & \underline{26.64$\pm$0.52} & \underline{26.56$\pm$0.38} & \underline{12.28$\pm$0.51} & \underline{12.26$\pm$0.52} & \underline{13.21$\pm$0.48} & \underline{21.32$\pm$1.05} & \underline{11.14$\pm$0.48} & \underline{22.35$\pm$0.59} & \underline{27.59$\pm$0.53} & \underline{20.00$\pm$0.55} & \underline{15.02$\pm$0.89} & \textbf{28.68$\pm$0.45} \\
				eurlex-sm & \underline{44.82$\pm$1.04} & \underline{46.48$\pm$0.81} & \underline{28.47$\pm$0.93} & \underline{28.45$\pm$0.90} & \underline{25.38$\pm$1.12} & \underline{44.34$\pm$1.11} & \underline{25.89$\pm$0.71} & \underline{41.63$\pm$0.99} & \underline{50.01$\pm$0.70} & \underline{40.32$\pm$1.10} & \underline{27.04$\pm$1.30} & \textbf{50.63$\pm$0.92} \\
				espgame & \underline{29.73$\pm$0.48} & \underline{24.47$\pm$0.69} & \underline{28.53$\pm$0.81} & \underline{28.65$\pm$0.76} & \underline{26.41$\pm$0.83} & \underline{28.62$\pm$1.13} & \underline{30.34$\pm$0.69} & \underline{27.03$\pm$0.55} & \underline{25.81$\pm$0.62} & \underline{29.22$\pm$0.62} & \textbf{34.57$\pm$0.78} & 33.76$\pm$0.31 \\
				stackex-chemistry & \underline{16.29$\pm$0.44} & \underline{19.49$\pm$0.78} & \underline{26.29$\pm$0.82} & \underline{26.33$\pm$0.70} & \underline{24.92$\pm$0.88} & \underline{23.16$\pm$0.76} & \underline{23.67$\pm$0.74} & \underline{25.80$\pm$0.69} & \underline{23.68$\pm$0.99} & \underline{26.58$\pm$0.63} & \underline{26.58$\pm$0.94} & \textbf{28.27$\pm$0.68} \\
				stackex-chess & \underline{19.58$\pm$1.28} & \underline{31.78$\pm$2.41} & \underline{27.62$\pm$1.74} & \underline{27.74$\pm$1.68} & \underline{27.20$\pm$1.55} & \underline{25.89$\pm$1.69} & \underline{25.51$\pm$1.50} & \underline{29.00$\pm$1.80} & \underline{28.58$\pm$1.89} & \underline{28.99$\pm$2.29} & \underline{31.00$\pm$1.31} & \textbf{32.93$\pm$2.53} \\
				stackex-coffee & \textbf{30.99$\pm$5.54} & \underline{19.98$\pm$3.71} & \underline{16.79$\pm$4.99} & \underline{16.05$\pm$4.40} & \underline{17.54$\pm$4.73} & \underline{13.37$\pm$3.93} & \underline{17.36$\pm$6.61} & \underline{22.42$\pm$7.16} & \underline{20.95$\pm$5.99} & \underline{16.78$\pm$3.75} & \underline{22.85$\pm$4.87} & 30.08$\pm$5.30 \\
				stackex-cooking & \underline{14.71$\pm$0.49} & \underline{19.48$\pm$3.45} & \underline{27.98$\pm$0.79} & \underline{27.99$\pm$0.72} & \underline{26.82$\pm$0.78} & \underline{25.73$\pm$0.87} & \underline{24.42$\pm$0.53} & \underline{29.11$\pm$0.62} & \underline{27.28$\pm$0.52} & \underline{31.45$\pm$0.90} & \textbf{32.99$\pm$0.69} & 31.99$\pm$0.41 \\
				stackex-cs & \underline{19.99$\pm$0.64} & \underline{30.44$\pm$0.68} & \underline{36.67$\pm$0.71} & \underline{36.76$\pm$0.67} & \underline{36.22$\pm$0.75} & \underline{33.34$\pm$0.96} & \underline{33.74$\pm$0.56} & \underline{36.53$\pm$0.87} & \underline{34.57$\pm$1.01} & \underline{38.08$\pm$0.61} & \underline{38.17$\pm$1.06} & \textbf{39.14$\pm$0.75} \\
				stackex-philosophy & \underline{15.69$\pm$0.86} & \underline{28.83$\pm$0.56} & \underline{25.95$\pm$0.61} & \underline{25.88$\pm$0.82} & \underline{25.15$\pm$0.94} & \underline{23.80$\pm$0.86} & \underline{23.13$\pm$0.83} & \underline{26.76$\pm$0.89} & \underline{26.88$\pm$1.02} & \underline{27.48$\pm$0.72} & \underline{29.88$\pm$0.68} & \textbf{31.25$\pm$0.72} \\
				\midrule
				Avg. Rank & 6.60  & 7.67  & 7.00  & 6.67  & 8.87  & 8.33  & 7.47  & 6.80  & 8.07  & 5.13  & 3.93  & \textbf{1.33}  \\
				\bottomrule
			\end{tabular}%
		}
	\end{table*}%

	\begin{table*}[t]
		\centering
		\scriptsize
		\setlength{\tabcolsep}{0.3mm}
		\caption{Experimental results (mean $\pm$ std) on the MLC benchmark datasets measured by P@$5$. Each underlined result indicates that SLDL is statistically superior to the comparing method. The best performance on each dataset is denoted in boldface.}
		\label{table:result-P@5}
		\resizebox{2\columnwidth}{!}{
			\begin{tabular}{lcccccccccccc}
				\toprule
				Dataset & SLEEC \cite{SLEEC} & DXML \cite{DXML} & FL \cite{focal} & CB \cite{CB} & DB \cite{DB} & PACA \cite{PACA} & FLEM \cite{FLEM} & DELA \cite{DELA} & CLIF \cite{CLIF} & KD-TEA \cite{KD-TEA} & AdaC2 \cite{AdaBoost.C2} & SLDL (Ours) \\
				\midrule
				cal500 & \underline{68.44$\pm$3.15} & \underline{67.80$\pm$3.51} & \underline{68.00$\pm$3.53} & \underline{68.17$\pm$3.30} & \underline{65.97$\pm$2.54} & \underline{67.87$\pm$3.16} & \underline{68.56$\pm$2.65} & \underline{67.17$\pm$3.15} & \underline{67.33$\pm$4.33} & \underline{69.56$\pm$2.47} & \underline{68.89$\pm$2.63} & \textbf{69.64$\pm$3.21} \\
				corel16k-s1 & \underline{22.24$\pm$0.61} & \underline{19.99$\pm$0.69} & \underline{20.82$\pm$0.33} & \underline{20.82$\pm$0.31} & \underline{19.15$\pm$0.41} & \underline{17.04$\pm$0.75} & \underline{22.45$\pm$0.42} & \underline{17.11$\pm$0.32} & \underline{16.69$\pm$0.53} & \underline{19.59$\pm$0.54} & \underline{22.30$\pm$0.44} & \textbf{23.40$\pm$0.42} \\
				corel16k-s2 & \underline{22.31$\pm$0.62} & \underline{19.94$\pm$0.47} & \underline{21.03$\pm$0.55} & \underline{21.02$\pm$0.58} & \underline{19.22$\pm$0.36} & \underline{17.69$\pm$0.54} & \underline{22.67$\pm$0.58} & \underline{17.33$\pm$0.42} & \underline{16.93$\pm$0.65} & \underline{19.71$\pm$0.65} & \underline{22.40$\pm$0.47} & \textbf{23.46$\pm$0.65} \\
				corel16k-s3 & \underline{22.31$\pm$0.60} & \underline{20.01$\pm$0.64} & \underline{21.03$\pm$0.63} & \underline{21.02$\pm$0.60} & \underline{19.19$\pm$0.67} & \underline{17.15$\pm$0.95} & \underline{22.56$\pm$0.71} & \underline{17.09$\pm$0.56} & \underline{16.67$\pm$0.56} & \underline{19.33$\pm$0.49} & \underline{22.23$\pm$0.57} & \textbf{23.37$\pm$0.58} \\
				CUB   & 76.56$\pm$0.94 & \underline{69.25$\pm$1.10} & \underline{75.26$\pm$1.03} & \underline{75.39$\pm$1.18} & \underline{75.93$\pm$0.96} & \underline{75.97$\pm$0.98} & \underline{73.35$\pm$0.71} & \textbf{77.59$\pm$0.72} & \underline{74.90$\pm$1.15} & 76.63$\pm$1.16 & \underline{71.40$\pm$1.40} & 76.45$\pm$0.83 \\
				delicious & \underline{56.59$\pm$0.82} & \underline{48.73$\pm$0.95} & \underline{55.36$\pm$0.47} & \underline{55.28$\pm$0.57} & \underline{52.78$\pm$0.55} & \underline{55.34$\pm$0.77} & \underline{56.62$\pm$0.50} & \underline{55.06$\pm$0.64} & \underline{51.51$\pm$0.78} & \underline{57.11$\pm$0.47} & \underline{56.51$\pm$0.76} & \textbf{57.64$\pm$0.68} \\
				eurlex-dc & \underline{17.71$\pm$0.35} & \underline{18.05$\pm$0.25} & \underline{8.71$\pm$0.33} & \underline{8.69$\pm$0.34} & \underline{9.08$\pm$0.28} & \underline{14.65$\pm$0.61} & \underline{7.89$\pm$0.30} & \underline{15.44$\pm$0.22} & \underline{18.43$\pm$0.30} & \underline{14.06$\pm$0.32} & \underline{10.41$\pm$0.45} & \textbf{18.68$\pm$0.30} \\
				eurlex-sm & \underline{31.52$\pm$0.64} & \underline{32.83$\pm$0.59} & \underline{19.49$\pm$0.66} & \underline{19.48$\pm$0.67} & \underline{19.88$\pm$0.68} & \underline{30.73$\pm$0.66} & \underline{17.66$\pm$0.57} & \underline{29.25$\pm$0.65} & \textbf{34.35$\pm$0.41} & \underline{28.64$\pm$0.62} & \underline{20.50$\pm$0.87} & 34.31$\pm$0.55 \\
				espgame & \underline{25.08$\pm$0.36} & \underline{20.83$\pm$0.40} & \underline{24.56$\pm$0.43} & \underline{24.58$\pm$0.46} & \underline{22.69$\pm$0.52} & \underline{24.00$\pm$0.61} & \underline{25.89$\pm$0.42} & \underline{22.75$\pm$0.44} & \underline{21.76$\pm$0.39} & \underline{25.06$\pm$0.42} & \textbf{28.93$\pm$0.42} & 28.37$\pm$0.40 \\
				stackex-chemistry & \underline{13.08$\pm$0.28} & \underline{15.25$\pm$0.63} & \underline{20.25$\pm$0.57} & \underline{20.19$\pm$0.63} & \underline{19.34$\pm$0.43} & \underline{17.69$\pm$0.56} & \underline{18.22$\pm$0.43} & \underline{19.54$\pm$0.60} & \underline{18.19$\pm$0.55} & \underline{20.44$\pm$0.49} & \underline{19.89$\pm$0.48} & \textbf{21.20$\pm$0.37} \\
				stackex-chess & \underline{15.58$\pm$1.02} & \underline{23.63$\pm$1.60} & \underline{20.94$\pm$0.98} & \underline{21.21$\pm$1.08} & \underline{20.64$\pm$1.21} & \underline{19.80$\pm$1.35} & \underline{19.84$\pm$1.03} & \underline{21.73$\pm$1.01} & \underline{21.42$\pm$1.18} & \underline{22.10$\pm$1.32} & \underline{22.85$\pm$1.11} & \textbf{24.24$\pm$1.38} \\
				stackex-coffee & \underline{22.40$\pm$2.38} & \underline{15.72$\pm$2.86} & \underline{12.55$\pm$3.23} & \underline{11.95$\pm$3.34} & \underline{13.29$\pm$3.37} & \underline{10.96$\pm$2.80} & \underline{13.26$\pm$4.06} & \underline{17.45$\pm$3.58} & \underline{15.68$\pm$3.58} & \underline{13.11$\pm$3.18} & \underline{17.36$\pm$3.09} & \textbf{22.67$\pm$2.97} \\
				stackex-cooking & \underline{11.30$\pm$0.29} & \underline{14.64$\pm$2.15} & \underline{20.48$\pm$0.35} & \underline{20.47$\pm$0.37} & \underline{19.75$\pm$0.38} & \underline{18.75$\pm$0.66} & \underline{18.30$\pm$0.34} & \underline{20.96$\pm$0.30} & \underline{19.66$\pm$0.45} & 22.76$\pm$0.64 & \textbf{23.15$\pm$0.39} & 22.59$\pm$0.23 \\
				stackex-cs & \underline{15.79$\pm$0.51} & \underline{22.99$\pm$0.62} & \underline{27.87$\pm$0.52} & \underline{27.78$\pm$0.49} & \underline{27.29$\pm$0.38} & \underline{25.29$\pm$0.59} & \underline{25.54$\pm$0.29} & \underline{27.80$\pm$0.59} & \underline{26.25$\pm$0.70} & \underline{28.74$\pm$0.37} & \underline{28.20$\pm$0.48} & \textbf{28.95$\pm$0.47} \\
				stackex-philosophy & \underline{12.94$\pm$0.60} & \underline{21.12$\pm$0.52} & \underline{19.54$\pm$0.50} & \underline{19.53$\pm$0.43} & \underline{19.43$\pm$0.49} & \underline{17.56$\pm$0.84} & \underline{17.56$\pm$0.47} & \underline{19.45$\pm$0.39} & \underline{19.66$\pm$0.54} & \underline{20.61$\pm$0.51} & \underline{21.75$\pm$0.41} & \textbf{22.69$\pm$0.56} \\
				\midrule
				Avg. Rank & 6.47  & 7.67  & 6.60  & 7.13  & 8.60  & 8.80  & 7.00  & 6.93  & 8.13  & 4.80  & 4.27  & \textbf{1.47}  \\
				\bottomrule
			\end{tabular}%
		}
	\end{table*}%

	\section{Experiments} \label{sec:experiments}
	
	In this section, we evaluate the efficiency and performance of SLDL across multiple multi-label classification datasets. All methods are implemented using PyTorch, and the experiments are performed on a GPU server equipped with an NVIDIA Tesla V100 GPU, AMD Ryzen 7 5800X processor CPU, and 32 GB GPU memory. To evaluate the proposed approach, we conduct experiments on fifteen benchmark datasets widely employed in multi-label classification research. The details of these datasets are provided in Table \ref{table:datasets}.

	\subsection{Comparing Algorithms}
	
	We compare our proposed methods with several state-of-the-art MLC technologies. The compared methods include:

	\begin{table*}[t]
		\centering
		\scriptsize
		\setlength{\tabcolsep}{0.3mm}
		\caption{Experimental results (mean $\pm$ std) on the MLC benchmark datasets measured by nDCG@$3$. Each underlined result indicates that SLDL is statistically superior to the comparing method. The best performance on each dataset is denoted in boldface.}
		\label{table:result-nDCG@3}
		\resizebox{2\columnwidth}{!}{
			\begin{tabular}{lcccccccccccc}
				\toprule
				Dataset & SLEEC \cite{SLEEC} & DXML \cite{DXML} & FL \cite{focal} & CB \cite{CB} & DB \cite{DB} & PACA \cite{PACA} & FLEM \cite{FLEM} & DELA \cite{DELA} & CLIF \cite{CLIF} & KD-TEA \cite{KD-TEA} & AdaC2 \cite{AdaBoost.C2} & SLDL (Ours) \\
				\midrule
				cal500 & \underline{75.98$\pm$2.48} & \underline{76.23$\pm$2.48} & \underline{77.18$\pm$3.32} & \underline{77.78$\pm$4.15} & \underline{75.17$\pm$3.22} & \underline{77.51$\pm$3.19} & \underline{77.79$\pm$2.89} & \underline{75.51$\pm$4.04} & \underline{75.44$\pm$4.00} & \underline{77.90$\pm$2.65} & \underline{78.01$\pm$2.61} & \textbf{78.30$\pm$2.62} \\
				corel16k-s1 & \underline{32.64$\pm$1.01} & \underline{28.75$\pm$0.75} & \underline{29.78$\pm$0.44} & \underline{29.81$\pm$0.36} & \underline{27.42$\pm$0.51} & \underline{25.02$\pm$1.24} & \underline{32.44$\pm$0.56} & \underline{24.72$\pm$0.60} & \underline{24.28$\pm$0.80} & \underline{28.04$\pm$0.89} & \underline{32.67$\pm$0.72} & \textbf{33.68$\pm$0.76} \\
				corel16k-s2 & \underline{32.40$\pm$0.90} & \underline{28.34$\pm$0.50} & \underline{29.81$\pm$0.79} & \underline{29.78$\pm$0.71} & \underline{27.36$\pm$0.81} & \underline{25.34$\pm$0.97} & \underline{32.43$\pm$0.76} & \underline{24.90$\pm$0.53} & \underline{24.09$\pm$0.85} & \underline{28.11$\pm$1.06} & \underline{32.54$\pm$1.02} & \textbf{33.39$\pm$0.89} \\
				corel16k-s3 & \underline{32.52$\pm$1.16} & \underline{28.79$\pm$1.14} & \underline{29.76$\pm$1.11} & \underline{29.87$\pm$1.10} & \underline{27.42$\pm$1.19} & \underline{25.02$\pm$1.44} & \underline{32.51$\pm$1.07} & \underline{24.77$\pm$0.86} & \underline{24.06$\pm$0.96} & \underline{28.18$\pm$1.17} & \underline{32.66$\pm$0.63} & \textbf{33.60$\pm$1.14} \\
				CUB   & \underline{81.52$\pm$1.06} & \underline{76.65$\pm$0.86} & \underline{80.77$\pm$0.84} & \underline{81.03$\pm$0.94} & \underline{81.24$\pm$1.00} & \underline{81.37$\pm$0.93} & \underline{79.39$\pm$0.74} & \textbf{82.59$\pm$0.81} & \underline{79.60$\pm$1.06} & 81.90$\pm$1.07 & \underline{72.37$\pm$3.02} & 81.80$\pm$0.89 \\
				delicious & \underline{62.88$\pm$0.85} & \underline{54.96$\pm$1.00} & \underline{61.28$\pm$0.48} & \underline{61.23$\pm$0.56} & \underline{58.52$\pm$0.63} & \underline{61.36$\pm$0.88} & \underline{62.74$\pm$0.42} & \underline{60.97$\pm$0.68} & \underline{57.10$\pm$0.94} & \underline{63.13$\pm$0.49} & \underline{62.00$\pm$1.14} & \textbf{63.39$\pm$0.74} \\
				eurlex-dc & \underline{60.80$\pm$1.17} & \underline{60.81$\pm$0.87} & \underline{29.90$\pm$1.27} & \underline{29.84$\pm$1.29} & \underline{31.87$\pm$1.28} & \underline{49.90$\pm$2.63} & \underline{26.74$\pm$1.22} & \underline{51.54$\pm$1.28} & \underline{63.87$\pm$1.02} & \underline{46.32$\pm$1.32} & \underline{36.01$\pm$1.91} & \textbf{66.64$\pm$1.04} \\
				eurlex-sm & \underline{62.85$\pm$1.23} & \underline{66.09$\pm$0.86} & \underline{35.89$\pm$1.07} & \underline{35.87$\pm$1.06} & \underline{34.08$\pm$1.19} & \underline{63.36$\pm$1.64} & \underline{32.61$\pm$0.91} & \underline{58.72$\pm$1.10} & \underline{72.05$\pm$0.69} & \underline{57.27$\pm$1.34} & \underline{36.16$\pm$1.95} & \textbf{72.30$\pm$1.04} \\
				espgame & \underline{32.55$\pm$0.58} & \underline{26.81$\pm$0.86} & \underline{30.95$\pm$1.08} & \underline{31.10$\pm$1.01} & \underline{28.94$\pm$1.00} & \underline{31.61$\pm$1.33} & \underline{33.00$\pm$0.96} & \underline{29.56$\pm$0.75} & \underline{28.15$\pm$0.72} & \underline{31.88$\pm$0.77} & \textbf{37.90$\pm$0.96} & 36.90$\pm$0.50 \\
				stackex-chemistry & \underline{23.92$\pm$0.62} & \underline{28.70$\pm$0.93} & \underline{38.97$\pm$0.97} & \underline{38.93$\pm$0.96} & \underline{37.19$\pm$1.11} & \underline{34.74$\pm$1.07} & \underline{35.16$\pm$1.23} & \underline{38.60$\pm$0.90} & \underline{35.52$\pm$1.35} & \underline{39.86$\pm$0.94} & \underline{40.38$\pm$1.29} & \textbf{42.82$\pm$0.60} \\
				stackex-chess & \underline{26.80$\pm$2.42} & \underline{45.57$\pm$3.47} & \underline{38.97$\pm$2.87} & \underline{39.42$\pm$3.19} & \underline{38.07$\pm$2.39} & \underline{36.34$\pm$2.18} & \underline{35.73$\pm$2.67} & \underline{41.12$\pm$2.89} & \underline{39.91$\pm$2.94} & \underline{40.72$\pm$3.11} & \underline{44.28$\pm$2.39} & \textbf{47.07$\pm$4.16} \\
				stackex-coffee & \textbf{49.07$\pm$10.20} & \underline{32.39$\pm$7.94} & \underline{26.30$\pm$9.30} & \underline{26.41$\pm$7.53} & \underline{27.79$\pm$8.13} & \underline{20.21$\pm$6.55} & \underline{27.92$\pm$10.56} & \underline{34.94$\pm$10.84} & \underline{33.31$\pm$9.31} & \underline{26.62$\pm$7.11} & \underline{37.69$\pm$7.27} & 48.35$\pm$9.28 \\
				stackex-cooking & \underline{21.15$\pm$0.56} & \underline{28.13$\pm$5.47} & \underline{40.30$\pm$0.97} & \underline{40.32$\pm$1.00} & \underline{38.82$\pm$1.18} & \underline{38.22$\pm$1.33} & \underline{34.92$\pm$0.66} & \underline{43.08$\pm$1.00} & \underline{40.18$\pm$0.98} & \underline{45.81$\pm$1.23} & \textbf{48.69$\pm$1.05} & 48.16$\pm$0.75 \\
				stackex-cs & \underline{25.64$\pm$1.01} & \underline{39.02$\pm$0.92} & \underline{46.98$\pm$1.13} & \underline{47.00$\pm$0.95} & \underline{46.96$\pm$1.08} & \underline{43.26$\pm$1.35} & \underline{42.93$\pm$1.04} & \underline{47.37$\pm$1.31} & \underline{44.65$\pm$1.37} & \underline{49.30$\pm$1.18} & \underline{49.37$\pm$1.63} & \textbf{50.63$\pm$1.76} \\
				stackex-philosophy & \underline{23.40$\pm$1.61} & \underline{44.57$\pm$1.48} & \underline{39.15$\pm$1.53} & \underline{39.09$\pm$1.66} & \underline{37.73$\pm$1.69} & \underline{36.06$\pm$1.90} & \underline{35.53$\pm$1.66} & \underline{40.76$\pm$1.86} & \underline{40.37$\pm$1.55} & \underline{41.52$\pm$1.77} & \underline{45.92$\pm$1.24} & \textbf{48.16$\pm$0.93} \\
				\midrule
				Avg. Rank & 6.67  & 7.60  & 7.20  & 6.93  & 8.87  & 8.20  & 7.60  & 6.67  & 8.13  & 5.13  & 3.67  & \textbf{1.33}  \\
				\bottomrule
			\end{tabular}%
		}
	\end{table*}%

	\begin{table*}[t]
		\centering
		\scriptsize
		\setlength{\tabcolsep}{0.3mm}
		\caption{Experimental results (mean $\pm$ std) on the MLC benchmark datasets measured by nDCG@$5$. Each underlined result indicates that SLDL is statistically superior to the comparing method. The best performance on each dataset is denoted in boldface.}
		\label{table:result-nDCG@5}
		\resizebox{2\columnwidth}{!}{
			\begin{tabular}{lcccccccccccc}
				\toprule
				Dataset & SLEEC \cite{SLEEC} & DXML \cite{DXML} & FL \cite{focal} & CB \cite{CB} & DB \cite{DB} & PACA \cite{PACA} & FLEM \cite{FLEM} & DELA \cite{DELA} & CLIF \cite{CLIF} & KD-TEA \cite{KD-TEA} & AdaC2 \cite{AdaBoost.C2} & SLDL (Ours) \\
				\midrule
				cal500 & \underline{72.00$\pm$2.61} & \underline{71.56$\pm$2.98} & \underline{72.08$\pm$3.15} & \underline{72.33$\pm$3.32} & \underline{69.98$\pm$2.61} & \underline{71.97$\pm$2.80} & \underline{72.63$\pm$2.44} & \underline{70.77$\pm$3.17} & \underline{70.90$\pm$3.98} & \underline{73.23$\pm$2.39} & \underline{72.67$\pm$2.60} & \textbf{73.33$\pm$2.65} \\
				corel16k-s1 & \underline{36.63$\pm$0.93} & \underline{32.56$\pm$0.97} & \underline{33.78$\pm$0.47} & \underline{33.78$\pm$0.42} & \underline{31.34$\pm$0.56} & \underline{28.16$\pm$1.30} & \underline{36.70$\pm$0.62} & \underline{27.91$\pm$0.57} & \underline{27.34$\pm$0.91} & \underline{31.78$\pm$0.86} & \underline{36.65$\pm$0.72} & \textbf{38.16$\pm$0.56} \\
				corel16k-s2 & \underline{36.07$\pm$1.01} & \underline{31.83$\pm$0.57} & \underline{33.52$\pm$0.77} & \underline{33.47$\pm$0.77} & \underline{30.96$\pm$0.67} & \underline{28.43$\pm$0.94} & \underline{36.34$\pm$0.87} & \underline{27.80$\pm$0.60} & \underline{27.04$\pm$0.99} & \underline{31.60$\pm$1.08} & \underline{36.24$\pm$0.94} & \textbf{37.42$\pm$1.05} \\
				corel16k-s3 & \underline{36.43$\pm$1.11} & \underline{32.41$\pm$1.13} & \underline{33.79$\pm$1.09} & \underline{33.83$\pm$1.01} & \underline{31.30$\pm$1.25} & \underline{28.10$\pm$1.50} & \underline{36.53$\pm$1.13} & \underline{27.82$\pm$0.86} & \underline{27.06$\pm$0.89} & \underline{31.66$\pm$1.04} & \underline{36.49$\pm$0.85} & \textbf{37.80$\pm$1.15} \\
				CUB   & \underline{78.68$\pm$0.99} & \underline{72.46$\pm$0.94} & \underline{77.66$\pm$0.95} & \underline{77.81$\pm$1.09} & \underline{78.20$\pm$1.00} & \underline{78.25$\pm$0.92} & \underline{75.93$\pm$0.70} & \textbf{79.71$\pm$0.75} & \underline{76.90$\pm$1.05} & 78.88$\pm$1.08 & \underline{71.66$\pm$2.09} & 78.77$\pm$0.86 \\
				delicious & \underline{59.30$\pm$0.82} & \underline{51.39$\pm$0.96} & \underline{57.91$\pm$0.49} & \underline{57.84$\pm$0.56} & \underline{55.27$\pm$0.56} & \underline{57.96$\pm$0.77} & \underline{59.26$\pm$0.45} & \underline{57.57$\pm$0.58} & \underline{53.94$\pm$0.79} & \underline{59.70$\pm$0.50} & \underline{58.86$\pm$0.84} & \textbf{60.09$\pm$0.70} \\
				eurlex-dc & \underline{63.46$\pm$1.24} & \underline{64.07$\pm$0.84} & \underline{32.30$\pm$1.28} & \underline{32.23$\pm$1.33} & \underline{33.90$\pm$1.24} & \underline{52.88$\pm$2.60} & \underline{29.00$\pm$1.19} & \underline{54.74$\pm$1.03} & \underline{66.75$\pm$0.95} & \underline{49.67$\pm$1.21} & \underline{38.42$\pm$1.76} & \textbf{68.81$\pm$1.04} \\
				eurlex-sm & \underline{66.76$\pm$1.15} & \underline{70.25$\pm$0.82} & \underline{37.52$\pm$1.18} & \underline{37.52$\pm$1.19} & \underline{38.56$\pm$1.19} & \underline{66.69$\pm$1.52} & \underline{34.04$\pm$1.01} & \underline{62.34$\pm$1.12} & \textbf{75.35$\pm$0.53} & \underline{61.19$\pm$1.21} & \underline{40.13$\pm$1.52} & 75.18$\pm$0.97 \\
				espgame & \underline{31.46$\pm$0.52} & \underline{26.00$\pm$0.67} & \underline{30.38$\pm$0.86} & \underline{30.44$\pm$0.85} & \underline{28.41$\pm$0.90} & \underline{30.53$\pm$1.02} & \underline{32.16$\pm$0.77} & \underline{28.63$\pm$0.73} & \underline{27.34$\pm$0.57} & \underline{31.30$\pm$0.61} & \textbf{36.67$\pm$0.75} & 35.78$\pm$0.53 \\
				stackex-chemistry & \underline{27.40$\pm$0.66} & \underline{32.35$\pm$0.96} & \underline{43.44$\pm$1.02} & \underline{43.34$\pm$1.07} & \underline{41.60$\pm$0.98} & \underline{38.47$\pm$1.25} & \underline{39.30$\pm$1.11} & \underline{42.61$\pm$1.01} & \underline{39.51$\pm$1.20} & \underline{44.30$\pm$0.97} & \underline{44.27$\pm$1.12} & \textbf{46.72$\pm$0.83} \\
				stackex-chess & \underline{29.60$\pm$2.49} & \underline{48.36$\pm$2.99} & \underline{41.75$\pm$2.35} & \underline{42.43$\pm$2.96} & \underline{40.94$\pm$2.57} & \underline{38.93$\pm$1.87} & \underline{39.04$\pm$2.37} & \underline{43.84$\pm$2.39} & \underline{42.54$\pm$2.61} & \underline{43.93$\pm$2.60} & \underline{46.72$\pm$1.93} & \textbf{49.40$\pm$3.55} \\
				stackex-coffee & \textbf{54.26$\pm$8.60} & \underline{37.42$\pm$8.56} & \underline{29.87$\pm$8.98} & \underline{29.71$\pm$8.43} & \underline{31.35$\pm$8.62} & \underline{24.59$\pm$6.80} & \underline{31.70$\pm$10.61} & \underline{40.23$\pm$9.80} & \underline{37.21$\pm$9.37} & \underline{30.54$\pm$7.74} & \underline{42.58$\pm$7.24} & 53.10$\pm$8.71 \\
				stackex-cooking & \underline{23.47$\pm$0.51} & \underline{30.78$\pm$5.46} & \underline{43.45$\pm$0.84} & \underline{43.47$\pm$0.92} & \underline{42.02$\pm$1.12} & \underline{41.03$\pm$1.44} & \underline{38.22$\pm$0.73} & \underline{45.94$\pm$0.86} & \underline{42.92$\pm$1.17} & \underline{49.06$\pm$1.28} & \textbf{51.20$\pm$0.96} & 50.58$\pm$0.77 \\
				stackex-cs & \underline{28.34$\pm$1.09} & \underline{42.14$\pm$0.98} & \underline{50.91$\pm$1.09} & \underline{50.79$\pm$0.99} & \underline{50.55$\pm$0.83} & \underline{46.72$\pm$1.14} & \underline{46.38$\pm$0.95} & \underline{51.30$\pm$1.20} & \underline{48.30$\pm$1.40} & \underline{53.18$\pm$1.01} & \underline{52.64$\pm$1.32} & \textbf{53.88$\pm$1.65} \\
				stackex-philosophy & \underline{26.83$\pm$1.52} & \underline{47.30$\pm$1.28} & \underline{42.36$\pm$1.45} & \underline{42.28$\pm$1.37} & \underline{41.32$\pm$1.39} & \underline{38.51$\pm$2.07} & \underline{38.47$\pm$1.38} & \underline{43.18$\pm$1.55} & \underline{42.95$\pm$1.26} & \underline{44.74$\pm$1.51} & \underline{48.56$\pm$0.96} & \textbf{50.87$\pm$1.14} \\
				\midrule
				Avg. Rank & 6.60  & 7.60  & 7.00  & 7.07  & 8.73  & 8.47  & 7.20  & 6.80  & 8.07  & 4.87  & 4.07  & \textbf{1.40}  \\
				\bottomrule
			\end{tabular}%
		}
	\end{table*}%

	\begin{itemize}
		\item \textit{Sparse local embeddings for extreme classification} (SLEEC) \cite{SLEEC}: a multi-label classification method specifically proposed for MLC tasks with large-scale output space. SLEEC aims to embed labels by preserving the pairwise distances between a few nearest label neighbors. In the experiments, the embedding dimension is set to 100, the number of neighbors for embedding is set to 250, and the number of neighbors for prediction is set to 25 as suggested in the code provided by the authors.
		\item \textit{Deep extreme multi-label learning} (DXML) \cite{DXML}: a deep embedding method specifically proposed for multi-label classification tasks with large-scale output space. It integrates the concepts of non-linear embedding and graph priors-based label space modeling simultaneously. In the experiments, the embedding dimension is set to 300 as suggested by the original paper.
		\item \textit{Focal loss reweighting} (FL) \cite{focal}: a simple yet commonly used approach in classification tasks. It elevates the loss weight for instances that are hard to classify, particularly those predicted with low probability on the ground-truth. We apply a balance parameter of 2 and set $\gamma=2$ for focal loss, following the suggestion in \cite{DB}.
		\item \textit{Class-balanced loss reweighting} (CB) \cite{CB}: a balancing method that adjusts weights for each class based on their effective number, represented as $E_n=(1-\beta^n)/(1-\beta)$. It further reweights focal loss, capturing diminishing marginal benefits of data and consequently reducing redundant information from head classes.
		\item \textit{Distribution balance loss reweighting} (DB) \cite{DB}: a specialized balancing method crafted for multi-label classification tasks. It strategically mitigates label co-occurrence redundancy, crucial in the multi-label scenario, and selectively assigns lower weights to ``easy-to-classify'' negative instances through a combination of rebalanced weighting and negative-tolerant regularization.
		\item \textit{Probabilistic label-specific feature learning} (PACA) \cite{PACA}: a multi-label classification method that transforms features based on label-specific prototypes in an end-to-end manner.  In the experiments, $\alpha$ and $\gamma$ are set to 1 as suggested in the code provided by the authors.
		\item \textit{Fusion label enhancement for multi-label learning} (FLEM) \cite{FLEM}: a label-enhancement-based multi-label classification method that integrates the label enhancement process and the multi-label classification training process. In the experiments, $\alpha$ and $\beta$ are set to 0.01 as suggested in the official code.
		\item \textit{Dual perspective of label-specific feature learning} (DELA) \cite{DELA}: a label-specific feature learning approach with a dual perspective for multi-label classification. In the experiments, $\beta$ is set to 1 as suggested in the code provided by the authors.
		\item \textit{Collaborative learning of label semantics and deep label-specific features} (CLIF) \cite{CLIF}: a multi-label classification method that learns label semantics and label-specific features in a collaborative way. In the experiments, the maximum number of iterations is set to 100, $\lambda$ is set to 0.01 as suggested in the code provided by the authors.
		\item \textit{Target Embedding Autoencoder framework based on Knowledge Distillation} (KD-TEA) \cite{KD-TEA}: a multi-label classification method based on target embedding autoencoder and knowledge distillation. It compresses a Teacher model with large parameters into a small Student model through knowledge distillation. In the experiments, the lever parameter $\alpha$ is set to 0.5 as suggested by the original paper.
		\item \textit{AdaBoost.C2} (AdaC2) \cite{AdaBoost.C2}: a multi-path AdaBoost framework specific to MLC, where each boosting path is established for distinct labels and the combination of them is able to provide a maximum optimization to Hamming loss. In the experiments, the maximum number of iterations is set to 10, and the hyperparameter $\delta$ is set to 0.01 as suggested by the original paper.
	\end{itemize}
	
	\begin{table*}[t]
		\centering
		\scriptsize
		\setlength{\tabcolsep}{0.42mm}
		\caption{Experimental results (mean $\pm$ std) on the MLC benchmark datasets measured by PSP@$1$ (PSnDCG@$1$). Each underlined result indicates that SLDL is statistically superior to the comparing method. The best performance on each dataset is denoted in boldface.}
		\label{table:result-PSP@1}
		\resizebox{2\columnwidth}{!}{
			\begin{tabular}{lcccccccccccc}
				\toprule
				Dataset & SLEEC \cite{SLEEC} & DXML \cite{DXML} & FL \cite{focal} & CB \cite{CB} & DB \cite{DB} & PACA \cite{PACA} & FLEM \cite{FLEM} & DELA \cite{DELA} & CLIF \cite{CLIF} & KD-TEA \cite{KD-TEA} & AdaC2 \cite{AdaBoost.C2} & SLDL (Ours) \\
				\midrule
				cal500 & \underline{37.64$\pm$1.61} & \underline{37.48$\pm$1.34} & \underline{37.54$\pm$1.55} & \underline{37.56$\pm$1.75} & \underline{36.65$\pm$1.68} & \underline{37.79$\pm$1.45} & \underline{37.75$\pm$1.46} & \underline{36.67$\pm$1.75} & \underline{36.58$\pm$1.92} & \underline{37.35$\pm$1.59} & \underline{37.33$\pm$2.46} & \textbf{38.31$\pm$1.20} \\
				corel16k-s1 & \underline{24.64$\pm$0.98} & \underline{20.66$\pm$0.64} & \underline{22.55$\pm$0.56} & \underline{22.53$\pm$0.66} & \underline{20.73$\pm$0.48} & \underline{19.55$\pm$1.20} & \underline{24.17$\pm$0.83} & \underline{19.66$\pm$0.57} & \underline{18.81$\pm$0.87} & \underline{20.70$\pm$1.12} & \underline{25.03$\pm$0.79} & \textbf{25.10$\pm$0.91} \\
				corel16k-s2 & \underline{24.15$\pm$0.85} & \underline{20.09$\pm$0.76} & \underline{22.57$\pm$0.66} & \underline{22.44$\pm$0.69} & \underline{20.53$\pm$1.08} & \underline{19.31$\pm$1.10} & \underline{24.04$\pm$0.95} & \underline{19.52$\pm$0.73} & \underline{18.53$\pm$0.87} & \underline{20.63$\pm$1.03} & \underline{24.53$\pm$1.20} & \textbf{24.66$\pm$0.78} \\
				corel16k-s3 & \underline{24.28$\pm$0.96} & \underline{20.64$\pm$0.94} & \underline{22.16$\pm$1.21} & \underline{22.25$\pm$1.15} & \underline{20.26$\pm$1.26} & \underline{19.53$\pm$1.03} & \underline{23.51$\pm$0.89} & \underline{19.39$\pm$0.92} & \underline{18.59$\pm$1.02} & \underline{21.09$\pm$1.10} & \textbf{24.92$\pm$0.58} & 24.81$\pm$1.05 \\
				CUB   & 58.16$\pm$1.23 & \underline{55.76$\pm$0.78} & \underline{56.60$\pm$0.73} & \underline{56.71$\pm$0.93} & \underline{57.63$\pm$1.09} & 58.27$\pm$0.87 & \underline{54.99$\pm$0.62} & \textbf{58.33$\pm$1.09} & \underline{57.25$\pm$0.85} & \underline{57.86$\pm$1.22} & \underline{50.58$\pm$3.92} & 58.10$\pm$0.85 \\
				delicious & \underline{33.17$\pm$0.54} & \underline{28.53$\pm$0.54} & \underline{33.46$\pm$0.62} & \underline{33.41$\pm$0.59} & 33.79$\pm$0.49 & \underline{32.82$\pm$0.55} & \underline{33.71$\pm$0.38} & \underline{33.08$\pm$0.45} & \underline{31.21$\pm$0.40} & \underline{32.62$\pm$0.58} & \textbf{35.59$\pm$0.83} & 33.72$\pm$0.74 \\
				eurlex-dc & \underline{41.46$\pm$1.34} & \underline{40.44$\pm$0.71} & \underline{15.66$\pm$0.88} & \underline{15.60$\pm$0.89} & \underline{16.33$\pm$0.95} & \underline{30.47$\pm$2.36} & \underline{13.04$\pm$0.84} & \underline{32.10$\pm$1.16} & \underline{44.86$\pm$1.14} & \underline{26.73$\pm$1.19} & \underline{20.41$\pm$1.54} & \textbf{48.68$\pm$1.21} \\
				eurlex-sm & \underline{48.42$\pm$1.11} & \underline{51.37$\pm$0.82} & \underline{23.46$\pm$1.08} & \underline{23.47$\pm$1.08} & \underline{20.46$\pm$0.85} & \underline{49.06$\pm$1.61} & \underline{20.55$\pm$0.94} & \underline{45.10$\pm$0.94} & \textbf{58.26$\pm$0.70} & \underline{43.27$\pm$1.38} & \underline{24.44$\pm$3.17} & 57.36$\pm$1.17 \\
				espgame & \underline{24.70$\pm$0.77} & \underline{19.75$\pm$0.95} & \underline{24.14$\pm$1.22} & \underline{24.29$\pm$1.08} & \underline{23.94$\pm$1.08} & \underline{25.79$\pm$0.99} & \underline{24.81$\pm$1.05} & \underline{23.90$\pm$0.97} & \underline{22.53$\pm$0.62} & \underline{23.78$\pm$0.74} & \textbf{31.03$\pm$0.91} & 28.06$\pm$0.80 \\
				stackex-chemistry & \underline{14.06$\pm$0.84} & \underline{17.69$\pm$0.90} & \underline{24.02$\pm$1.22} & \underline{23.94$\pm$1.26} & \underline{23.02$\pm$1.14} & \underline{24.53$\pm$0.93} & \underline{20.65$\pm$1.24} & \underline{27.70$\pm$1.28} & \underline{25.58$\pm$1.19} & \underline{26.91$\pm$1.14} & \underline{29.04$\pm$1.83} & \textbf{30.30$\pm$0.70} \\
				stackex-chess & \underline{12.61$\pm$2.09} & \underline{26.55$\pm$2.27} & \underline{20.25$\pm$1.60} & \underline{20.89$\pm$2.32} & \underline{18.86$\pm$1.84} & \underline{19.96$\pm$2.01} & \underline{17.57$\pm$1.70} & \underline{24.00$\pm$2.37} & \underline{22.79$\pm$2.44} & \underline{21.73$\pm$2.08} & \underline{27.00$\pm$2.39} & \textbf{28.65$\pm$2.87} \\
				stackex-coffee & \underline{28.75$\pm$8.49} & \underline{19.85$\pm$6.99} & \underline{15.46$\pm$8.03} & \underline{16.71$\pm$7.32} & \underline{17.45$\pm$9.82} & \underline{10.01$\pm$4.57} & \underline{17.47$\pm$9.18} & \underline{20.00$\pm$6.95} & \underline{22.05$\pm$9.11} & \underline{14.51$\pm$6.80} & \underline{27.67$\pm$7.19} & \textbf{32.92$\pm$8.27} \\
				stackex-cooking & \underline{15.12$\pm$0.71} & \underline{18.46$\pm$5.30} & \underline{27.06$\pm$0.98} & \underline{27.15$\pm$0.94} & \underline{26.26$\pm$1.15} & \underline{29.57$\pm$1.28} & \underline{21.51$\pm$0.78} & \underline{33.93$\pm$0.93} & \underline{31.86$\pm$0.75} & \underline{32.83$\pm$0.76} & \underline{38.58$\pm$1.28} & \textbf{39.06$\pm$0.80} \\
				stackex-cs & \underline{15.88$\pm$0.78} & \underline{24.33$\pm$1.07} & \underline{28.92$\pm$1.38} & \underline{28.88$\pm$1.37} & \underline{30.21$\pm$1.17} & \underline{30.21$\pm$1.09} & \underline{25.11$\pm$1.34} & \underline{34.33$\pm$1.59} & \underline{31.78$\pm$1.15} & \underline{32.83$\pm$1.41} & \textbf{35.21$\pm$1.48} & 34.47$\pm$1.64 \\
				stackex-philosophy & \underline{11.06$\pm$1.21} & \underline{27.21$\pm$1.56} & \underline{21.29$\pm$1.51} & \underline{21.26$\pm$1.33} & \underline{19.74$\pm$1.59} & \underline{22.19$\pm$2.39} & \underline{18.81$\pm$1.26} & \underline{26.97$\pm$1.73} & \underline{26.27$\pm$1.69} & \underline{24.52$\pm$2.01} & \underline{30.93$\pm$1.62} & \textbf{31.50$\pm$1.68} \\
				\midrule
				Avg. Rank & 6.53  & 7.93  & 7.47  & 7.40  & 8.27  & 6.93  & 7.73  & 6.13  & 7.27  & 6.93  & 3.73  & \textbf{1.60}  \\
				\bottomrule
			\end{tabular}%
		}
	\end{table*}%

	\begin{table*}[t]
		\centering
		\scriptsize
		\setlength{\tabcolsep}{0.28mm}
		\caption{Experimental results (mean $\pm$ std) on the MLC benchmark datasets measured by PSP@$5$. Each underlined result indicates that SLDL is statistically superior to the comparing method. The best performance on each dataset is denoted in boldface.}
		\label{table:result-PSP@5}
		\resizebox{2\columnwidth}{!}{
			\begin{tabular}{lcccccccccccc}
				\toprule
				Dataset & SLEEC \cite{SLEEC} & DXML \cite{DXML} & FL \cite{focal} & CB \cite{CB} & DB \cite{DB} & PACA \cite{PACA} & FLEM \cite{FLEM} & DELA \cite{DELA} & CLIF \cite{CLIF} & KD-TEA \cite{KD-TEA} & AdaC2 \cite{AdaBoost.C2} & SLDL (Ours) \\
				\midrule
				cal500 & \underline{39.30$\pm$1.83} & \underline{38.80$\pm$1.99} & \underline{38.53$\pm$2.04} & \underline{38.61$\pm$1.94} & \underline{37.68$\pm$1.41} & \underline{38.65$\pm$1.88} & \underline{38.67$\pm$1.68} & \underline{38.42$\pm$1.67} & \underline{38.64$\pm$2.58} & \underline{39.60$\pm$1.43} & \underline{39.45$\pm$1.72} & \textbf{39.89$\pm$1.76} \\
				corel16k-s1 & \underline{33.00$\pm$0.89} & \underline{28.37$\pm$1.07} & \underline{30.84$\pm$0.53} & \underline{30.85$\pm$0.53} & \underline{28.34$\pm$0.72} & \underline{25.65$\pm$1.13} & \underline{32.77$\pm$0.54} & \underline{26.64$\pm$0.54} & \underline{25.79$\pm$0.84} & \underline{29.06$\pm$0.75} & \underline{32.62$\pm$0.60} & \textbf{34.08$\pm$0.71} \\
				corel16k-s2 & \underline{32.78$\pm$1.16} & \underline{27.88$\pm$0.88} & \underline{30.77$\pm$1.09} & \underline{30.76$\pm$1.13} & \underline{28.18$\pm$0.85} & \underline{26.24$\pm$0.86} & \underline{32.72$\pm$1.14} & \underline{26.71$\pm$0.85} & \underline{25.92$\pm$1.17} & \underline{28.91$\pm$1.11} & \underline{32.46$\pm$1.02} & \textbf{34.01$\pm$1.33} \\
				corel16k-s3 & \underline{33.47$\pm$1.01} & \underline{28.72$\pm$0.99} & \underline{31.42$\pm$1.02} & \underline{31.40$\pm$0.93} & \underline{28.74$\pm$0.99} & \underline{26.18$\pm$1.29} & \underline{33.19$\pm$1.17} & \underline{26.88$\pm$0.89} & \underline{26.04$\pm$0.90} & \underline{28.99$\pm$0.65} & \underline{32.98$\pm$0.89} & \textbf{34.34$\pm$0.83} \\
				CUB   & 59.88$\pm$0.91 & \underline{52.67$\pm$0.93} & \underline{57.29$\pm$0.91} & \underline{57.44$\pm$1.07} & \underline{58.92$\pm$0.94} & 59.39$\pm$1.00 & \underline{55.03$\pm$0.67} & \textbf{60.32$\pm$0.92} & \underline{58.71$\pm$1.14} & \underline{59.00$\pm$1.21} & \underline{56.40$\pm$1.49} & 59.18$\pm$0.78 \\
				delicious & \underline{34.84$\pm$0.63} & \underline{28.93$\pm$0.54} & \underline{34.73$\pm$0.45} & \underline{34.70$\pm$0.52} & \underline{34.64$\pm$0.53} & \underline{34.61$\pm$0.61} & \underline{35.02$\pm$0.41} & \underline{34.87$\pm$0.52} & \underline{32.93$\pm$0.50} & \underline{34.82$\pm$0.36} & \textbf{36.84$\pm$0.61} & 35.85$\pm$0.56 \\
				eurlex-dc & \underline{60.18$\pm$1.61} & \underline{60.92$\pm$1.33} & \underline{22.28$\pm$1.11} & \underline{22.20$\pm$1.14} & \underline{23.69$\pm$1.00} & \underline{44.21$\pm$2.25} & \underline{19.39$\pm$0.96} & \underline{48.45$\pm$1.10} & \underline{62.30$\pm$1.37} & \underline{42.06$\pm$1.25} & \underline{29.86$\pm$1.39} & \textbf{64.77$\pm$1.19} \\
				eurlex-sm & \underline{65.97$\pm$1.01} & \underline{68.62$\pm$0.91} & \underline{35.37$\pm$1.15} & \underline{35.35$\pm$1.17} & \underline{36.41$\pm$1.24} & \underline{62.02$\pm$1.44} & \underline{31.41$\pm$1.02} & \underline{58.40$\pm$1.20} & \underline{72.45$\pm$0.49} & \underline{56.63$\pm$1.19} & \underline{39.00$\pm$1.79} & \textbf{73.12$\pm$0.83} \\
				espgame & \underline{26.04$\pm$0.41} & \underline{20.82$\pm$0.53} & \underline{26.07$\pm$0.53} & \underline{26.10$\pm$0.58} & \underline{24.57$\pm$0.56} & \underline{25.99$\pm$0.65} & \underline{26.68$\pm$0.50} & \underline{25.46$\pm$0.57} & \underline{24.34$\pm$0.44} & \underline{26.14$\pm$0.54} & \textbf{31.22$\pm$0.54} & 30.19$\pm$0.37 \\
				stackex-chemistry & \underline{23.35$\pm$0.81} & \underline{27.87$\pm$1.34} & \underline{38.06$\pm$1.25} & \underline{37.93$\pm$1.31} & \underline{36.63$\pm$1.05} & \underline{36.02$\pm$1.42} & \underline{32.20$\pm$0.80} & \underline{40.07$\pm$0.75} & \underline{37.55$\pm$0.68} & \underline{40.46$\pm$1.29} & \underline{38.42$\pm$0.75} & \textbf{41.86$\pm$0.71} \\
				stackex-chess & \underline{21.27$\pm$1.95} & \underline{36.28$\pm$2.99} & \underline{29.07$\pm$1.90} & \underline{29.45$\pm$2.17} & \underline{29.04$\pm$2.03} & \underline{29.48$\pm$2.50} & \underline{26.56$\pm$1.96} & \underline{35.45$\pm$2.13} & \underline{35.07$\pm$2.37} & \underline{31.89$\pm$2.31} & \underline{34.59$\pm$2.75} & \textbf{37.81$\pm$1.84} \\
				stackex-coffee & \underline{58.14$\pm$7.81} & \underline{37.40$\pm$8.15} & \underline{31.72$\pm$8.67} & \underline{30.32$\pm$9.49} & \underline{31.82$\pm$8.21} & \underline{22.23$\pm$6.44} & \underline{34.19$\pm$11.45} & \underline{45.27$\pm$10.29} & \underline{40.19$\pm$9.85} & \underline{27.88$\pm$7.56} & \underline{41.48$\pm$8.61} & \textbf{58.48$\pm$9.85} \\
				stackex-cooking & \underline{19.88$\pm$0.61} & \underline{25.13$\pm$4.76} & \underline{35.77$\pm$0.78} & \underline{35.75$\pm$0.80} & \underline{35.11$\pm$0.90} & \underline{35.65$\pm$1.65} & \underline{30.16$\pm$0.85} & \underline{40.46$\pm$0.83} & \underline{38.36$\pm$1.18} & \underline{41.81$\pm$1.54} & \textbf{43.10$\pm$0.93} & 42.47$\pm$0.81 \\
				stackex-cs & \underline{22.94$\pm$0.86} & \underline{34.82$\pm$1.35} & \underline{43.14$\pm$1.30} & \underline{42.99$\pm$1.17} & \underline{43.02$\pm$1.00} & \underline{41.78$\pm$1.23} & \underline{37.51$\pm$0.89} & \underline{47.13$\pm$1.32} & \underline{45.10$\pm$1.46} & \underline{46.66$\pm$1.18} & \underline{45.71$\pm$1.13} & \textbf{47.44$\pm$1.33} \\
				stackex-philosophy & \underline{19.06$\pm$1.03} & \underline{35.55$\pm$1.49} & \underline{29.89$\pm$1.20} & \underline{29.86$\pm$1.07} & \underline{30.33$\pm$1.52} & \underline{29.50$\pm$2.29} & \underline{25.30$\pm$1.08} & \underline{34.56$\pm$1.27} & \underline{35.36$\pm$1.93} & \underline{33.77$\pm$1.27} & \underline{36.20$\pm$1.38} & \textbf{37.99$\pm$1.41} \\
				\midrule
				Avg. Rank & 6.27  & 7.80  & 7.40  & 7.80  & 8.60  & 8.53  & 7.67  & 5.73  & 7.07  & 5.53  & 4.20  & \textbf{1.40}  \\
				\bottomrule
			\end{tabular}%
		}
	\end{table*}%

	\subsection{Evaluation Metrics}
	We evaluate the performance of the algorithms using several widely-used evaluation metrics for multi-label classification. The metrics include Precision (P$@k$), nDCG (nDCG@$k$), Propensity Scored precision (PSP@$k$), and Propensity Scored nDCG (PSnDCG@$k$). These metrics are defined as follows:
	\begin{equation}
		\text{P}@k={{\mathbb{E}}_{\boldsymbol{x}}}\left[ \sum\limits_{i\in {{Top}_{k}}\left( y \right)}{\frac{1}{k}\cdot {\mathbb{P}\left( {{y}^{(i)}}=1\left| \boldsymbol{x} \right. \right)}} \right],
	\end{equation}
	\begin{equation}
		\text{nDCG}@k={{\mathbb{E}}_{\boldsymbol{x}}}\left[ \sum\limits_{i\in {{Top}_{k}}\left( y \right)}{\frac{\mathbb{P}\left( {{y}^{(i)}}=1\left| \boldsymbol{x} \right. \right)}{\cdot \log \left( i+1 \right)\cdot \sum\nolimits_{i=1}^{k}{\frac{1}{\log \left( i+1 \right)}}}} \right],
	\end{equation}
	\begin{equation}
		\text{PSP}@k={{\mathbb{E}}_{\boldsymbol{x}}}\left[ \sum\limits_{i\in {{Top}_{k}}\left( y \right)}{\frac{1}{k}\cdot \frac{\mathbb{P}\left( {{y}^{(i)}}=1\left| \boldsymbol{x} \right. \right)}{{{\xi }^{(i)}}}} \right],
	\end{equation}
	\begin{equation}
		\begin{aligned}
			&\quad \text{PSnDCG}@k\\
			& ={{\mathbb{E}}_{\boldsymbol{x}}}\left[ \sum\limits_{i\in {{Top}_{k}}\left( y \right)}{\frac{\mathbb{P}\left( {{y}^{(i)}}=1\left| \boldsymbol{x} \right. \right)}{{{\xi }^{(i)}}\cdot \log \left( i+1 \right)\cdot \sum\nolimits_{i=1}^{k}{\frac{1}{\log \left( i+1 \right)}}}} \right],
		\end{aligned}
	\end{equation}
	where 
	\begin{equation}
		\begin{aligned}
			&\quad \mathbb{P}\left( y^{(i)}_{\ast}=1\left| x \right. \right)\\
			&=\mathbb{P}\left( {{y}^{(i)}}=1\left| x \right. \right)\cdot {{\mathbb{E}}_{{{y}^{(\neg i)}}\left| x,{{y}^{(i)}}=1 \right.}}\left[ \frac{1}{1+\sum\nolimits_{j\ne i}{{{y}^{(j)}}}} \right],
		\end{aligned}
	\end{equation}
	${Top}_k(y)$ returns the $k$ largest indices of $y$ ranked in descending order, ${y}^{(\neg i)}$ denotes the vector of all but the $i$-th label, and $\xi^{(i)}$ is the propensity score for the $i$-th label which defined in \cite{PfastreXML}. In this experiment, we utilize P@$1$ (equal to nDCG@$1$), P@$3$, P@$5$, nDCG@3, nDCG@5, PSP@$1$ (equal to PSnDCG@$1$), PSP@$5$, and PSnDCG@5 to evaluate each algorithm.

	\begin{table*}[t]
		\centering
		\scriptsize
		\setlength{\tabcolsep}{0.3mm}
		\caption{Experimental results (mean $\pm$ std) on the MLC benchmark datasets measured by PSnDCG@$5$. Each underlined result indicates that SLDL is statistically superior to the comparing method. The best performance on each dataset is denoted in boldface.}
		\label{table:result-PSnDCG@5}
		\resizebox{2\columnwidth}{!}{
			\begin{tabular}{lcccccccccccc}
				\toprule
				Dataset & SLEEC \cite{SLEEC} & DXML \cite{DXML} & FL \cite{focal} & CB \cite{CB} & DB \cite{DB} & PACA \cite{PACA} & FLEM \cite{FLEM} & DELA \cite{DELA} & CLIF \cite{CLIF} & KD-TEA \cite{KD-TEA} & AdaC2 \cite{AdaBoost.C2} & SLDL (Ours) \\
				\midrule
				cal500 & \underline{38.72$\pm$1.50} & \underline{38.37$\pm$1.58} & \underline{38.22$\pm$1.73} & \underline{38.34$\pm$1.87} & \underline{37.37$\pm$1.44} & \underline{38.32$\pm$1.60} & \underline{38.36$\pm$1.50} & \underline{37.85$\pm$1.56} & \underline{38.05$\pm$2.30} & \underline{38.99$\pm$1.23} & \underline{39.05$\pm$1.63} & \textbf{39.34$\pm$1.44} \\
				corel16k-s1 & \underline{29.43$\pm$0.81} & \underline{25.09$\pm$0.82} & \underline{27.27$\pm$0.40} & \underline{27.28$\pm$0.38} & \underline{25.29$\pm$0.51} & \underline{23.00$\pm$1.11} & \underline{29.10$\pm$0.51} & \underline{23.48$\pm$0.53} & \underline{22.81$\pm$0.86} & \underline{25.52$\pm$0.70} & \underline{29.32$\pm$0.62} & \textbf{30.13$\pm$0.67} \\
				corel16k-s2 & \underline{28.75$\pm$0.95} & \underline{24.21$\pm$0.61} & \underline{26.81$\pm$0.82} & \underline{26.77$\pm$0.82} & \underline{24.88$\pm$0.74} & \underline{22.94$\pm$0.81} & \underline{28.55$\pm$0.93} & \underline{23.24$\pm$0.69} & \underline{22.44$\pm$0.93} & \underline{25.11$\pm$0.96} & \underline{28.82$\pm$0.96} & \textbf{29.40$\pm$1.02} \\
				corel16k-s3 & \underline{29.26$\pm$1.01} & \underline{24.93$\pm$1.02} & \underline{27.19$\pm$1.00} & \underline{27.21$\pm$0.91} & \underline{25.28$\pm$1.05} & \underline{23.01$\pm$1.19} & \underline{28.84$\pm$1.02} & \underline{23.39$\pm$0.78} & \underline{22.60$\pm$0.87} & \underline{25.37$\pm$0.93} & \underline{29.27$\pm$0.80} & \textbf{29.75$\pm$1.03} \\
				CUB   & 59.52$\pm$0.94 & \underline{53.35$\pm$0.81} & \underline{57.13$\pm$0.83} & \underline{57.28$\pm$0.99} & \underline{58.66$\pm$0.92} & 59.17$\pm$0.95 & \underline{55.07$\pm$0.66} & \textbf{59.89$\pm$0.92} & \underline{58.38$\pm$1.01} & \underline{58.75$\pm$1.16} & \underline{55.17$\pm$2.00} & 58.96$\pm$0.79 \\
				delicious & \underline{34.50$\pm$0.59} & \underline{28.84$\pm$0.53} & \underline{34.45$\pm$0.44} & \underline{34.42$\pm$0.50} & \underline{34.49$\pm$0.51} & \underline{34.22$\pm$0.57} & \underline{34.75$\pm$0.36} & \underline{34.48$\pm$0.46} & \underline{32.55$\pm$0.51} & \underline{34.32$\pm$0.38} & \textbf{36.62$\pm$0.64} & 35.32$\pm$0.64 \\
				eurlex-dc & \underline{53.63$\pm$1.30} & \underline{53.47$\pm$0.87} & \underline{20.92$\pm$0.90} & \underline{20.86$\pm$0.93} & \underline{22.39$\pm$0.87} & \underline{39.54$\pm$2.34} & \underline{18.26$\pm$0.80} & \underline{42.64$\pm$0.99} & \underline{56.04$\pm$1.05} & \underline{36.85$\pm$0.93} & \underline{27.94$\pm$1.19} & \textbf{59.53$\pm$0.96} \\
				eurlex-sm & \underline{58.84$\pm$1.12} & \underline{61.65$\pm$0.76} & \underline{28.62$\pm$1.01} & \underline{28.62$\pm$1.02} & \underline{29.99$\pm$1.08} & \underline{56.52$\pm$1.60} & \underline{25.53$\pm$0.88} & \underline{52.22$\pm$1.13} & \underline{67.09$\pm$0.50} & \underline{50.76$\pm$1.21} & \underline{32.71$\pm$1.24} & \textbf{67.40$\pm$0.98} \\
				espgame & \underline{24.60$\pm$0.41} & \underline{19.56$\pm$0.57} & \underline{24.47$\pm$0.76} & \underline{24.51$\pm$0.74} & \underline{23.54$\pm$0.77} & \underline{25.00$\pm$0.66} & \underline{25.03$\pm$0.65} & \underline{24.23$\pm$0.65} & \underline{23.11$\pm$0.49} & \underline{24.62$\pm$0.53} & \textbf{30.25$\pm$0.69} & 28.72$\pm$0.43 \\
				stackex-chemistry & \underline{19.67$\pm$0.58} & \underline{23.65$\pm$0.88} & \underline{32.64$\pm$1.04} & \underline{32.55$\pm$1.05} & \underline{31.40$\pm$1.06} & \underline{31.45$\pm$1.15} & \underline{27.85$\pm$0.88} & \underline{35.30$\pm$0.79} & \underline{32.76$\pm$0.89} & \underline{35.11$\pm$1.08} & \underline{35.06$\pm$1.14} & \textbf{37.07$\pm$0.84} \\
				stackex-chess & \underline{17.75$\pm$2.02} & \underline{32.81$\pm$2.92} & \underline{25.72$\pm$1.77} & \underline{26.26$\pm$2.37} & \underline{25.43$\pm$2.13} & \underline{25.33$\pm$2.00} & \underline{23.23$\pm$1.78} & \underline{31.21$\pm$2.11} & \underline{30.45$\pm$2.07} & \underline{28.05$\pm$2.65} & \underline{32.07$\pm$2.17} & \textbf{34.42$\pm$2.81} \\
				stackex-coffee & \textbf{46.03$\pm$8.57} & \underline{29.33$\pm$7.67} & \underline{25.01$\pm$7.74} & \underline{25.03$\pm$6.89} & \underline{25.15$\pm$7.14} & \underline{16.53$\pm$5.33} & \underline{26.77$\pm$9.16} & \underline{34.57$\pm$9.10} & \underline{31.80$\pm$8.83} & \underline{21.54$\pm$6.36} & \underline{35.83$\pm$7.41} & 45.78$\pm$9.16 \\
				stackex-cooking & \underline{17.88$\pm$0.38} & \underline{22.40$\pm$5.03} & \underline{32.23$\pm$0.69} & \underline{32.23$\pm$0.75} & \underline{31.55$\pm$0.90} & \underline{33.02$\pm$1.56} & \underline{26.65$\pm$0.68} & \underline{37.62$\pm$0.69} & \underline{35.51$\pm$1.00} & \underline{38.10$\pm$1.17} & \textbf{41.09$\pm$0.90} & 40.40$\pm$0.80 \\
				stackex-cs & \underline{19.49$\pm$0.83} & \underline{30.07$\pm$0.99} & \underline{37.28$\pm$1.07} & \underline{37.18$\pm$0.97} & \underline{37.78$\pm$0.92} & \underline{36.67$\pm$0.80} & \underline{32.15$\pm$0.84} & \underline{41.77$\pm$1.16} & \underline{39.55$\pm$1.21} & \underline{40.98$\pm$1.12} & \underline{41.40$\pm$1.17} & \textbf{41.85$\pm$1.75} \\
				stackex-philosophy & \underline{16.95$\pm$1.04} & \underline{33.80$\pm$1.21} & \underline{27.79$\pm$1.16} & \underline{27.75$\pm$1.08} & \underline{27.44$\pm$1.20} & \underline{27.61$\pm$2.15} & \underline{23.92$\pm$1.03} & \underline{32.55$\pm$1.40} & \underline{32.67$\pm$1.49} & \underline{31.23$\pm$1.39} & \underline{35.46$\pm$1.29} & \textbf{36.42$\pm$1.04} \\
				\midrule
				Avg. Rank & 6.13  & 7.93  & 7.67  & 7.53  & 8.40  & 8.20  & 7.87  & 5.93  & 7.20  & 6.00  & 3.53  & \textbf{1.47}  \\
				\bottomrule
			\end{tabular}%
		}
	\end{table*}%

	\begin{table}[t]
		\centering
		\caption{Friedman statistics $F_F$ in terms of each evaluation criterion and the critical value at 0.05 significance level.}
		\setlength{\tabcolsep}{6.5mm}
		\resizebox{\columnwidth}{!}{
			\begin{tabular}{lcc}
				\toprule
				Evaluation Criterion & $F_F$ & Critical value \\
				\midrule
				P@1   & 6.0870  & \multirow{8}[2]{*}{1.8513 } \\
				P@3   & 7.3261  &  \\
				P@5   & 7.0806  &  \\
				nDCG@3 & 7.7378  &  \\
				nDCG@5 & 7.1698  &  \\
				PSP@1 & 5.6957  &  \\
				PSP@5 & 6.9181  &  \\
				PSnDCG@5 & 7.1817  &  \\
				\bottomrule
			\end{tabular}%
		}		
		\label{tab:Friedman}%
	\end{table}%

	\begin{figure*} [t]
		\centering 
		\resizebox{2\columnwidth}{!}{
			\subfloat[P@1]{
				\includegraphics[width=0.486\columnwidth]{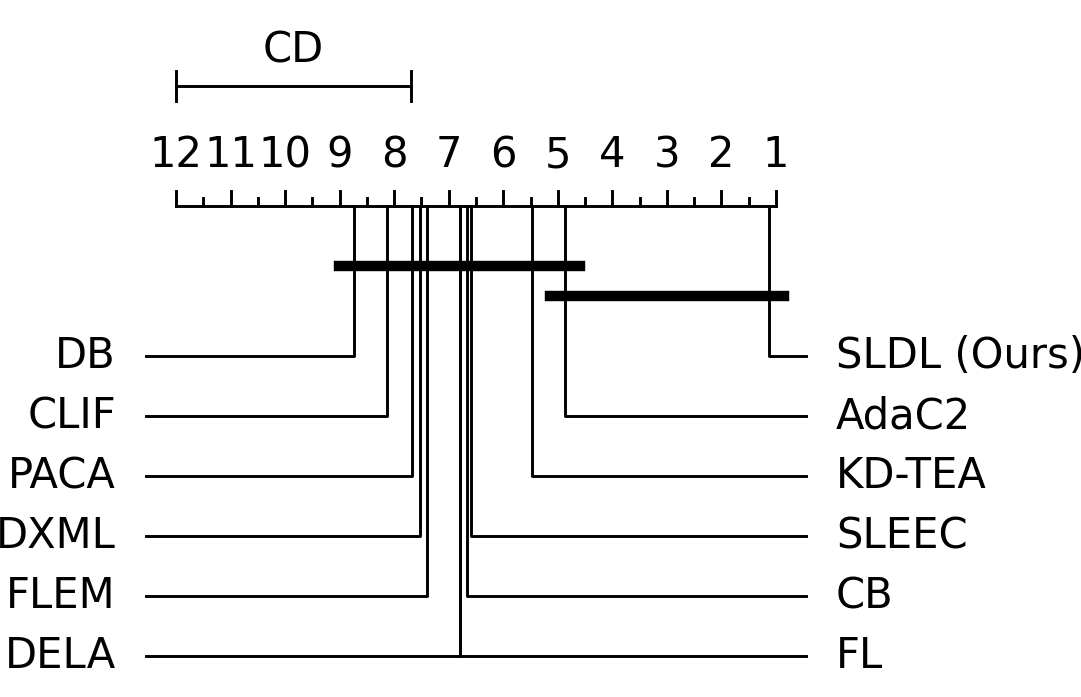}
				\label{fig:CD-diagram-P@1}
			}
			\subfloat[P@3]{
				\includegraphics[width=0.486\columnwidth]{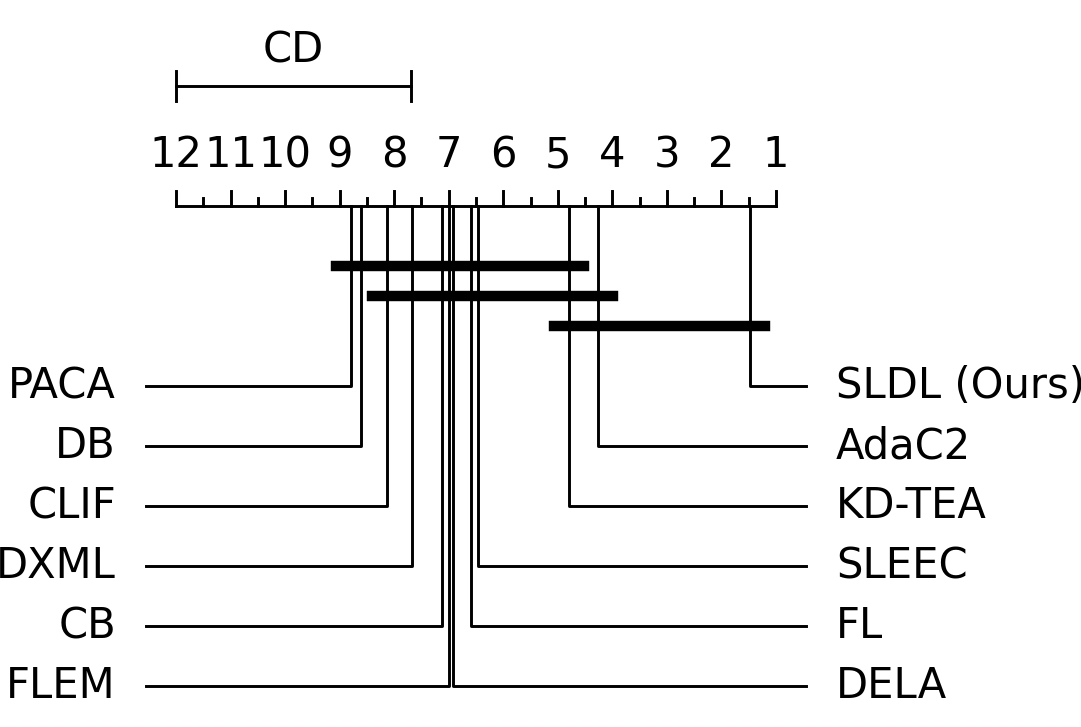}
				\label{fig:CD-diagram-P@3}
			}
			\subfloat[P@5]{
				\includegraphics[width=0.486\columnwidth]{assets/CD-diagram/P-5}
				\label{fig:CD-diagram-P@5}
			}
			\subfloat[nDCG@3]{
				\includegraphics[width=0.486\columnwidth]{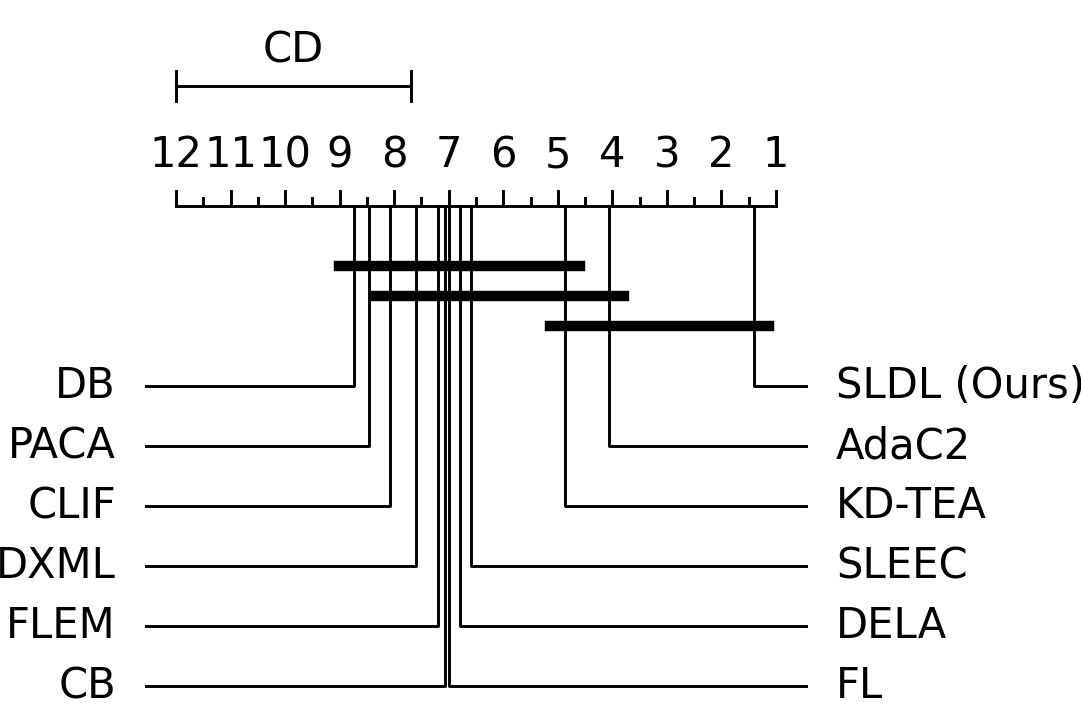}
				\label{fig:CD-diagram-nDCG@3}
			}
		}
		\resizebox{2\columnwidth}{!}{
			\subfloat[nDCG@5]{
				\includegraphics[width=0.486\columnwidth]{assets/CD-diagram/nDCG-5}
				\label{fig:CD-diagram-nDCG@5}
			}
			\subfloat[PSP@1]{
				\includegraphics[width=0.486\columnwidth]{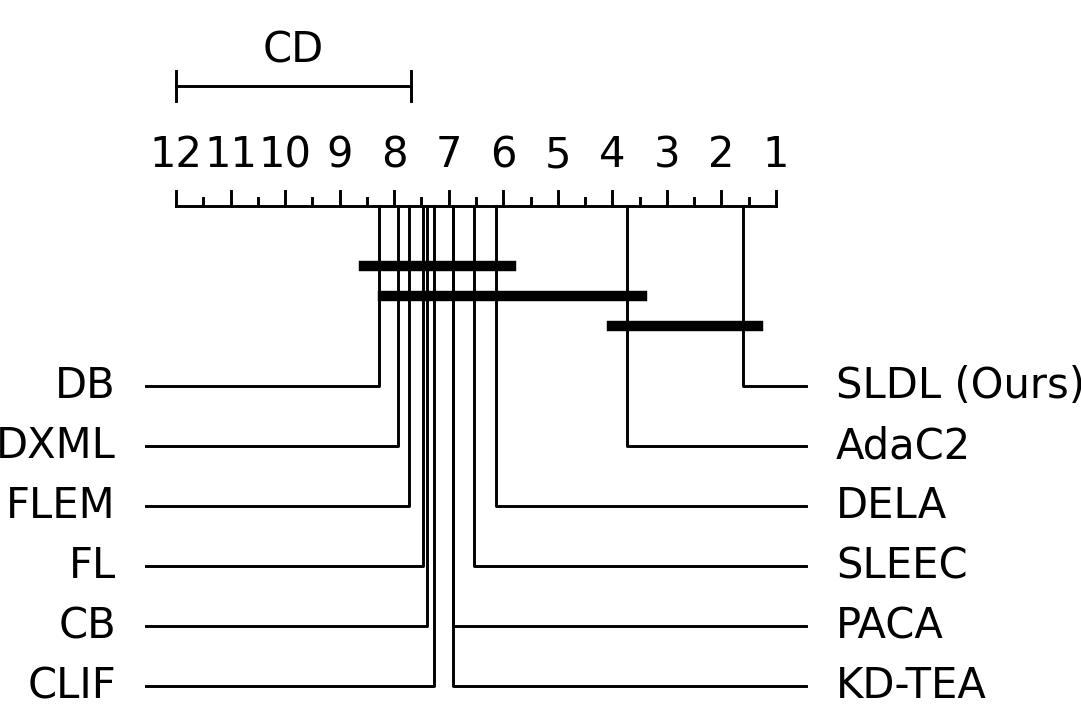}
				\label{fig:CD-diagram-PSP@1}
			}
			\subfloat[PSP@5]{
				\includegraphics[width=0.486\columnwidth]{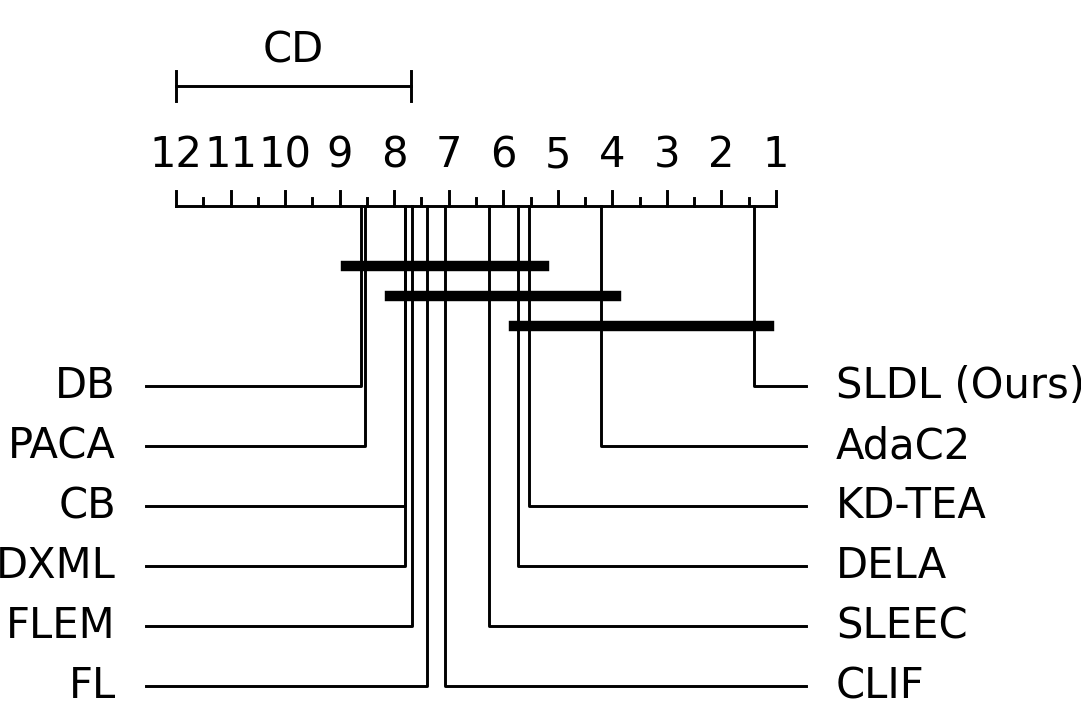}
				\label{fig:CD-diagram-PSP@5}
			}
			\subfloat[PSnDCG@5]{
				\includegraphics[width=0.486\columnwidth]{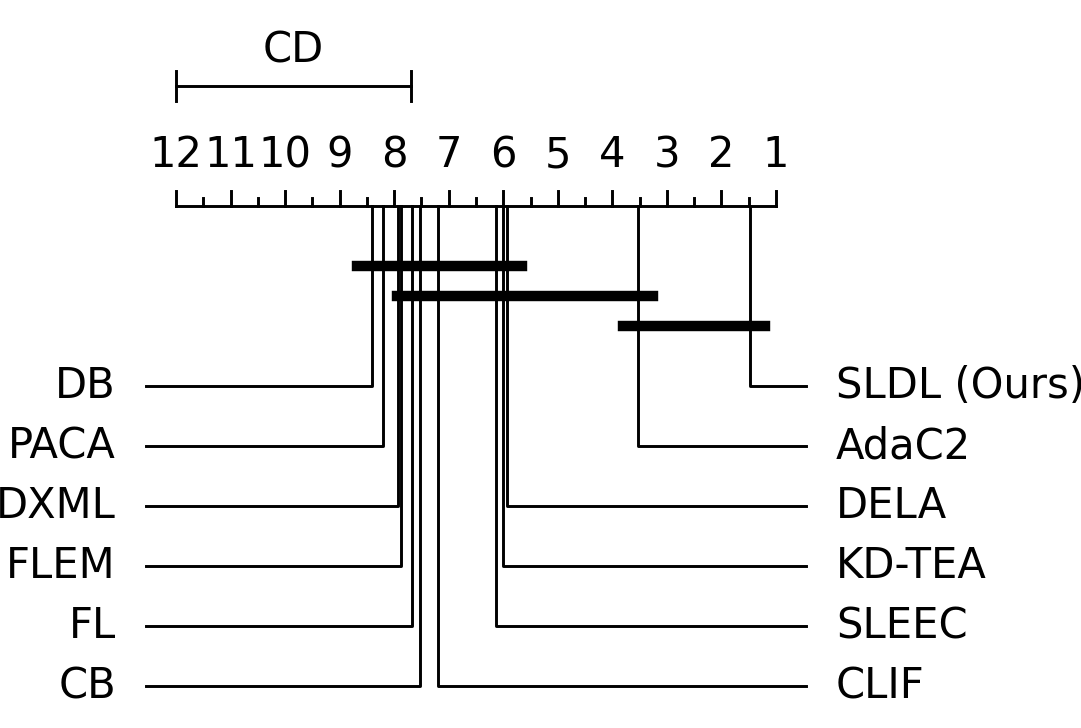}
				\label{fig:CD-diagram-PSnDCG@5}
			}
		}
		\caption{Comparison of SLDL (control algorithm) against comparing algorithms with the \textit{Nemenyi test}. Algorithms not connected with SLDL in the CD diagram are considered to have a significantly different performance from the control algorithm.}
		\label{fig:CD-diagram}
	\end{figure*}

	\begin{figure*} [t]
		\resizebox{2\columnwidth}{!}{
			\subfloat[cal500]{
				\includegraphics[width=0.395\columnwidth]{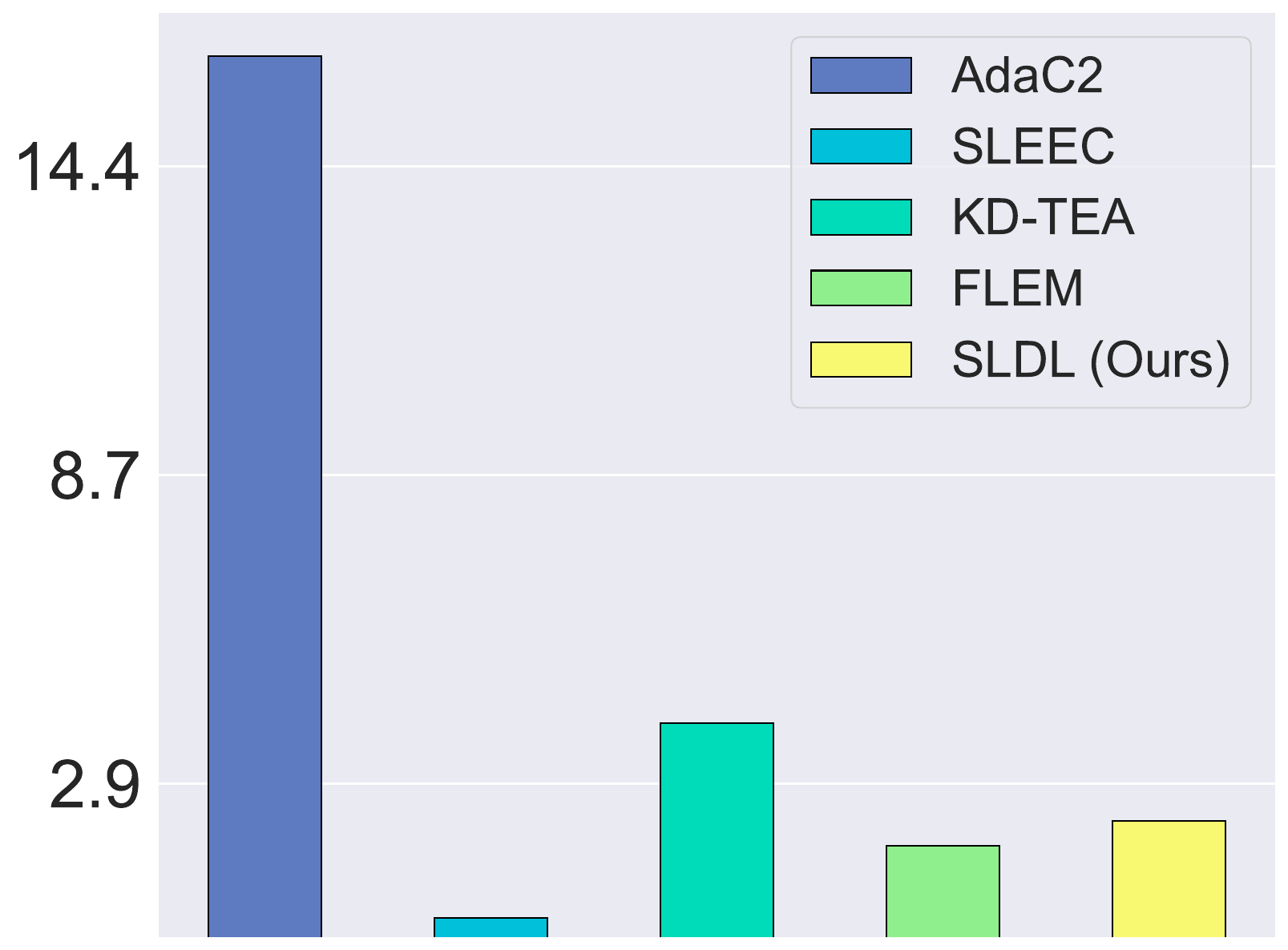}
			}
			\subfloat[corel16k-s1]{
				\includegraphics[width=0.395\columnwidth]{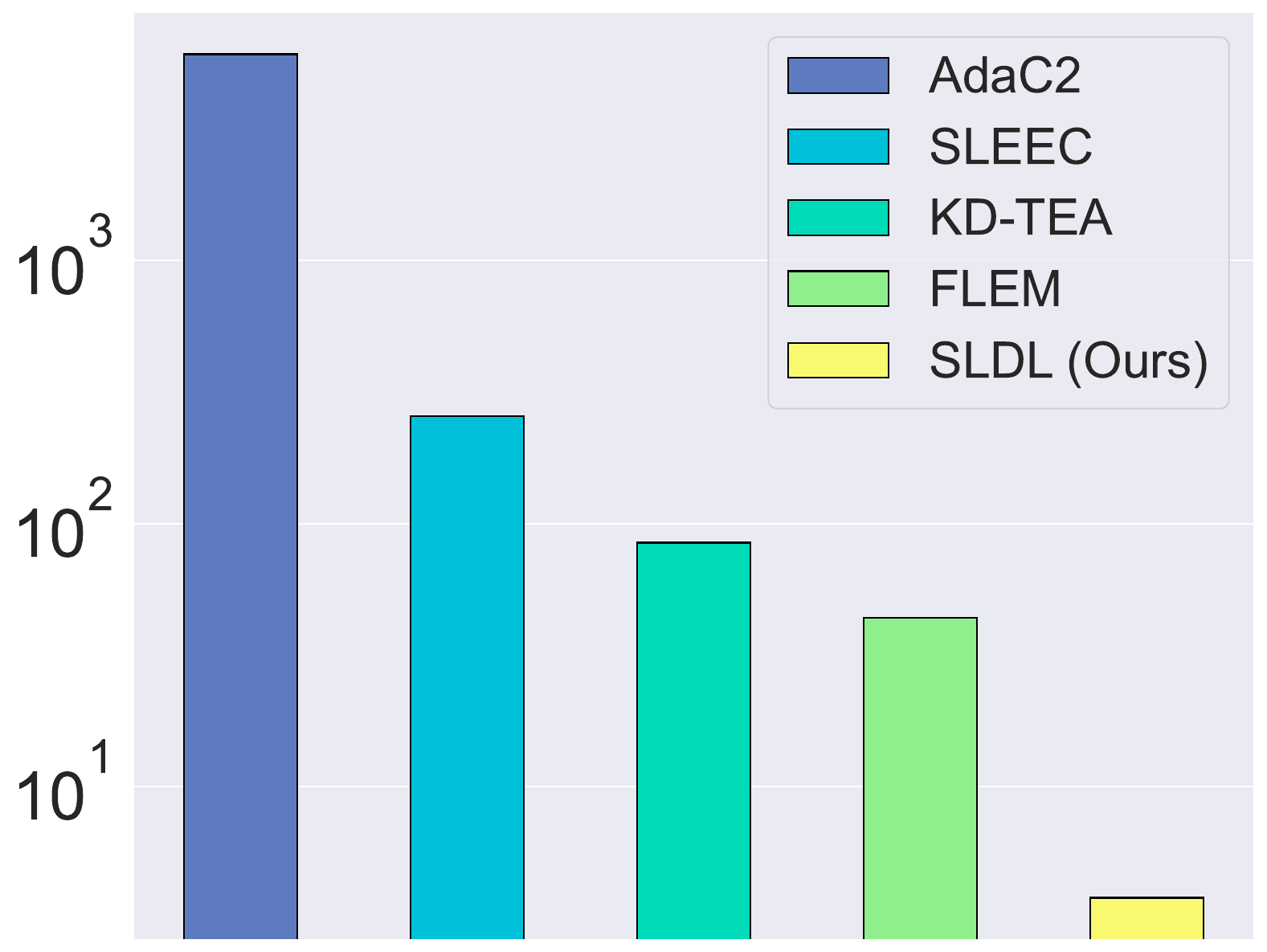}
			}
			\subfloat[corel16k-s2]{
				\includegraphics[width=0.395\columnwidth]{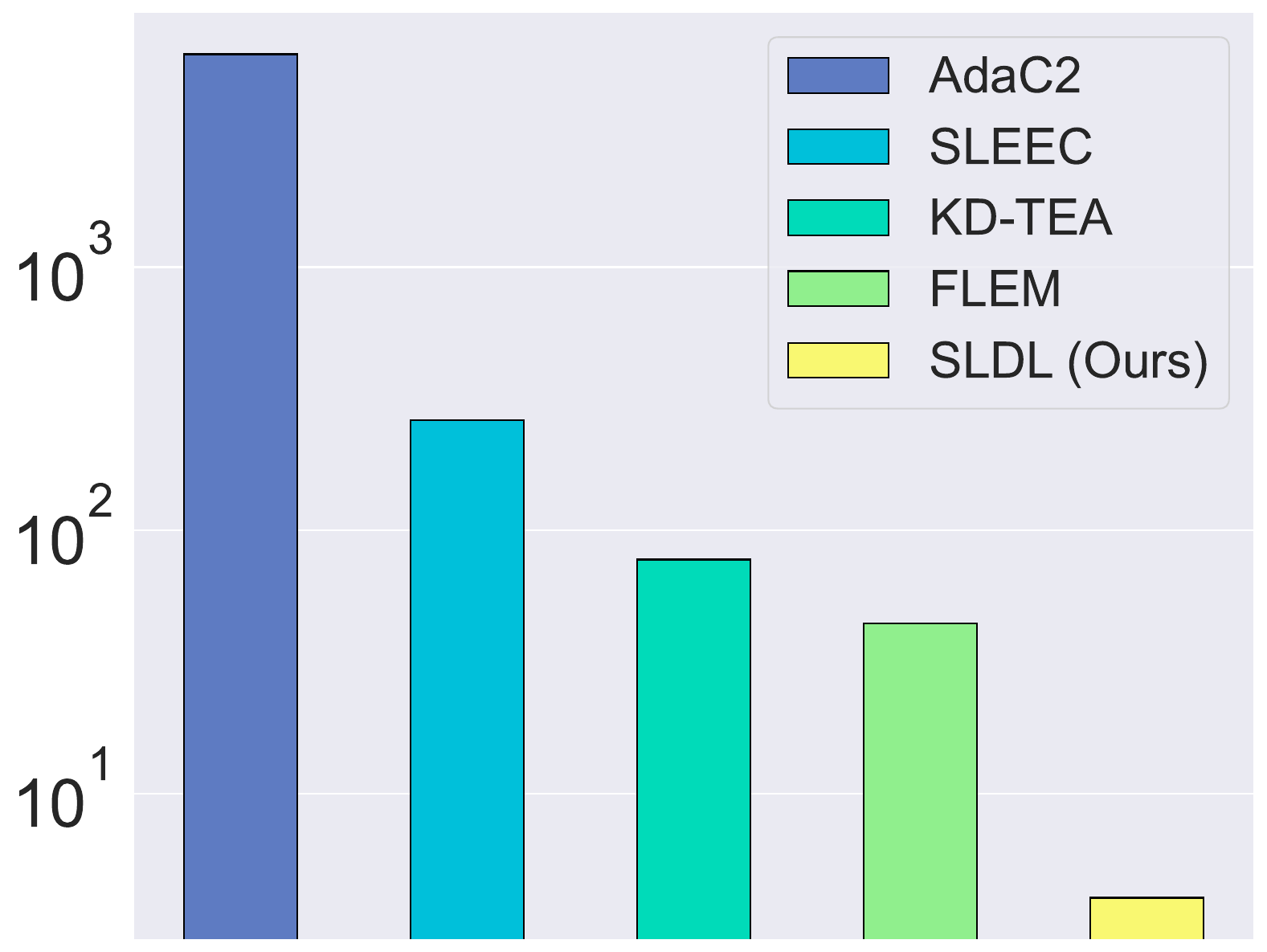}
			}
			\subfloat[corel16k-s3]{
				\includegraphics[width=0.395\columnwidth]{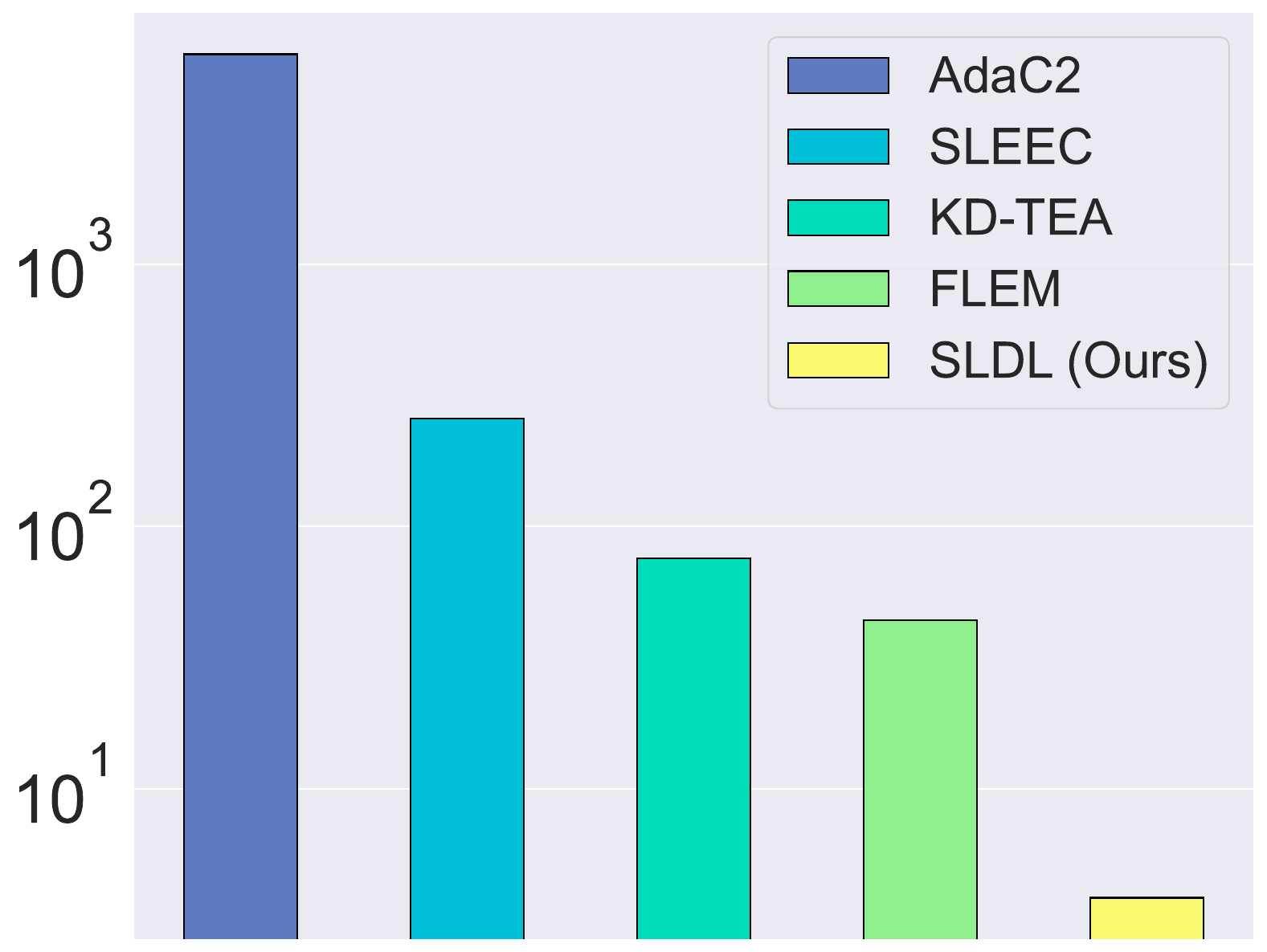}
			}
			\subfloat[CUB]{
				\includegraphics[width=0.395\columnwidth]{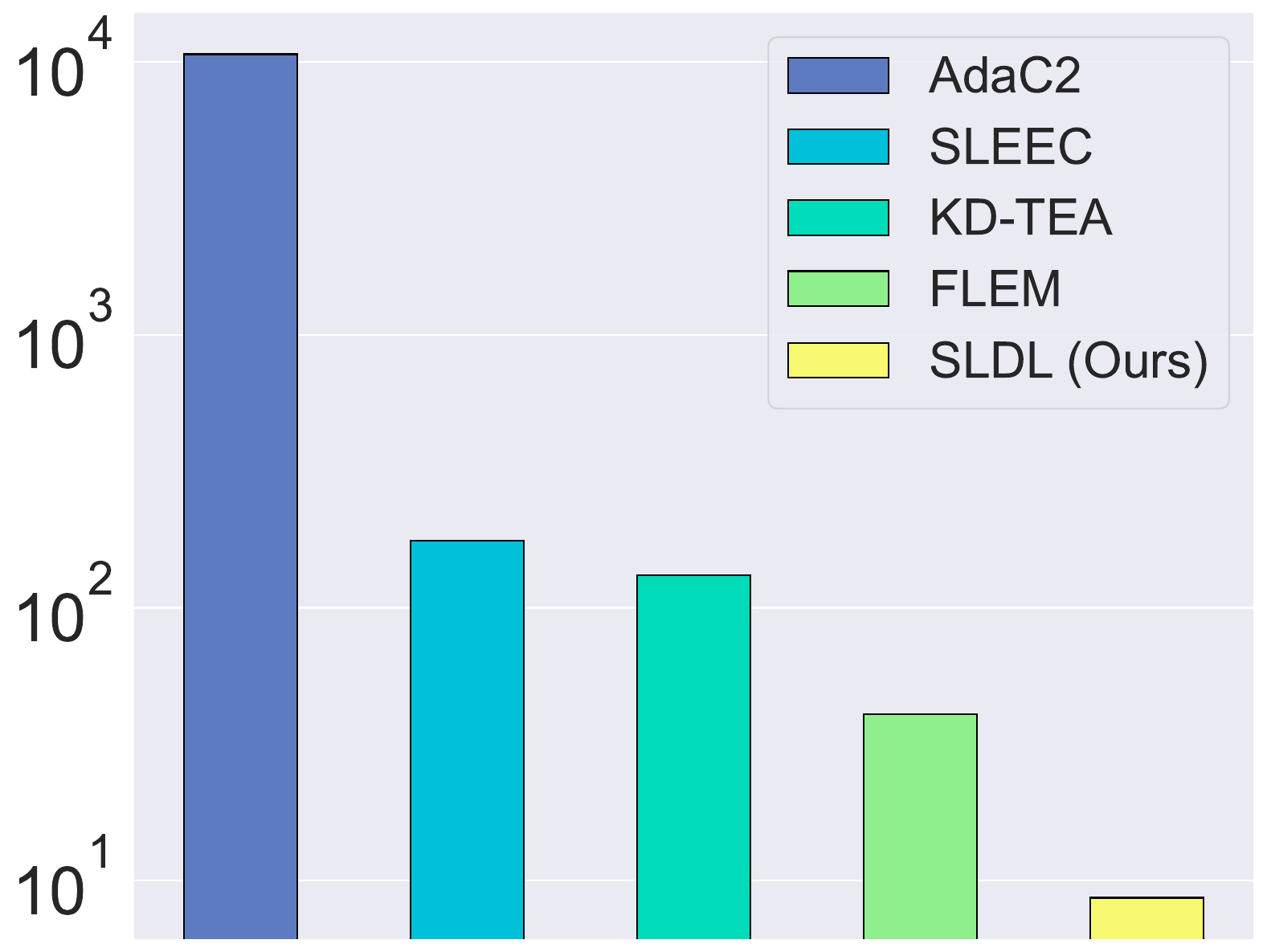}
			}
		}
		
		\resizebox{2\columnwidth}{!}{
			\subfloat[delicious]{
				\includegraphics[width=0.395\columnwidth]{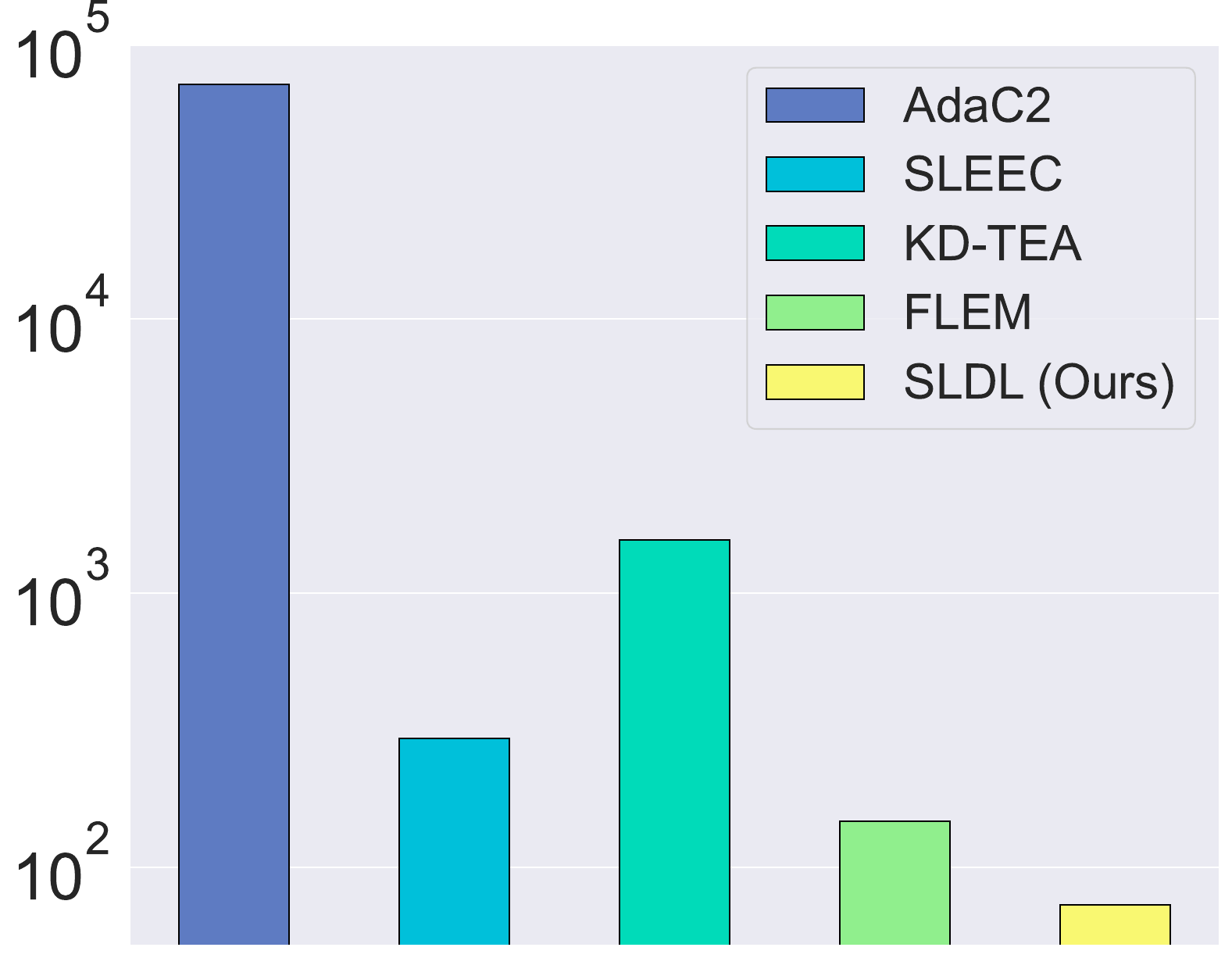}
			}
			\subfloat[eurlex-dc]{
				\includegraphics[width=0.395\columnwidth]{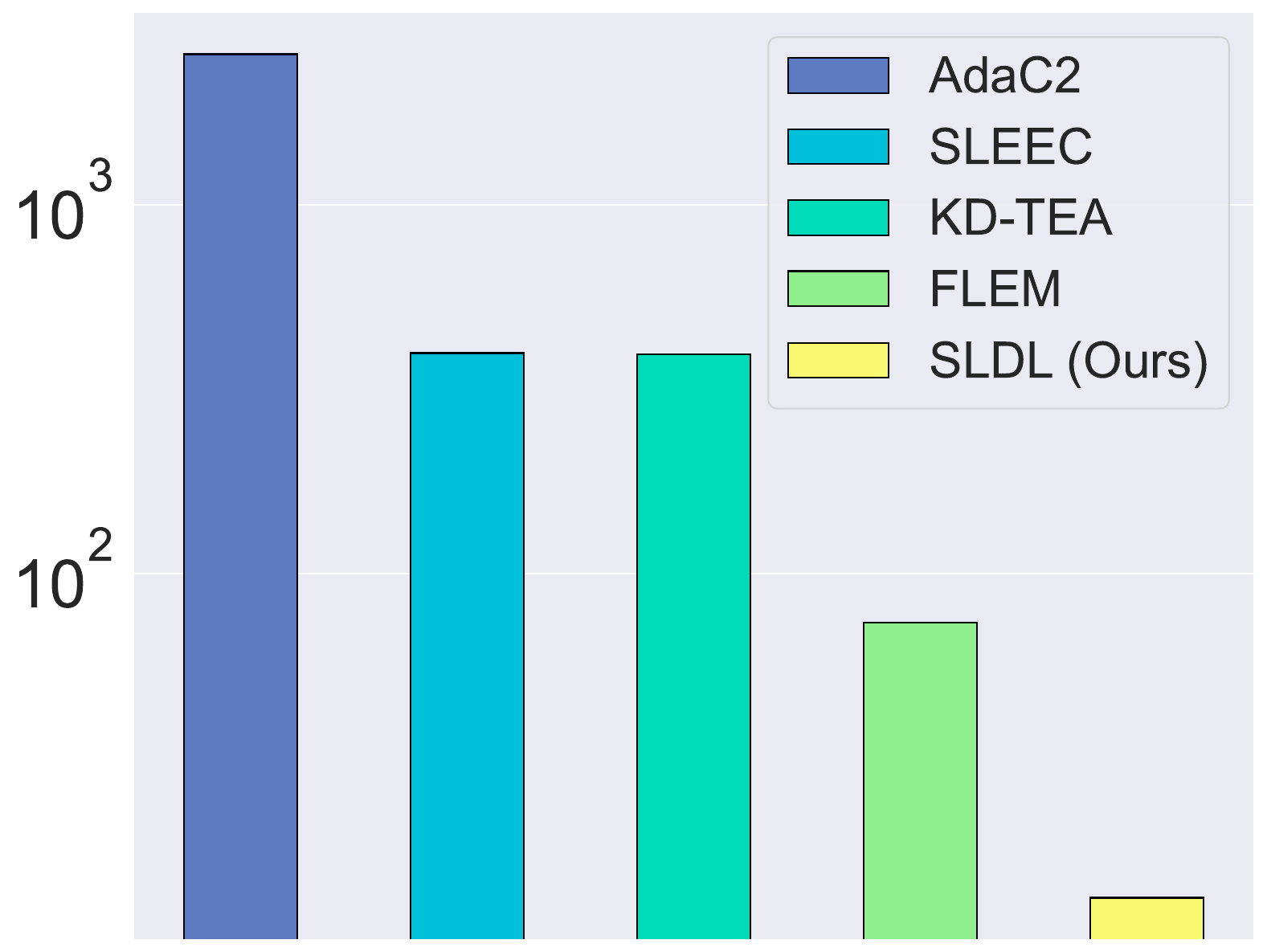}
			}
			\subfloat[eurlex-sm]{
				\includegraphics[width=0.395\columnwidth]{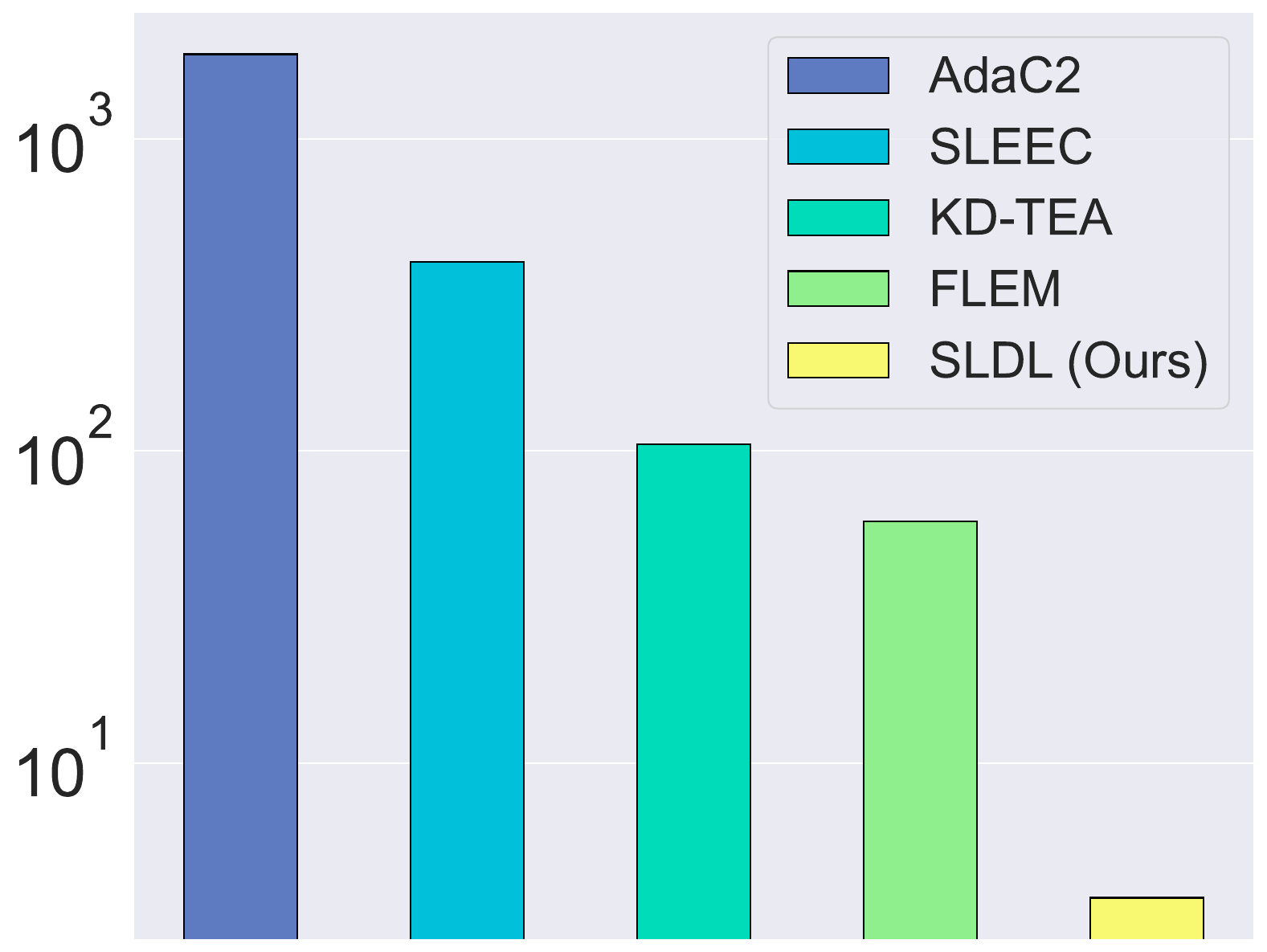}
			}
			\subfloat[espgame]{
				\includegraphics[width=0.395\columnwidth]{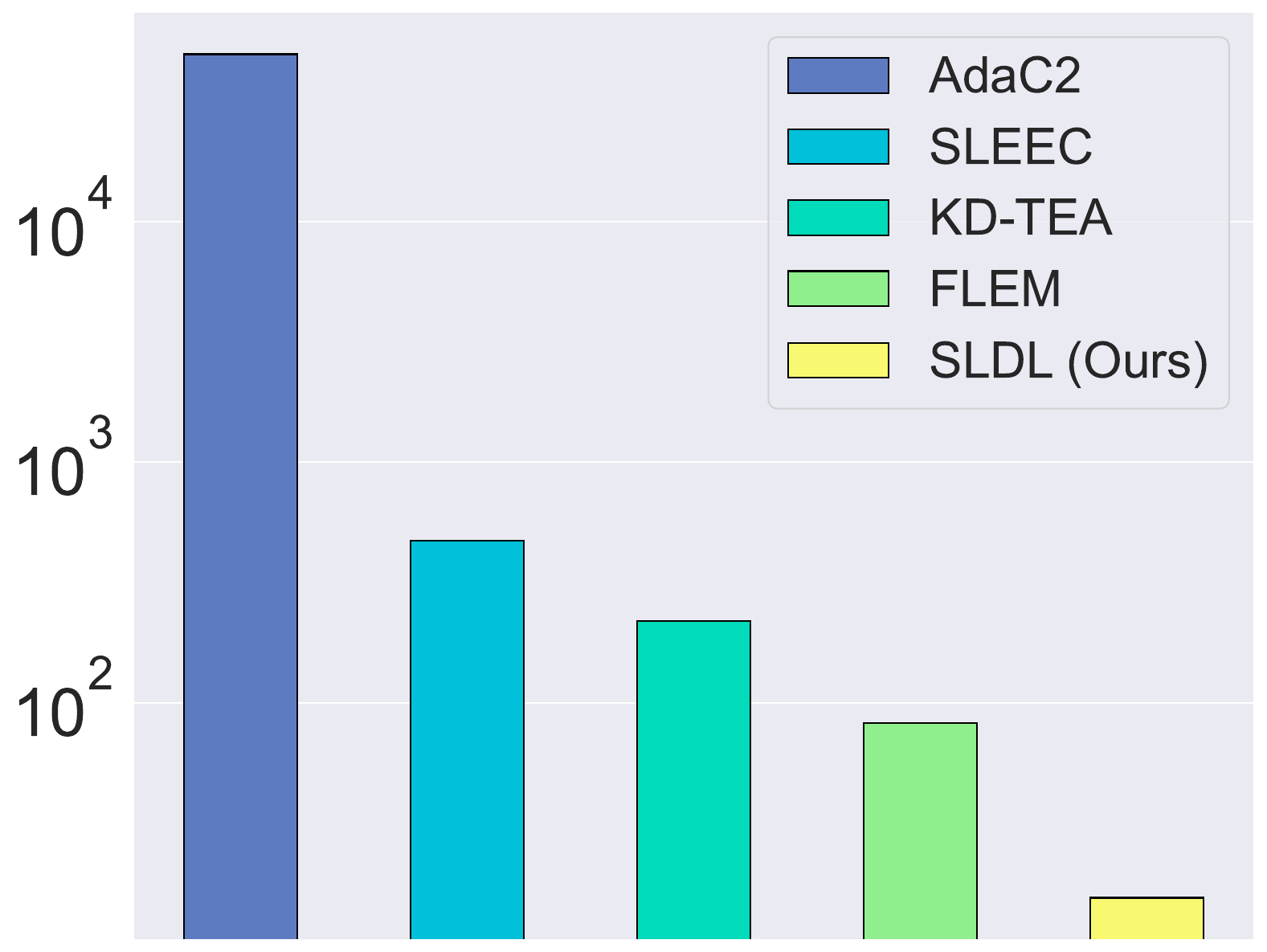}
			}
			\subfloat[stackex-chemistry]{
				\includegraphics[width=0.395\columnwidth]{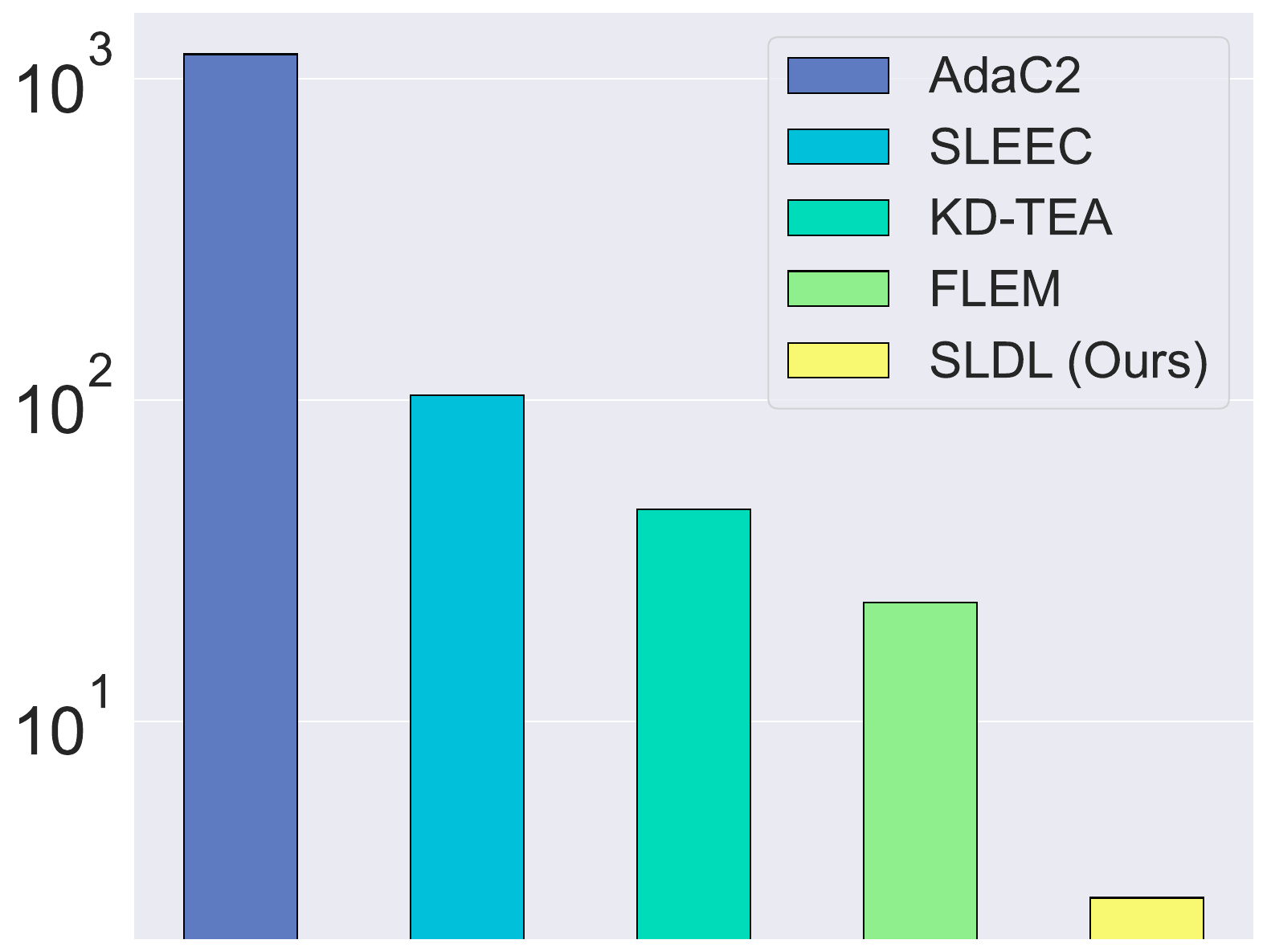}
			}
		}
		
		\resizebox{2\columnwidth}{!}{
			\subfloat[stackex-chess]{
				\includegraphics[width=0.395\columnwidth]{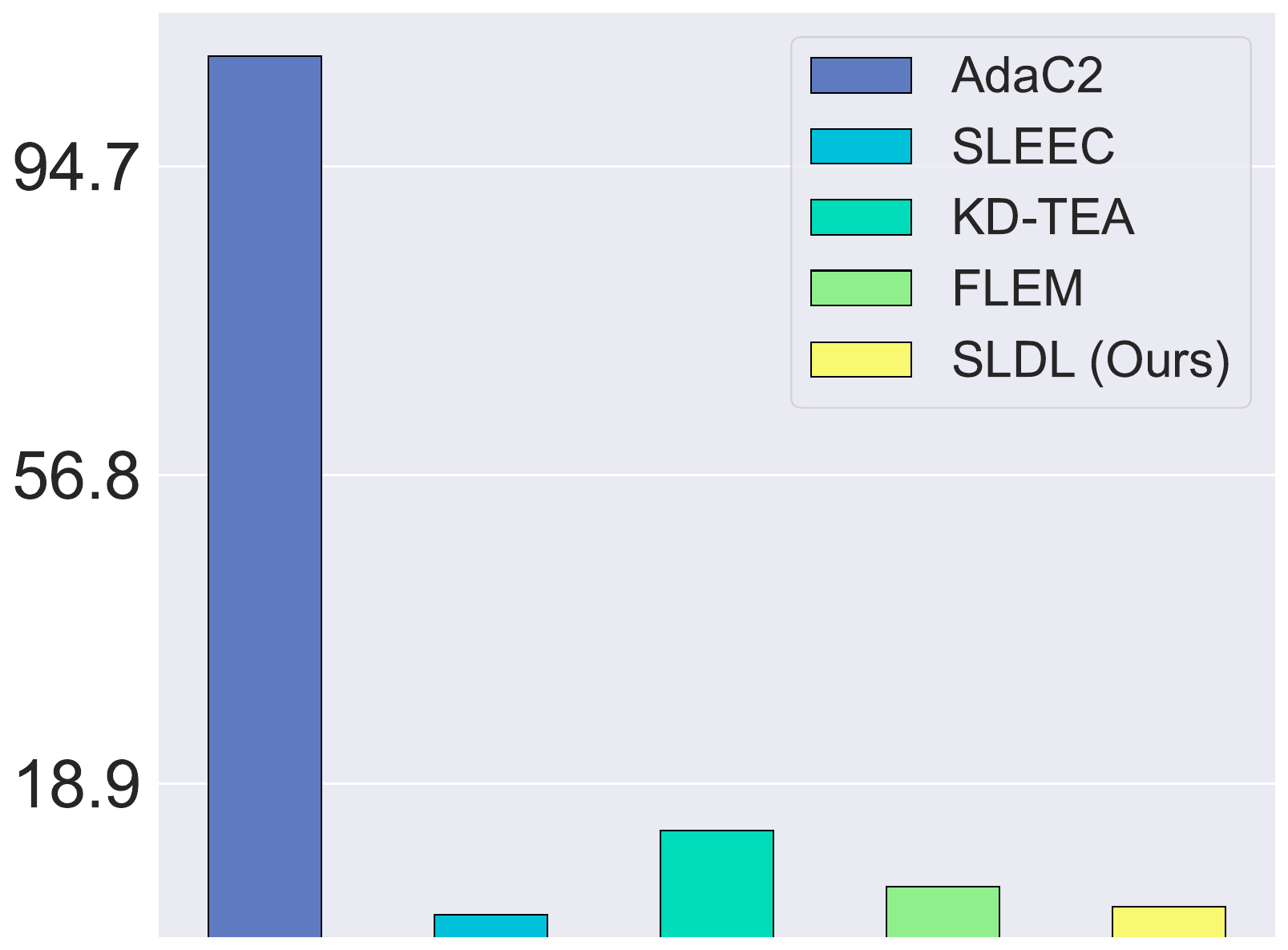}
			}
			\subfloat[stackex-coffee]{
				\includegraphics[width=0.395\columnwidth]{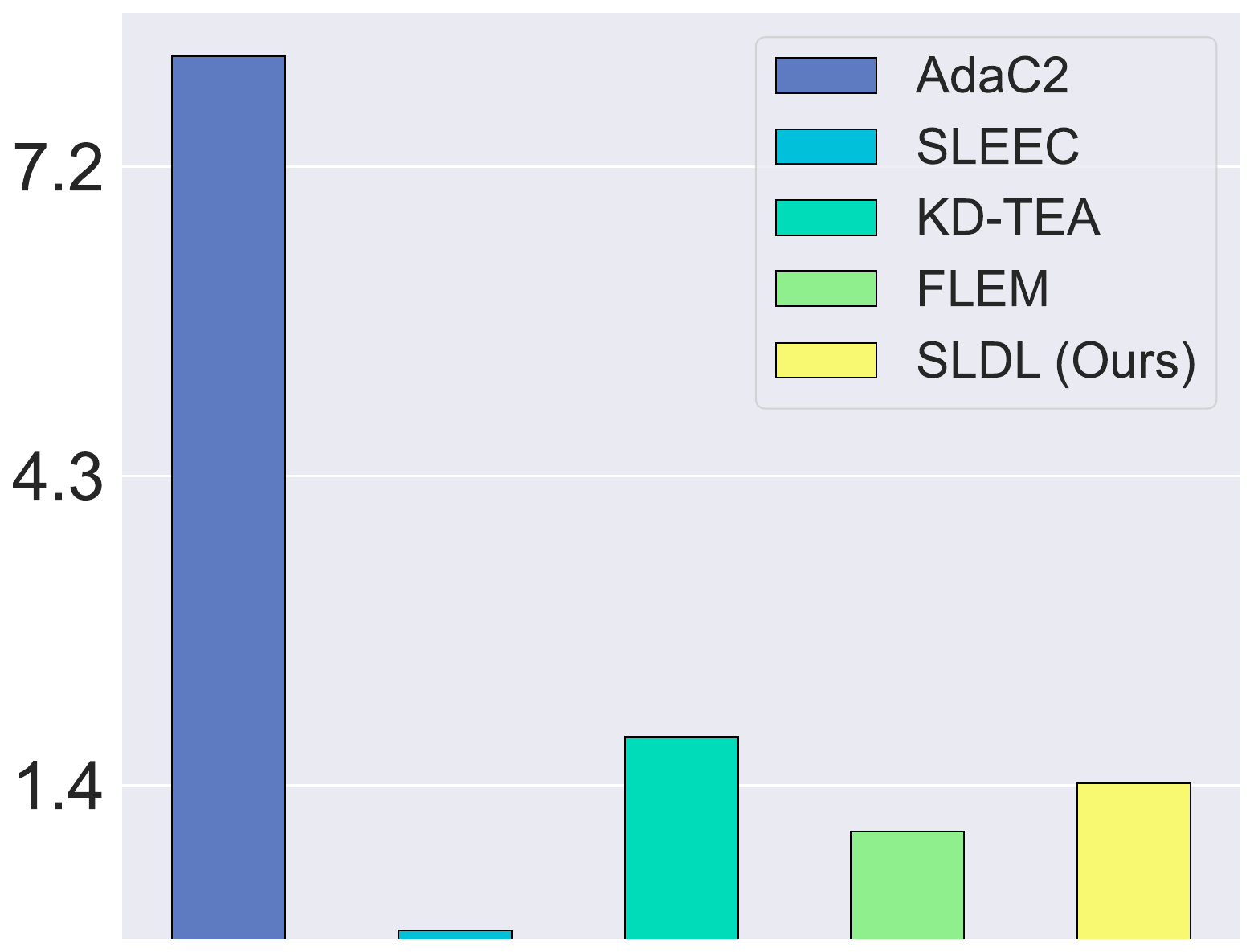}
			}
			\subfloat[stackex-cooking]{
				\includegraphics[width=0.395\columnwidth]{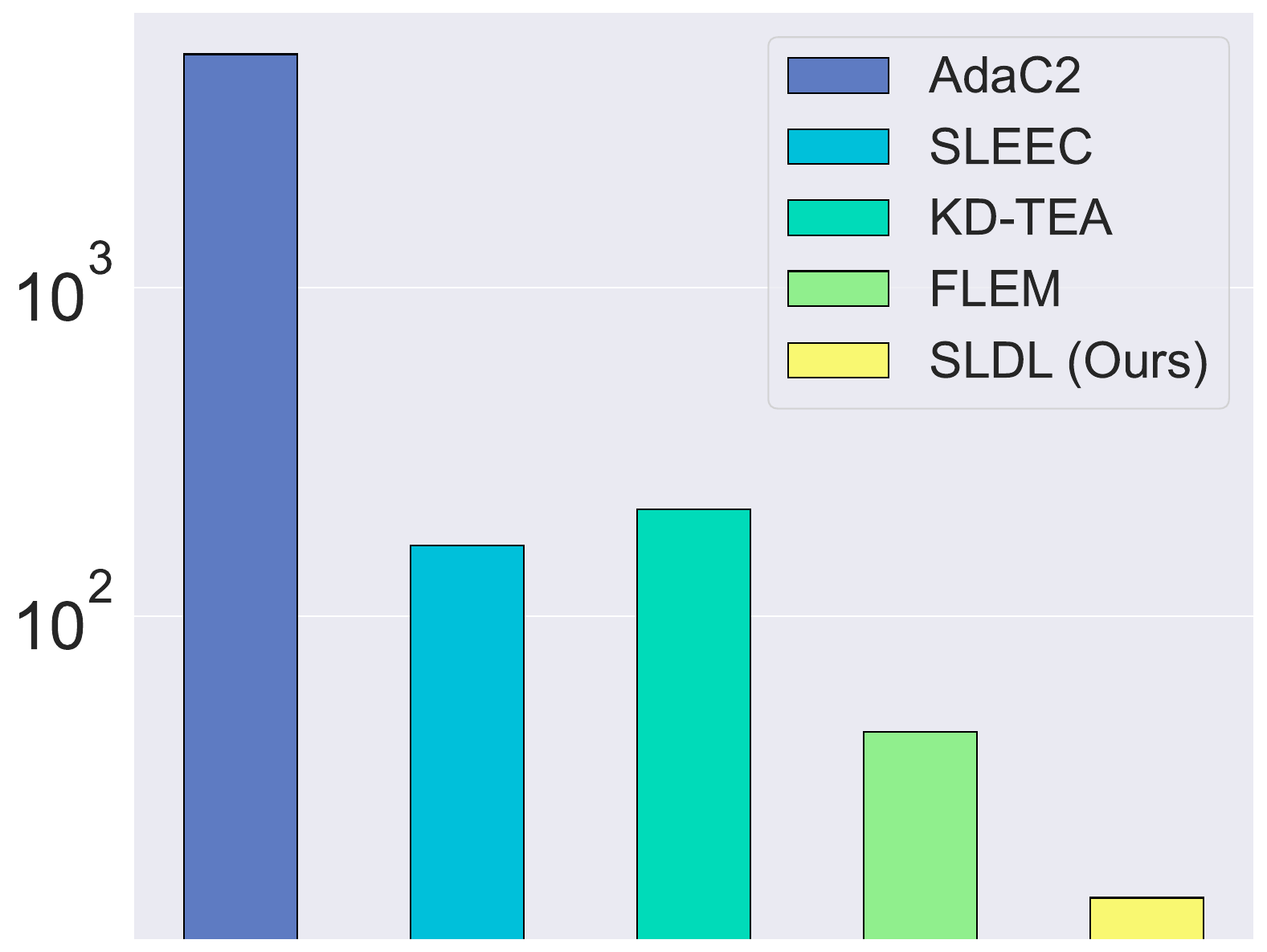}
			}
			\subfloat[stackex-cs]{
				\includegraphics[width=0.395\columnwidth]{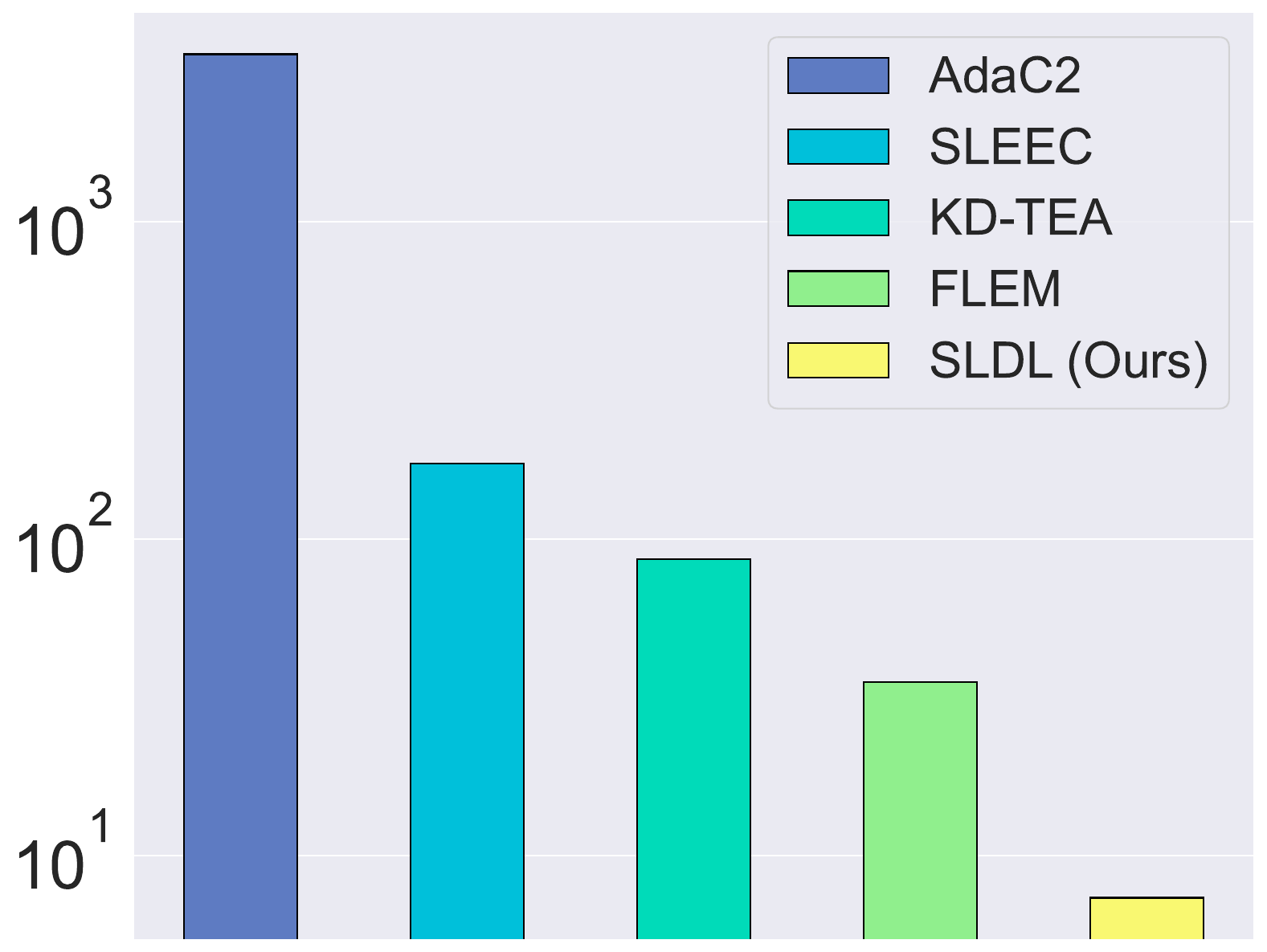}
			}
			\subfloat[stackex-philosophy]{
				\includegraphics[width=0.395\columnwidth]{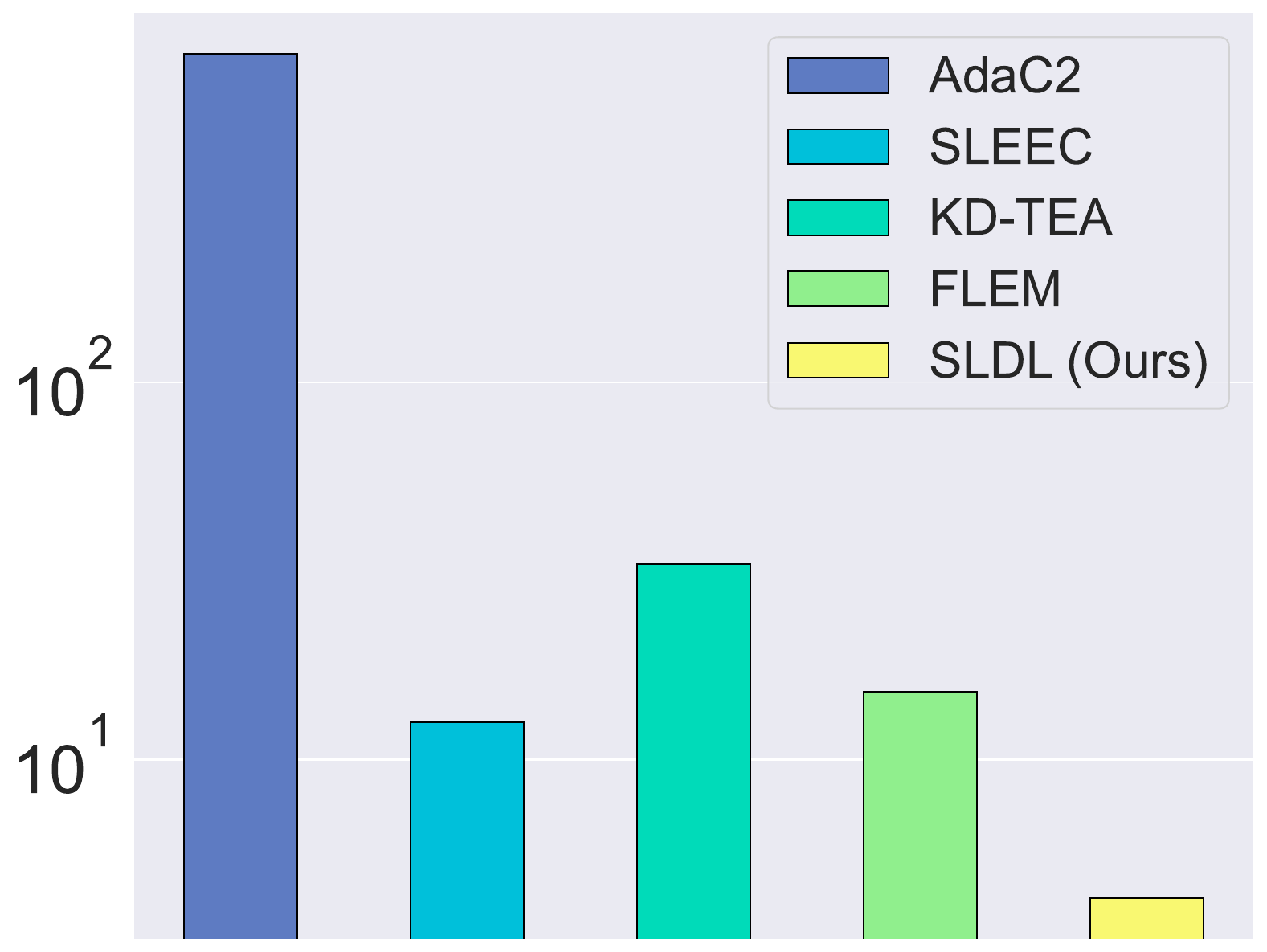}
			}
		}
		\caption{Training time of AdaC2, SLEEC, FLEM, KD-TEA, and SLDL on different datasets. In each subfigure, the x-axis indicates the MLC method, the y-axis indicates the training time ($s$).}
		\label{fig:running-times}
	\end{figure*}

	\begin{figure*} [t]
		\centering
		\resizebox{2\columnwidth}{!}{
			\subfloat[P@$1$]{
				\includegraphics[width=0.965\columnwidth]{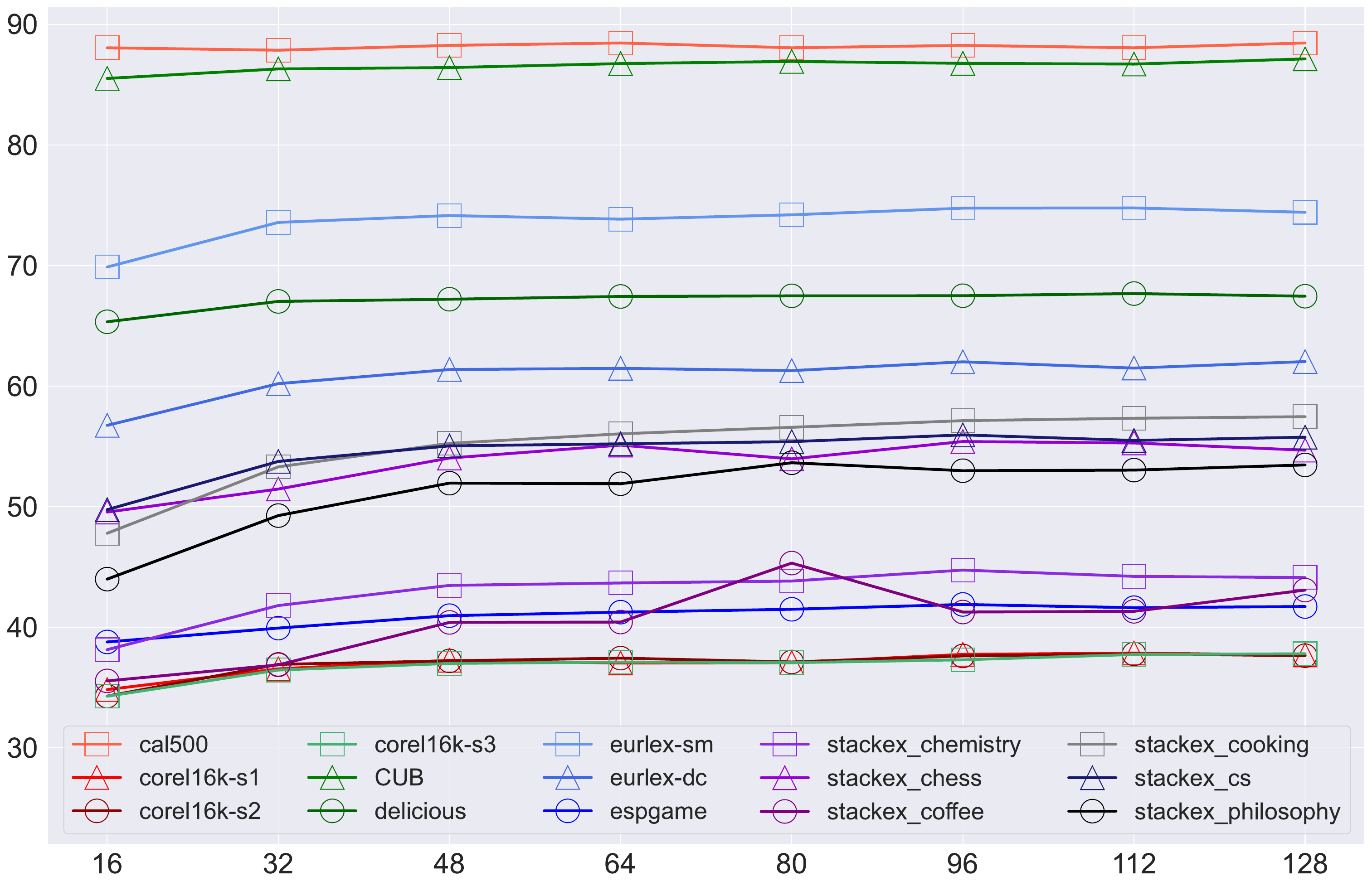}
			}
			\subfloat[P@$5$]{
				\includegraphics[width=0.965\columnwidth]{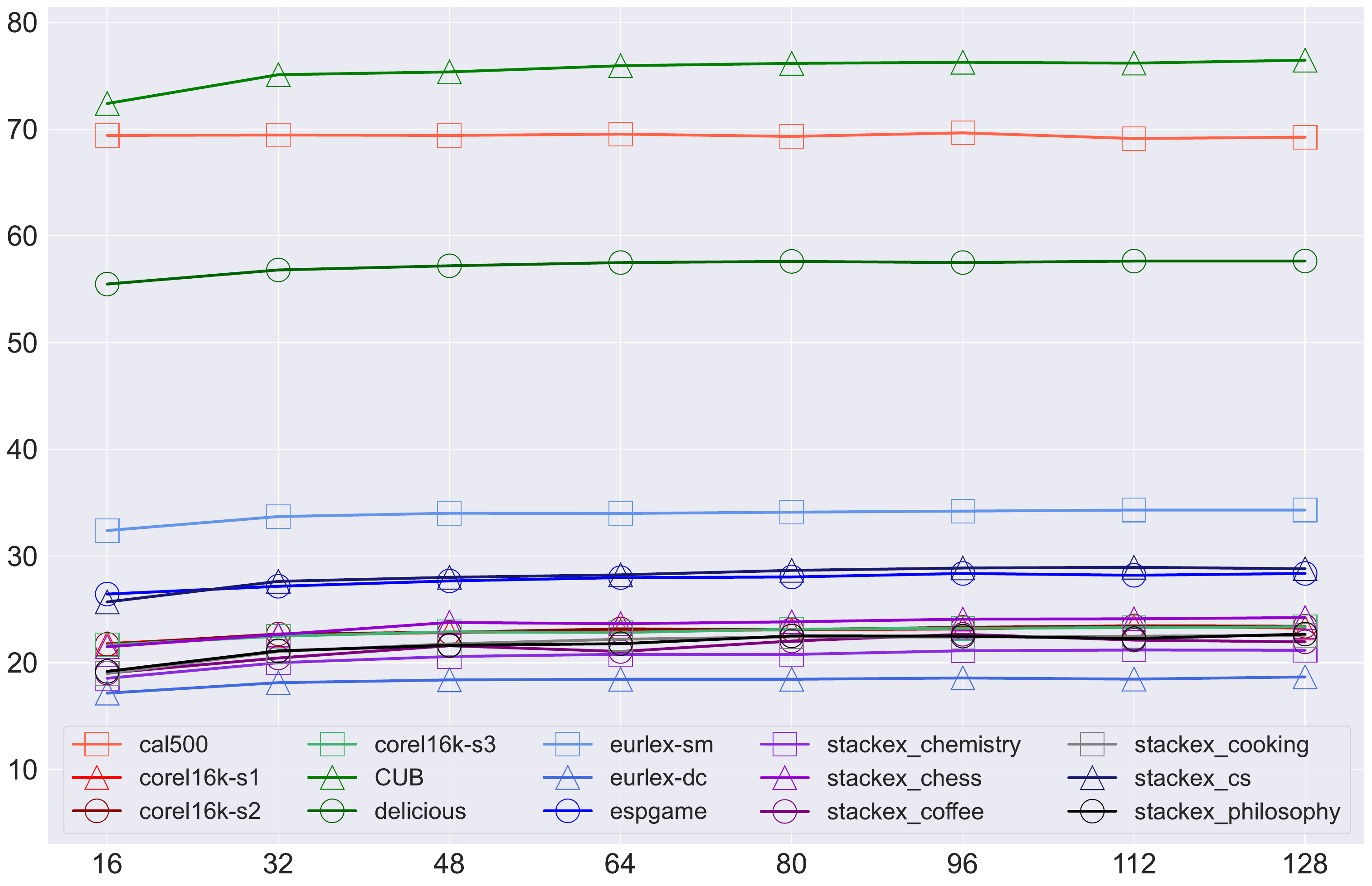}
			}
		}
		\caption{Effects of $\hat{c}$ on P@$1$ and P@$5$. In each subfigure, the x-axis indicates the value of $\hat{c}$, the y-axis indicates the value of P@$1$ and P@$5$.}
		\label{fig:ablation-param-c}
	\end{figure*}

	\begin{figure*} [t]
		\centering
		\resizebox{2\columnwidth}{!}{
			\subfloat[P@$1$]{
				\includegraphics[width=0.965\columnwidth]{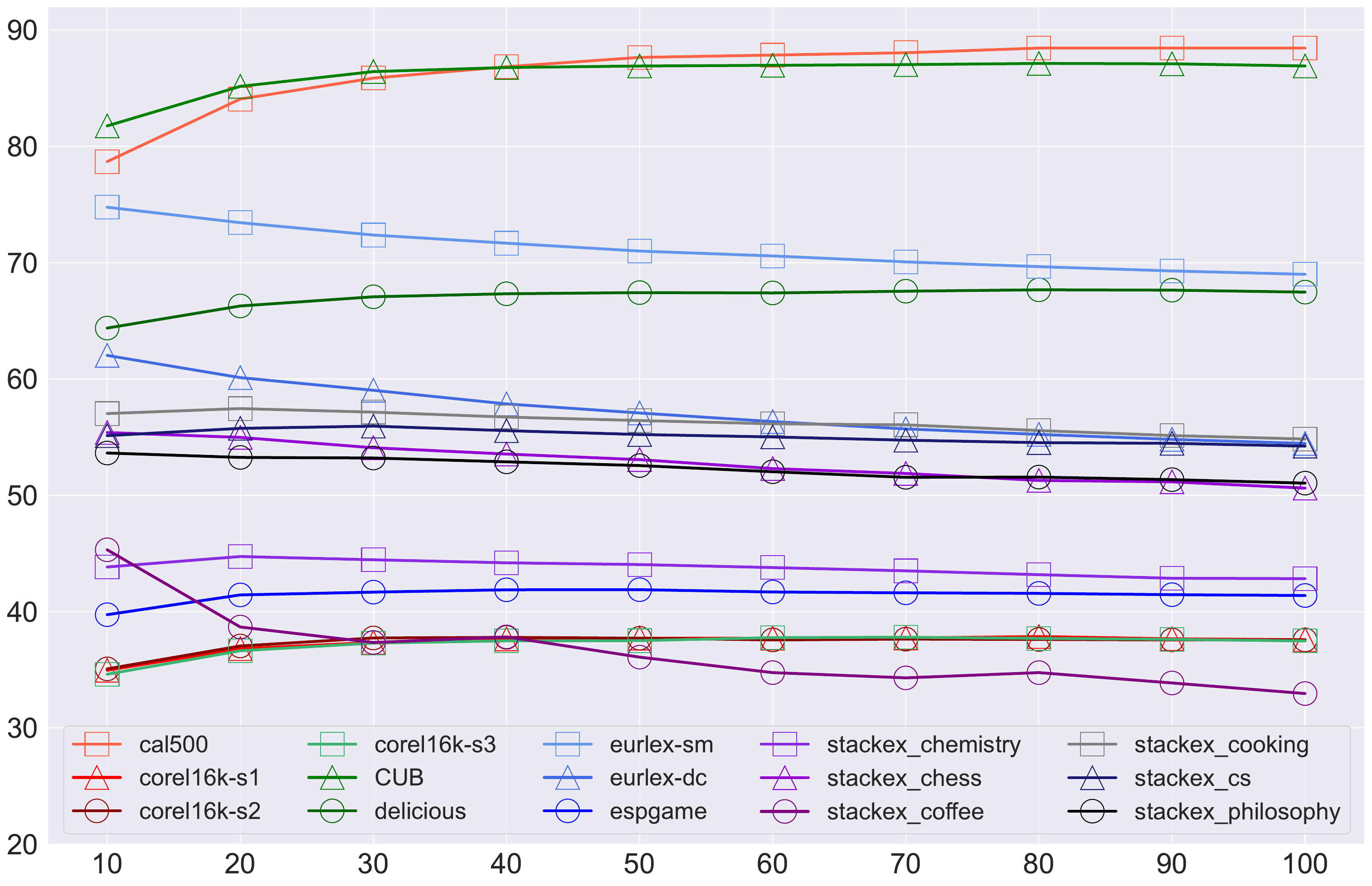}
			}
			\subfloat[P@$5$]{
				\includegraphics[width=0.965\columnwidth]{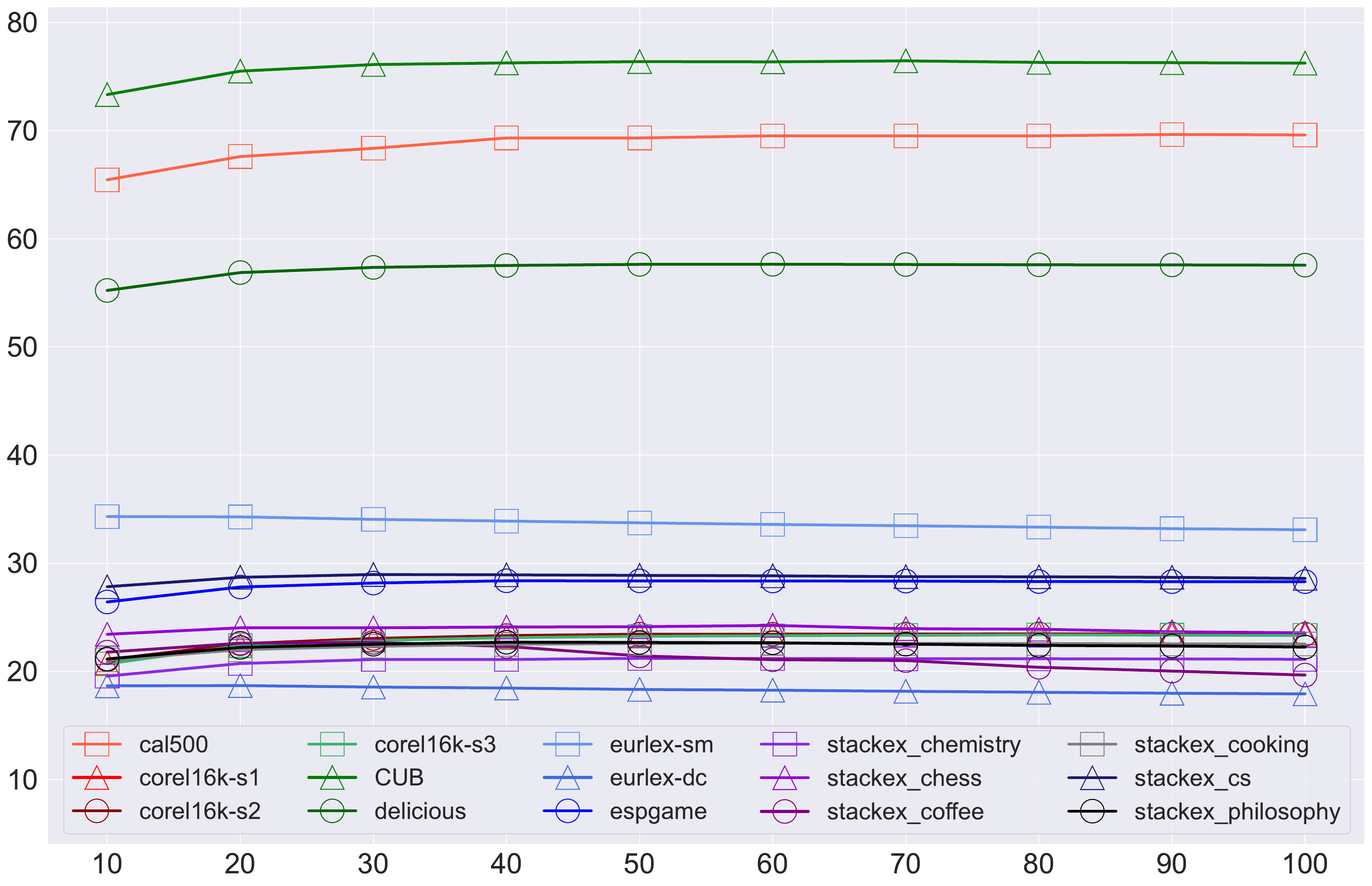}
			}
		}
		\caption{Effects of $k$ on P@$1$ and P@$5$. In each subfigure, the x-axis indicates the value of $\hat{c}$, the y-axis indicates the value of P@$1$ and P@$5$.}
		\label{fig:ablation-param-k}
	\end{figure*}

	\subsection{Performance Comparison}
	
	Tables \ref{table:result-P@1} to \ref{table:result-PSnDCG@5} present the experimental results of different algorithms evaluated using different evaluation metrics. For SLDL, the threshold $\tau$ for optimizing the target embedding distributions is set to 0.1, the balancing factor $\alpha$ is set to 1, the embedding dimension $\hat{c}$ and the number of neighbors $k$ are selected through grid search from the set $\{16,32,\cdots,128\}$ and $\{10,20,\cdots,100\}$, respectively. We perform ten-fold cross-validation for each algorithm and record the mean values and standard deviations of different methods on each evaluation metric. The two-tailed t-test at 0.05 significance level is conducted, and the best performance on each dataset is denoted in boldface. Each underlined result indicates that SLDL is statistically superior to the comparing method. Additionally, the average rank of each algorithm across all datasets is presented in the last row of each table.
	
	From Tables \ref{table:result-P@1} to \ref{table:result-PSnDCG@5}, it can be observed that SLDL outperforms existing multi-label classification algorithms in almost all cases. Moreover, SLDL can consistently statistically outperform the state-of-the-art multi-label classification methods in 96.4\% cases. Specifically, the compared methods, including SLEEC, DXML, CLIF, and KD-TEA try to mine and leverage label correlation, while the performance of our proposed SLDL significantly surpasses these approaches in almost all cases. These experimental results reveal the effectiveness of the asymmetric-correlational Gaussian embedding approach in SLDL and its ability to improve classification performance. In the meantime, we observe that on some datasets where asymmetric label correlations are not prominent, SLDL does not perform as well as on others. For example, on the CUB dataset, because the labels are fine-grained, some of the labels are very similar, while these labels are very different from others. Therefore, the asymmetric label correlations are not as prominent as in other datasets. Despite that, SLDL still achieves competitive performance on these datasets. When considering the average ranks across all fifteen benchmark datasets, SLDL demonstrates superior performance compared to other algorithms. Especially, SLDL achieves 1st or 2nd in 92.5\% cases when compared to these state-of-the-art algorithms. Therefore, SLDL exhibits superior performance over the state-of-the-art algorithms across all evaluation metrics.
	
	To further investigate the relative performance among the comparing algorithms, a statistical test, \textit{Friedman test}, is conducted. Table \ref{tab:Friedman} reports the Friedman statistics $F_F$ and the corresponding critical value on each evaluation criterion. As shown in Table \ref{tab:Friedman}, the null hypothesis of indistinguishable performance among the compared methods is rejected at a significance level of 0.05 for each evaluation criterion. Subsequently, a \textit{post-hoc Nemenyi test} \cite{Demsar2006} is applied at a significance level of 0.05 to determine whether the proposed SLDL method achieves competitive performance against the compared algorithms. The critical difference (CD) diagrams are depicted in Fig. \ref{fig:CD-diagram}. The lines connecting different algorithms in each sub-graph indicate that the corresponding compared algorithm does not exhibit a significant difference from SLDL. From Fig. \ref{fig:CD-diagram}, it can be observed that in almost all cases, our proposed SLDL method significantly outperforms the existing methods. Only the most recent state-of-the-art multi-label classification methods, KD-TEA and AdaC2, can achieve comparable performance to SLDL in some metrics. In the meantime, SLDL can outperform the compared algorithms on all metrics. These observations further demonstrate the effectiveness of SLDL.
	
	\begin{table}[t]
		\centering
		\caption{Comparison of SLDL with different cases on P@$5$.}
		\setlength{\tabcolsep}{2.8mm}
		\begin{tabular}{lccc}
			\toprule
			Dataset & w/o GE & w/o AC & SLDL (Ours) \\
			\midrule
			cal500 & 69.32$\pm$3.19 & 68.88$\pm$3.26 & \boldmath{}\textbf{69.64$\pm$3.21}\unboldmath{} \\
			corel16k-s1 & 23.15$\pm$0.51 & 23.23$\pm$0.51 & \boldmath{}\textbf{23.40$\pm$0.42}\unboldmath{} \\
			corel16k-s2 & 23.08$\pm$0.64 & 23.31$\pm$0.65 & \boldmath{}\textbf{23.46$\pm$0.65}\unboldmath{} \\
			corel16k-s3 & 23.31$\pm$0.48 & 23.31$\pm$0.58 & \boldmath{}\textbf{23.37$\pm$0.58}\unboldmath{} \\
			CUB   & 74.95$\pm$0.85 & 76.25$\pm$0.85 & \boldmath{}\textbf{76.45$\pm$0.83}\unboldmath{} \\
			delicious & 56.05$\pm$0.72 & 57.52$\pm$0.58 & \boldmath{}\textbf{57.64$\pm$0.68}\unboldmath{} \\
			eurlex-dc & 18.59$\pm$0.23 & 18.63$\pm$0.30 & \boldmath{}\textbf{18.68$\pm$0.30}\unboldmath{} \\
			eurlex-sm & 34.13$\pm$0.53 & 34.27$\pm$0.57 & \boldmath{}\textbf{34.31$\pm$0.55}\unboldmath{} \\
			espgame & 27.36$\pm$0.46 & 28.15$\pm$0.50 & \boldmath{}\textbf{28.37$\pm$0.40}\unboldmath{} \\
			stackex-chemistry & 20.42$\pm$0.47 & 20.63$\pm$0.28 & \boldmath{}\textbf{21.20$\pm$0.37}\unboldmath{} \\
			stackex-chess & 22.54$\pm$1.30 & 22.61$\pm$1.59 & \boldmath{}\textbf{24.24$\pm$1.38}\unboldmath{} \\
			stackex-coffee & 19.56$\pm$2.91 & 20.17$\pm$3.16 & \boldmath{}\textbf{22.67$\pm$2.97}\unboldmath{} \\
			stackex-cooking & 21.61$\pm$0.40 & 21.91$\pm$0.33 & \boldmath{}\textbf{22.59$\pm$0.23}\unboldmath{} \\
			stackex-cs & 28.65$\pm$0.42 & 28.55$\pm$0.62 & \boldmath{}\textbf{28.95$\pm$0.47}\unboldmath{} \\
			stackex-philosophy & 20.08$\pm$0.66 & 20.66$\pm$0.57 & \boldmath{}\textbf{22.69$\pm$0.56}\unboldmath{} \\
			\bottomrule
		\end{tabular}%
		\label{tab:ablation}%
	\end{table}%

	\subsection{Analysis of Time Consumption}
	
	In this subsection, we compare the training time of different methods. Several multi-label classification methods are included: the representative classifier chain-based approach, AdaC2 \cite{AdaBoost.C2}; the representative embedding-based approach, SLEEC \cite{SLEEC}; the representative label distribution learning (LDL)-based approach, FLEM \cite{FLEM}; and the most recent multi-label classification approach, KD-TEA \cite{KD-TEA}. For a fair comparison, all these methods run on CPU. For SLDL, the embedding dimension $\hat{c}$ is set to 128 and the number of neighbors $k$ is set to 100. The training time of these methods on different datasets is illustrated in Fig. \ref{fig:running-times}. As shown in Fig. \ref{fig:running-times}, AdaC2 requires significant time consumption when dealing with large-scale output space multi-label classification problems because it necessitates training a series of classifiers. For the other compared methods, although the time consumption has decreased compared to AdaC2, it remains very high on most datasets. Owing to the leveraging of simple yet effective Gaussian embedding and model training approaches, SLDL can achieve quite competitive results with lower time consumption in most cases. However, in some datasets with small sample sizes and numbers of labels (e.g., cal500 and stackex-coffee), SLDL takes longer time than SLEEC and FLEM. This phenomenon occurs because these datasets are relatively small, so the time reduction achieved by the Gaussian embedding method proposed in this paper is not as significant as with other large-scale datasets. Despite that, SLDL can achieve 10x to even 1000x speedup compared to other methods on most datasets, especially when the number of samples and the number of labels are large, which demonstrates the efficiency of SLDL.

	\subsection{Ablation Study} \label{sec:ablation}
	
	For a comprehensive understanding of our model, we further design two cases for SLDL to evaluate the effect of the proposed asymmetric-correlational Gaussian embedding approach:
	
	\begin{itemize}
		\item Case 1 (w/o GE):  without learning the Gaussian embedding, i.e., replace latent multivariate Gaussian distributions and Eq.(\ref{eq:loss-embedding}) with latent embedding vectors and mean square error loss function.
		\item Case 2 (w/o AC): without learning the asymmetric correlation among different labels, i.e., replace KL divergence with Jensen–Shannon (JS) divergence.
	\end{itemize}
	
	Table \ref{tab:ablation} tabulates the experimental results of different cases on P@$5$. From Table \ref{tab:ablation}, we can find that the Gaussian embedding is beneficial to improve the performance.	 Moreover, the introduction of asymmetric correlation can also promote the performance of the classification model. According to Theorem \ref{thm:bound}, these results further demonstrate that our embedding approaches are effective and superior. Moreover, the observations further prove the validity of the proposed asymmetric-correlational Gaussian embedding approach.

	\subsection{Effect of Hyperparameters $\hat{c}$ and $k$}
	
	In this subsection, we explore the effect of hyperparameters $\hat{c}$ and $k$. We compare the performances of SLDL with different values of $\hat{c}$ and $k$ on the fifteen datasets measured by P@$1$ and P@$5$. Figs. \ref{fig:ablation-param-c} and \ref{fig:ablation-param-k} illustrates the performances of SLDL with different values of $\hat{c}$ and $k$. From these curves, we can find that: 1) overall, SLDL has stable performances with a wide range of hyperparameter values on all fifteen datasets; 2) on each dataset, the impact trends of $\hat{c}$ and $k$ on P@$1$ and P@$5$ are almost the same; 3) the optimal performance of the model can be obtained when the value of $\hat{c}$ is between 64 and 96; 4) appropriate values of $k$ can bring slight performance gains on some datasets. These findings further demonstrate the robustness of the proposed SLDL.
	
	\section{Conclusion} \label{sec:conclusions}
	In this paper, a novel label distribution learning approach named SLDL is proposed, which can effectively explore and leverage the asymmetric correlation of different labels in multi-label classification and be well resilient to large-scale output space multi-label classification problems. SLDL can learn the distribution representation of different labels in the latent embedding space to extract the asymmetric correlation among different labels and ameliorate the scalability of the model. It first transforms labels into continuous distributions within a low-dimensional latent space, employing an asymmetric metric to establish correlations between different labels. Subsequently, it learns a mapping from the feature space to the latent space, where computational complexity is no longer tied to the number of labels. Finally, SLDL utilizes a nearest-neighbor-based strategy to decode the latent representations and generate the ultimate predictions. The effectiveness of SLDL is proved by both theoretical analysis and experimental results. In future work, we will further explore the theory of SLDL and promote the versatility of SLDL.

	\bibliographystyle{IEEEtran}
	\bibliography{references}

\end{document}